\documentclass[sn-mathphys,Numbered]{sn-jnl}

\usepackage{graphicx}%
\usepackage{amsmath,amssymb,amsfonts}%
\usepackage{amsthm}%
\usepackage{mathrsfs}%
\usepackage[title]{appendix}%
\usepackage{xcolor}%
\usepackage{textcomp}%
\usepackage{manyfoot}%
\usepackage{booktabs}%
\usepackage{algorithm}%
\usepackage{algorithmicx}%
\usepackage{algpseudocode}%
\usepackage{listings}%
\usepackage{color}%
\usepackage{float}%
\usepackage{hyperref}%
\usepackage{url}%
\usepackage{optidef}%
\usepackage{multirow}%
\usepackage{lscape}
\usepackage{subcaption}
\usepackage{a4wide}
\mathtoolsset{showonlyrefs}

\theoremstyle{thmstyleone}%
\newtheorem{theorem}{Theorem}
\theoremstyle{thmstyletwo}%

\theoremstyle{thmstylethree}%

\begin{document}

\title[Article Title]{A Two-Stage Training Method for Modeling Constrained Systems With Neural Networks}


\author*[1]{\fnm{C.} \sur{Coelho}}\email{cmartins@cmat.uminho.pt}

\author[1]{\fnm{M. Fernanda} \sur{P. Costa}}\email{mfc@math.uminho.pt}

\author[1,2]{\fnm{L.L.} \sur{Ferrás}}\email{lferras@fe.up.pt}

\affil[1]{\orgdiv{Centre of Mathematics (CMAT)}, \orgname{University of Minho}, \orgaddress{\city{Braga}, \postcode{4710-057}, \country{Portugal}}}

\affil[2]{\orgdiv{Department of Mechanical Engineering (Section of Mathematics) - FEUP}, \orgname{University of Porto}, \city{Porto}, \postcode{4200-465}, \country{Portugal}}

\abstract{
Real-world systems are often formulated as constrained optimization problems. Techniques to incorporate constraints into Neural Networks (NN), such as Neural Ordinary Differential Equations (Neural ODEs), have been used. However, these introduce hyperparameters that require manual tuning through trial and error, raising doubts about the successful incorporation of constraints into the generated model.
This paper describes in detail the two-stage training method for Neural ODEs, a simple, effective, and penalty parameter-free approach to model constrained systems. In this approach the constrained optimization problem is rewritten as two unconstrained sub-problems that are solved in two stages. The first stage aims at finding feasible NN parameters by minimizing a measure of constraints violation.
The second stage aims to find the optimal NN parameters by minimizing the loss function while keeping inside the feasible region.
We experimentally demonstrate that our method produces models that satisfy the constraints and also improves their predictive performance. 
Thus, ensuring compliance with critical system properties and also contributing to reducing data quantity requirements.
Furthermore, we show that the proposed method improves the convergence to an optimal solution and improves the explainability of Neural ODE models.
Our proposed two-stage training method can be used with any NN architectures.
}




\maketitle

\section{Introduction}
\label{sec:introduction}

Mathematical modeling of real-world systems is a pivotal subject. The description of system behavior and state variable interactions through a set of differential equations is crucial for understanding its temporal evolution, predicting its response under varying conditions. Furthermore, it helps to simulate diverse scenarios and conditions in a virtual environment, without the need for costly physical manipulation.

Due to the ability to approximate functions from data, neural networks (NNs) have been extensively used to model physical, biological, chemical and mechanical systems. However, real-world systems exhibit changes that unfold continuously over time, whereas classical NN models are discrete and fail to capture the complete temporal dynamics. Moreover, these systems are frequently unknown only represented by irregularly sampled data, posing a challenge as time-dependency is absent in NNs. Thus, to address these limitations, Neural ODEs were introduced \cite{chenNeuralOrdinaryDifferential2019a}.

In contrast to classical NNs, Neural ODEs model the hidden dynamics of the data as a time dependent functions using a continuous-depth NN. The temporal nature of Neural ODEs enables handling of irregularly sampled data and allows predictions throughout the entire time domain at arbitrary time-steps. The literature showcases various examples of real-world systems modeling using Neural ODEs, as can be found in \cite{xingContinuousGlucoseMonitoring2022, suKineticsParameterOptimization2022}. 
In \cite{xingContinuousGlucoseMonitoring2022} address the challenge of obtaining continuous glucose monitoring measurements, which can be expensive and limited in availability, to build predictive models of a patient's blood glucose levels. To tackle this issue, the authors propose the use of a Neural ODE to model the underlying dynamics based on low-quality and sparse data. Through their experiments and evaluations, the authors demonstrate the effectiveness of the Neural ODE approach in making accurate long-term predictions despite the sparse data.
In \cite{suKineticsParameterOptimization2022} employ a Neural ODE to capture and model the intricate dynamics of complex combustion phenomena. Due to the high cost associated with generating training data, the goal is to simulate only a limited amount of data and leverage the power of Neural ODEs to recover the unobserved dynamics. Different datasets of different sizes, types and noise levels were used to test the approach and the results demonstrate that the Neural ODE model fits the data and also retains its physical interpretability.

The availability of a large and diverse dataset plays a crucial role in training a NN model that can effectively capture the underlying patterns and behaviors of real-world systems. By using extensive and high-quality data, the NN model can achieve a higher level of fidelity in replicating complex dynamics. 
Being real-world systems often governed by fundamental laws or exhibit inherent constraints, it becomes crucial for a NN to perceive and extract the underlying principles from the training data. The ability of the NN model to capture and comprehend the governing laws or constraints significantly impacts its effectiveness in accurately representing and predicting the behavior of the system. Thus, being important to have a large and comprehensive dataset for training these models.
However, in numerous cases, the constraints associated to the system being modeled are known a priori, rendering the NN discovery of these constraints from data an unnecessary undertaking. Moreover, the inherent \emph{black-box} nature of NN models introduces uncertainty as to whether these known constraints are effectively extracted and learned from the data. Several critical constraints, such as conservation laws, must be satisfied by the model otherwise, the predictions lack meaningful interpretation. This requirement poses a significant difficulty to the adoption of these models by domain experts who rely on accurate adherence to these constraints for meaningful analysis and decision-making. Consequently, it becomes imperative to address the challenge of incorporating prior knowledge and enforcing constraints within NN models. This ensures their usefulness and applicability in real-world scenarios. Furthermore, incorporating constraints explicitly into NN models contributes to lower the amount of data needed for training a NN model.

The modeling of constrained real-world systems can be formulated as a constrained optimization problem, with the goal to produce a model that fits the data and satisfies the constraints. The constrained optimization techniques enable the exploration of the search space that not only achieves the desired goal but also hold the necessary constraints throughout the optimization process. 
One of the most common approaches for solving the problem of modeling constrained systems with NNs, is to add the constraints as penalty terms to the loss function, a well-established technique in constrained optimization designated by penalty methods \cite{nocedalNumericalOptimization2006}. 
Penalty terms introduce a regularization mechanism that penalizes violations of the constraints and multiplying them with a penalty parameter $\mu > 0$.
One significant obstacle lies in initializing and updating the penalty parameter. This parameter holds significant importance as it governs the delicate balance between the original loss function and the penalty terms, thereby impacting the training procedure and performance of the generated model. Consequently, identifying the optimal value for $\mu$ is a non-trivial task. 

In our previous preliminary work \cite{coelho2023prior} we briefly introduced a novel two-stage training method for Neural ODEs, which explicitly incorporates constraints into NNs, avoiding the introduction of penalty parameters as required by penalty methods.
In this paper, we extend our previous work by making the following significant contributions beyond the prior publication:

\begin{itemize}
	
	\item Background: We provide a comprehensive background section to ensure a thorough understanding of the concepts and techniques employed in our approach;

	\item Mathematical Formulation: We present detailed mathematical formulations of the problems addressed, providing a rigorous foundation for our methodology;

	\item Justification: We offer a more meticulous and compelling justification for employing our method over penalty methods. We explicitly articulate why the use of penalty parameters is undesirable and discuss whether introducing an additional training stage to bypass these parameters is a worthwhile endeavor;

        \item Mathematical proof: We demonstrate the equivalence of the constrained optimization problem and solving the two unconstrained sub-problems using the proposed two-stage method, thereby establishing that they have the same optimal minimizers;

	\item Pseudocode Algorithm: We present the pseudocode algorithm along with a detailed explanation, enabling readers to gain a clear understanding of the proposed two-stage training method and enhancing reproducibility;

        \item Computational Cost and Explainability Discussion: We add a brief discussion on the computational cost increase of the proposed two-stage method over \emph{vanilla} Neural ODE. Furthermore we also state how the proposed method helps enhancing explainability;

	\item Numerical Experiments: We include additional numerical experiments to analyse how the performance of the two-stage training method scales with data sparsity. These experiments provide valuable insights into the method's effectiveness under various conditions;

	\item Comparison with Baseline: We extend the results' analysis by introducing a \emph{vanilla} Neural ODE baseline, i.e. without the admissibility stage. This baseline serves as a comparison to demonstrate the superior performance of models generated with the two-stage training method;

	\item Convergence Study: We conducted a comprehensive convergence study to assess the convergence properties and stability of our proposed method. This analysis provides insight into the reliability and efficiency of our approach;

	\item Stopping/Convergence Criteria: We provide details and experimental investigations regarding stopping and convergence criteria. This includes an exploration of their influence on the model's performance, ensuring a robust evaluation of the proposed method;

	\item Developed Datasets: We present and characterize two datasets that have been specifically developed to evaluate the performance of the proposed method. We provide motivation for creating these datasets and explain the process behind their creation.

\end{itemize}

This paper provides a more in-depth exploration of the two-stage training method for Neural ODEs, presenting additional experimental results, comparative analyses, convergence studies, and algorithmic explanations to enhance the comprehensiveness and effectiveness of our previous work.

In Section \ref{sec:background}, as a foundation for understanding the subsequent sections of the paper, we provide an overview of key concepts related to Neural ODEs, the constrained optimization problem and penalty techniques, and existing research on modeling constrained systems using Neural ODEs. In Section \ref{sec:method}, we present a detailed description of the proposed two-stage training method for Neural ODEs. We provide explanations of the methodology and introduce a pseudocode algorithm that outlines the step-by-step process of applying the two-stage training method to Neural ODEs. In Section \ref{sec:results}, we present the results of the numerical experiments conducted to evaluate the performance and effectiveness of the proposed method. We include a thorough discussion of the results, highlighting the main key findings and insights. Additionally, we perform an experimental convergence analysis to assess the convergence properties and stability of our method. Moreover, we introduce two new specially developed datasets to evaluate the proposed method, providing insights into their characteristics and the motivation behind their creation. In Section \ref{sec:conclusion}, we summarize the main contributions and findings of the paper. We draw conclusions based on the results and discussions presented in earlier sections. Furthermore, we discuss potential future directions for research for further exploration in the field of modeling constrained systems with Neural ODEs. 
Finally we provide the development process of the datasets in Appendix \ref{app:datasets} and the full experimental setup details in Appendix \ref{subsec:setup} to promote reproducibility.

\section{Background and Related Work}
\label{sec:background}

Consider a real-world system described by a time series, with input $\boldsymbol{X}=(\boldsymbol{x}_0, \boldsymbol{x}_1, \dots, \boldsymbol{x}_{N-1})$, with $\boldsymbol{x}_n \in \mathbb{R}^d$ at time step $t_n$ with $n=0,\dots,N-1$. Let $\boldsymbol{Y}=(\boldsymbol{y}_1, \boldsymbol{y}_2, \dots, \boldsymbol{y}_N)$ be the corresponding ground-truth output time series, with $\boldsymbol{y}_n \in \mathbb{R}^{d^*}$ at time step $t_n$ with $n=1,\dots,N$. Let $\boldsymbol{\hat{Y}}(\boldsymbol{\theta})=(\boldsymbol{\hat{y}}_1(\boldsymbol{\theta}), \boldsymbol{\hat{y}}_2(\boldsymbol{\theta}), \dots, \boldsymbol{\hat{y}}_N(\boldsymbol{\theta}))$ be the NN prediction, with parameters $\boldsymbol{\theta}$, with $\boldsymbol{\hat{y}}_n(\boldsymbol{\theta}) \in \mathbb{R}^{d^*}$, at time step $t_n$ with $n=1,\dots,N$.

\subsection{Neural ODEs}

Real-world systems often exhibit continuous-time dynamics, which poses a challenge when using NNs for modeling due to the lack of the ability to explicitly handle the temporal aspect of the data, thus only able to handle regularly sampled data.
To address this issue, the authors in \cite{chenNeuralOrdinaryDifferential2019a} introduced Neural ODEs. This NN architecture adjusts a continuous-time dependent function $\boldsymbol{f_\theta}$, an ODE, so that when solving an Initial Value Problem (IVP) using a numerical method its curve of solutions fits the training data:

\begin{equation*}
    \boldsymbol{\hat{Y}}(\boldsymbol{\theta}) = ODESolve(\boldsymbol{f_\theta}, \boldsymbol{y}_0, (t_0,t_N))
\end{equation*}

\noindent where the initial condition $(\boldsymbol{y}_0, t_0)$ with $\boldsymbol{y}_0 \equiv \boldsymbol{x}_0$, and $(t_0, t_{N})$ is the range over which the curve of solutions is fitted to the data.

The result of training a Neural ODE is an ODE. To make predictions an IVP is solved using a numerical method with initial condition $(\boldsymbol{y}_0, t_0)$ in the desired time interval and discretization \cite{chenNeuralOrdinaryDifferential2019a}.

\textbf{Remark:} For simplicity we consider that the output of the numerical method is the output of the Neural ODE, $\boldsymbol{\hat{Y}}(\boldsymbol{\theta})$. However this may not be always the case and a second NN might be used to transform the numerical method solution into the prediction $\boldsymbol{\hat{Y}}(\boldsymbol{\theta})$.

\subsection{Constrained Optimization Problem}

A time series describes the evolution of a system throughout time and can be given by a matrix with state variables' vectors at each time step. When modeling a time series using a NN, the goal is to find the NN's parameters $\boldsymbol{\theta}$ that minimize a loss function $l(\boldsymbol{\theta})$:

\begin{mini}|l|[0]
    {\boldsymbol{\theta} \in \mathbb{R}^{n_\theta}}{l(\boldsymbol{\theta}) = \dfrac{1}{N} \sum_{n=1}^{N} (\boldsymbol{\hat{y}}_n(\boldsymbol{\theta}) - \boldsymbol{y}(\boldsymbol{\theta}))^2.}
    {\label{eq:unconstrained}}
    {}
\end{mini}

\noindent where $l: \mathbb{R}^{n_\theta} \rightarrow \mathbb{R}$ is a measure of the error between the ground-truth $\boldsymbol{Y}$ and predicted $\boldsymbol{\hat{Y}}$ time series. In this work, we consider $\boldsymbol{l(\boldsymbol{\theta})}$ to be the Mean Squared Error (MSE). 

When a real-world system is constrained by governing laws then they must be satisfied throught the whole time interval. Thus, the problem of estimating $\boldsymbol{\theta}$ is formulated as a constrained optimization problem:

\begin{mini}|l|[0]
    {\boldsymbol{\theta} \in \mathbb{R}^{n_\theta}}{l(\boldsymbol{\theta}) = \dfrac{1}{N} \sum_{n=0}^{N-1} (\boldsymbol{\hat{y}}_n(\boldsymbol{\theta}) - \boldsymbol{y}(\boldsymbol{\theta}))^2}
    {\label{eq:constrained}}
    {}
    \addConstraint{c_{t_n}^j(\boldsymbol{\hat{y}}_n(\boldsymbol{\theta}))}{ \le 0}, \,\, j \in \varepsilon ,\,\,\, n=1,\dots,N,
    \addConstraint{c_{t_n}^i(\boldsymbol{\hat{y}}_n(\boldsymbol{\theta}))}{= 0}, \,\,\, i \in \mathcal{I}  ,\,\,\, n=1,\dots,N,
\end{mini}

\noindent where $c_{t_n}^i, c_{t_n}^j: \mathbb{R}^{n_\theta} \rightarrow \mathbb{R}$ are the equality and inequality constraint functions, respectively, with $\mathcal{I}$ the equality and $\varepsilon$ the inequality index sets of constraints, over a time interval $(t_0, t_{N-1})$. The set of points $\boldsymbol{\theta}$ that satisfy all the constraints defines the feasible set $$\mathcal{S}=\{ \boldsymbol{\theta}  \in \mathbb{R}^{n_\theta} : c_{t_n}^i(\boldsymbol{\hat{y}}_n(\boldsymbol{\theta}))=0, \,  i \in \mathcal{I} ;\,\, c_{t_n}^j(\boldsymbol{\hat{y}}_n(\boldsymbol{\theta})) \leq 0, \, j \in \varepsilon, n=1,\dots,N \}.$$

To solve \eqref{eq:constrained}, a common approach is to rewrite the constrained problem as an unconstrained one using a penalty method. This approach enables the use of well-established unconstrained optimization techniques to find an optimal solution. 

Penalty methods define a penalty function, $\phi(\boldsymbol{\theta})$, that combines the loss function and the constraints violations, known as penalty terms, multiplied by a penalty parameter $\mu>0$.

In these methods, a sequence of sub-problems, each with gradually larger penalty parameter $\mu \rightarrow \infty$, is solved. The sequence of solutions of the sub-problems converge towards the optimal solution of the original constrained problem \cite{nocedalNumericalOptimization2006}. 

Selecting suitable initial values and designing an appropriate sequence of penalty parameter values can accelerate the optimization process. Well-chosen values can provide a good initial balance between satisfying the constraints and minimizing the loss function, promoting faster convergence to the feasible region and optimal solutions. However, the optimal selection of $\mu$ values often relies on problem-specific knowledge and may require experimentation and fine-tuning \cite{nocedalNumericalOptimization2006}.

One of the most commonly used penalty methods is the L1 exact penalty method \cite{nocedalNumericalOptimization2006}. Problem \eqref{eq:constrained} can be reformulated using the L1 exact penalty function as:

\begin{mini}|l|[0]
	{\boldsymbol{\theta} \in \mathbb{R}^{n_\theta}}{ \phi(\boldsymbol{\theta}) = l(\boldsymbol{\theta}) + \mu \sum_{j=1}^\varepsilon  P_j(\boldsymbol{\theta})  + \mu \sum_{i=1}^\mathcal{I}  P_i(\boldsymbol{\theta}).}
    {\label{eq:penalty}}
    {}
\end{mini}

\noindent where $P_j, P_i$ are the penalty terms defined as the total constraint violations $j,i$ at all time steps, defined respectively as: 
$P_j(\boldsymbol{\theta}) =  \sum_{n=1}^{N} |c_{t_n}^j(\boldsymbol{\theta})| , \,\, j \in \varepsilon; \,\,\,
P_i(\boldsymbol{\theta}) =  \sum_{n=1}^{N}  [c_{t_n}^i(\boldsymbol{\theta})]^+, \,\, i \in \mathcal{I},$
whit $[z]^+=\max(z, 0)$.

In the context of NNs, penalty methods are highly popular for incorporating constraints into the model due to its simplicity of implementation \cite{raissiPhysicsInformedDeep2017a}. 
One major challenge lies in determining the appropriate values for $\mu$, as they strongly influence the optimization process. Poorly chosen penalty parameters may hinder convergence or lead to infeasible solutions. 


\subsection{Approaches to Model Constrained Systems with Neural ODEs}

In the literature, some works on strategies to model constrained systems with Neural ODEs are available.
One approach is to modify the architecture of Neural ODEs in such a way that specific constraints are embedded \cite{roehrlModelingSystemDynamics2020,zhongDissipativeSymODENEncoding2020}. 
In \cite{roehrlModelingSystemDynamics2020}, the authors proposed Physics-Informed Neural Ordinary Differential Equations (PINODE), an architecture in which the equations of motion from Lagrange mechanics are directly incorporated into the the NN's structure. This architecture is specific for model-based control and mechanical systems, governed by Lagrangian mechanics.
In \cite{zhongDissipativeSymODENEncoding2020}, the authors introduced Dissipative SymODEN, a modified Neural ODE architecture that incorporates port-Hamiltonian dynamics into its structure, used to describe energy flow and dissipation in physical systems. By explicitly incorporating port-Hamiltonian dynamics, Dissipative SymODEN provides a specialized architecture to capture the behavior of physical systems subjected to dissipative forces.

However, the modified Neural ODE architectures, for instance PINODE and Dissipative SymODEN, are often designed to address particular types of constraints making them less versatile for handling arbitrary constraints that may arise in different domains. As a result, while these specialized architectures provide valuable solutions for their target systems, they may require further adaptation or extension to handle different types of constraints or more complex systems that are not within the scope of their initial design.

 A common practice is to use L1 exact penalty method. This allows to impose various types of constraints in an independent way of the architecture \cite{tuorConstrainedNeuralOrdinary2020,limUnifyingPhysicalSystems2022}. In \cite{tuorConstrainedNeuralOrdinary2020}, the authors proposes using a L1 exact penalty method to incorporate prior knowledge into Neural ODEs with application to modeling and control of industrial systems. This is done by incorporating inequality constraints into Neural ODEs by adding the respective penalty terms to the original loss function. 
 In \cite{limUnifyingPhysicalSystems2022}, the authors propose using a L1 exact penalty method to model systems that follow the law of conservation of energy, without requiring any modification of the architecture. This is particularly suitable for Hamiltonian and dissipative systems. The authors acknowledge that choosing appropriate values for the penalty parameters can be challenging. Setting the value too low may not effectively enforce the constraints, while setting it too high may jeopardize the training process. This trade-off between enforcing the constraints and maintaining effective training is a key consideration in the L1 exact penalty method \cite{limUnifyingPhysicalSystems2022}.

Although penalty methods are general and easily implementable compared to modifying the Neural ODE architecture itself, the selection of appropriate penalty parameters is a challenging task.

\section{The Two-stage Method}
\label{sec:method}

 In this work we present the two-stage training method for Neural ODEs to model constrained systems that is parameter-free \cite{coelho2023prior}.

In our approach, we reformulate \eqref{eq:constrained} as two unconstrained minimization sub-problems. The first sub-problem is defined by the minimization of the average of the total constraints violation, over the entire time interval, of the original problem:

\begin{mini}|l|[0]
	{\boldsymbol{\theta} \in \mathbb{R}^{n_\theta}}{ \mathcal{L}_I(\boldsymbol{\theta}) = \sum_{j=1}^\varepsilon \left(\dfrac{1}{N} \sum_{n=0}^{N-1} |c_{t_n}^j(\boldsymbol{\theta})| \right) + \sum_{i=1}^\mathcal{I} \left(\dfrac{1}{N} \sum_{n=0}^{N-1} [c_{t_n}^i(\boldsymbol{\theta})]^+ \right). }
    {\label{eq:admissibilitystage}}
    {}
\end{mini}

The second sub-problem is defined by the minimization of the original loss function:

\begin{mini}|l|[0]
	{\boldsymbol{\theta} \in \mathcal{S}}{\mathcal{L}_{II}(\boldsymbol{\theta}) = l(\boldsymbol{\theta}).}
    {\label{eq:optimizationstage}}
    {}
\end{mini}

In the first stage of our proposed two-stage training method, the aim is to find a feasible solution of the original problem by solving \eqref{eq:admissibilitystage} by finding $\boldsymbol{\theta}$ that minimizes the loss function $\mathcal{L}_I$.Thus, the solution obtained in this stage is a feasible point for the original constrained optimization problem \eqref{eq:constrained}. To this stage we call admissibility stage and it allows us to obtain a feasible starting point for the subsequent refinement in the second stage of our training method.

In the second stage, \eqref{eq:optimizationstage} is solved by using the solution obtained in the first stage as the starting point for the optimization process. By doing so, we ensure that we begin the second stage from a point within the feasible region of the original constrained problem.
The goal of the second stage is to minimize the loss function $\mathcal{L}_{II}$ of the unconstrained sub-problem \eqref{eq:optimizationstage}. Thus, to this stage we call optimization stage.

Note that, using the solution point of the admissibility stage as the starting point of the optimization stage only guarantees us the second stage starts in the feasible region. However there is no guarantee that the subsequent points are feasible. Thus, in the second stage, in case the optimization process leaves the feasible region, a \emph{preference point} strategy is proposed.

%

Each point computed in the optimization stage is a trial point $\boldsymbol{\theta}_{\text{trial}}$ and it is only accepted as a new point $\boldsymbol{\theta}_k$ at iteration $k$ if its admissibility is better (lesser than a feasibility tolerance). Otherwise it is rejected and, either the previous point $\boldsymbol{\theta}_{k-1}$ (\emph{updatePrevious}, see Algorithm \ref{alg:updPrev}) or the point with the best admissibility $\boldsymbol{\theta}_{\text{best}}$ (\emph{updateBest}, see Algorithm \ref{alg:updBest}), is chosen as $\boldsymbol{\theta}_k$.

By leveraging the solution from the first stage as the starting point and continuously monitoring the constraints violation, when searching for the optimal solution we guarantee that the optimization process remains within the feasible region of the original constrained problem. This ensures that the final solution obtained from the second stage is also a solution to the original constrained problem \eqref{eq:constrained}.


We now demonstrate the equivalence between problem \eqref{eq:constrained} and the pair of two sub-problems \eqref{eq:admissibilitystage} and \eqref{eq:optimizationstage}, thereby establishing that they have the same global minimizers.

\begin{theorem}\label{thm:equivalence}
Let $\boldsymbol{\theta}^*$ be a global solution to the constrained problem \eqref{eq:constrained}. Then, $\boldsymbol{\theta}^*$ is also a global solution of the unconstrained sub-problem \eqref{eq:admissibilitystage} and \eqref{eq:optimizationstage}.
\end{theorem}

\begin{proof}
Let $\boldsymbol{\theta}^*$ be a global solution to the constrained problem \eqref{eq:constrained}. By definition, we have that $\boldsymbol{\theta}^*$ satisfies all the equality and inequality constraints, for all $j \in \varepsilon$ and $i \in \mathcal{I}$ at each time-step $t_n$. Therefore we have,
$\boldsymbol{c}^j_{t_n}(\boldsymbol{\theta}^*) = 0 \quad \text{and} \quad \boldsymbol{c}^i_{t_n}(\boldsymbol{\theta}^*) \le 0, \,\, n=1,\dots,N.$

We now consider the unconstrained sub-problem (\ref{eq:admissibilitystage}) in which the constraints of \eqref{eq:constrained} define the loss function.
Since $\boldsymbol{\theta}^*$ is a global solution to \eqref{eq:constrained}, thus satisfying all the constraints, the corresponding terms in the loss function for $\mathcal{L}_I(\boldsymbol{\theta}^*)$ are zero. 
Where for all $j \in \varepsilon$ we have,
$\dfrac{1}{N} \sum_{n=0}^{N-1} |c_{t_n}^j(\boldsymbol{\theta}^*)| = \dfrac{1}{N} \sum_{n=0}^{N-1} |0| = 0,$
and for all $i \in \mathcal{I}$ we have,
$\dfrac{1}{N} \sum_{n=0}^{N-1} [c_{t_n}^i(\boldsymbol{\theta}^*)]^+ = \dfrac{1}{N} \sum_{n=0}^{N-1} [0]^+ = 0,$
thus, $\mathcal{L}_I(\boldsymbol{\theta}^*) = 0.$ 
Therefore, $\boldsymbol{\theta}^*$ is a global solution of problem \eqref{eq:admissibilitystage}.

On the other hand, by definition, we have that $\boldsymbol{\theta}^* \in S$ and $\boldsymbol{\theta}^*$ is the global minimizer of $l(\boldsymbol{\theta})$. Since problem \eqref{eq:optimizationstage} is the minimization of $l(\boldsymbol{\theta})$ over the feasible set $\mathcal{S}$, we have
$l(\boldsymbol{\theta}^*) = \mathcal{L}_{II}(\boldsymbol{\theta}^*).$
Therefore, $\boldsymbol{\theta}^*$ is a global solution of problem \eqref{eq:optimizationstage}.
\end{proof}

\begin{theorem}\label{thm:equivalence_II}
Let $\boldsymbol{\theta}^*$ be a global solution to the pair of unconstrained sub-problems \eqref{eq:admissibilitystage} and \eqref{eq:optimizationstage}. Then, $\boldsymbol{\theta}^*$ is also a global solution of the constrained problem \eqref{eq:constrained}.
\end{theorem}

\begin{proof}
Let $\boldsymbol{\theta}^*$ be a global solution to the pair of unconstrained sub-problems \eqref{eq:admissibilitystage}-\eqref{eq:optimizationstage}. By definition, we have that $\boldsymbol{\theta}^*$ satisfies all the equality and inequality constraints, for all $j \in \varepsilon$ and $i \in \mathcal{I}$ at each time-step $t_n$. Therefore we have,
$\boldsymbol{c}^j_{t_n}(\boldsymbol{\theta}^*) = 0 \quad \text{and} \quad \boldsymbol{c}^i_{t_n}(\boldsymbol{\theta}^*) \le 0, \,\, n=1,\dots,N.$ 
Therefore, $\boldsymbol{\theta}^*$ is a feasible solution of problem \eqref{eq:constrained}.

Since $\boldsymbol{\theta}^*$ is a global minimizer of sub-problem \eqref{eq:optimizationstage} over the feasible set $\mathcal{S}$, we have
$\mathcal{L}_{II}(\boldsymbol{\theta}^*) = l(\boldsymbol{\theta}^*).$
Therefore, $\boldsymbol{\theta}^*$ is a global solution of problem \eqref{eq:constrained}.
Therefore, $\boldsymbol{\theta}^*$ is a global solution of problem \eqref{eq:admissibilitystage}.
\end{proof}


\paragraph{Computational Cost}

From a high-level perspective, the computational cost comparison between the proposed two-stage training method and the \emph{vanilla} Neural ODE, which exclusively addresses unconstrained optimization problems,
hinges on the additional steps introduced in the two-stage approach. While the \emph{vanilla} Neural ODE involves solving a single unconstrained optimization problem, optimizing the original loss function directly, the two-stage method incurs additional computational costs associated with the addition of the admissibility stage, where the constrained optimization problem is reformulated as two unconstrained sub-problems. 

The first phase involves solving the admissibility problem to identify a feasible starting point, progressing until a specified feasibility tolerance value is attained. The subsequent optimization stage fine-tunes the solution within the feasible region. Although the two-stage method introduces additional computational steps, the advantage lies in its capability to explicitly and effectively incorporate constraints without relying on the selection of penalty parameters. Furthermore, the admissibility stage contributes to the potential for enhanced convergence and solution quality.

The computational cost of the two-stage method is contingent upon the efficiency of solving the supplementary unconstrained admissibility sub-problem and the intricacy of the constraints involved. While there is an inherent increase in computational overhead due to the two-stage nature of the method, the benefits in terms of constraint handling and potential improvement in convergence and solution quality may outweigh the added cost, particularly in scenarios where explicit constraint management is crucial.

\paragraph{Improved Explainability}

The proposed two-stage training method represents a significant advancement in enhancing the explainability of Neural ODEs in comparison to both \emph{vanilla} Neural ODEs and Neural ODEs employing penalty methods. In contrast to \emph{vanilla} Neural ODEs, which inherently lack explicit consideration of constraints, and penalty methods, which introduce additional hyperparameters, the two-stage approach provides a more interpretable and controlled optimization process.

By reformulating the constrained optimization problem into two unconstrained minimization sub-problems, admissibility and optimization stages, this method ensures a transparent optimization process. The admissibility stage is dedicated to identifying a feasible starting point that satisfies the constraints, thereby establishing a clear foundation for subsequent refinement in the optimization stage. The incorporation of a \emph{preference point} strategy during the optimization stage further enhances interpretability by allowing only those trial points that do not result in worse admissibility, providing a more comprehensible trajectory for the optimization process.

This step-by-step approach not only promotes a more understandable optimization path but also ensures that the final solution obtained is interpretable and adheres to the constraints specified in the original problem. Consequently, it enhances the overall explainability of the Neural ODE model, offering insights into the model's decision-making process and the impact of constraints on its behavior.

\subsection{Algorithm for Neural ODEs}

In this section we present the algorithm that implements the proposed two-stage method for Neural ODEs to model constrained real-world systems, Algorithm \ref{alg:2Stage}.

The integration of the two-stage method into the training algorithm of Neural ODEs is straightforward and simple. It involves adding an additional training loop, the admissibility stage, before the conventional NN training loop, which corresponds to the optimization stage. The admissibility stage shares similarities with the optimization stage, with the key difference being the loss function being minimized. In the admissibility stage, the function to minimize is given by the constraints violation, loss function $\mathcal{L}_I$. In the optimization stage, is given by the original loss function $\mathcal{L}_{II}$.

The stopping criteria of the admissibility stage is the feasibility tolerance, $tol$. When the new point $\boldsymbol{\theta}$ produces an $\mathcal{L}_{I}$  value smaller or equal than the predefined tolerance, $\mathcal{L}_I(\boldsymbol{\theta}) \leq tol$, then the first stage stops and $\boldsymbol{\theta}$ is the starting point for the second stage.
Due to the nature of the constraints of the system being modeled, the random initialization of $\boldsymbol{\theta}$ can meet the stopping criteria condition and thus not performing a single iteration. To promote the improvement of the admissibility ($L_{I}(\boldsymbol{\theta}) < tol$ and $L_{I}(\boldsymbol{\theta}) \neq 0$), we set a minimum amount of iterations for this stage, $k_{\min}$, to ensure that a certain amount of optimization steps are performed. The goal is to try to obtain a better feasible point.

Once the admissibility stage is completed, the final parameters $\boldsymbol{\theta}$ are used as the initialization $\boldsymbol{\theta}_0$ for the subsequent optimization stage. In this stage, if the \emph{updateBest} strategy is employed (see Algorithm \ref{alg:updBest}): if the computed parameters $\boldsymbol{\theta}_{trial}$ produce better admissibility than $\boldsymbol{\theta}_{best}$, $P_{best}$, then the point is accepted as $\boldsymbol{\theta}_{k}$ and replaces $\boldsymbol{\theta}_{best}$ and $P_{best}$;
    otherwise it is rejected and $\boldsymbol{\theta}_{k}$ is updated with $\boldsymbol{\theta}_{best}$.

If the \emph{updatePrevious} strategy is employed (see Algorithm \ref{alg:updPrev}): 
    if $\boldsymbol{\theta}_{trial}$ produce better admissibility than $\boldsymbol{\theta}_{k-1}$, $P_{k-1}$, then the point is accepted as $\boldsymbol{\theta}_{k}$ and replaces $\boldsymbol{\theta}_{k-1}$ and $P_{k-1}$;
    otherwise it is rejected and $\boldsymbol{\theta}_{k}$ is updated with $\boldsymbol{\theta}_{k-1}$.
 
The stopping criteria of the optimization stage is the maximum number of iterations, $k_{\max}$. When $k_{\max}$ iterations have been reached the parameters, $\boldsymbol{\theta}$, of the NN that build the ODE dynamics are the ones at the last iteration.

\begin{algorithm}[]
\caption{: The two-stage training method algorithm for Neural ODEs.}
\label{alg:2Stage}
\begin{algorithmic}
\State \textbf{Input:} Initial condition $(\boldsymbol{y}_0,t_0)$, start time $t_0$, end time $t_{N}$, minimum number of iterations $k_{\min}$, maximum number of iterations $k_{\max}$, feasibility tolerance $tol$;
\State $\boldsymbol{f_\theta} \leftarrow DynamicsNN()$;
\State Initialize parameters $\boldsymbol{\theta}$;

\While{$\mathcal{L}_I \geq tol$ OR $k \leq k_{\min}$}
    \State \{$\boldsymbol{\hat{y}}_n\}_{n=1, \dots, N} \leftarrow ODESolve(\boldsymbol{f_\theta}, \boldsymbol{y}_0, (t_0, t_N))$;
    \State Evaluate $\mathcal{L}_I$;
    \State $\nabla \mathcal{L}_I \leftarrow Optimiser.BackpropCall(\mathcal{L}_I)$;
    \State $\boldsymbol{\theta} \leftarrow Optimiser.Step(\nabla \mathcal{L}_I, \boldsymbol{\theta})$;
\EndWhile

\State $\boldsymbol{\theta}_0 \leftarrow \boldsymbol{\theta}$;
\State $P_{0} \leftarrow$ Evaluate $\mathcal{L}_I(\boldsymbol{\theta}_{0})$;
\State Initialize $P_{\text{best}}, \boldsymbol{\theta}_{\text{best}}$, $\mathcal{L}_{\text{best}}$;
\For {$k=1:k_{\max}$}
    \State $\nabla \mathcal{L}_{II} \leftarrow Optimiser.BackpropCall(\mathcal{L}_{II})$;
    \State $\boldsymbol{\theta}_{\text{trial}} \leftarrow Optimiser.Step(\nabla \mathcal{L}_{II}, \boldsymbol{\theta}_{k-1})$;
    \State \{$\boldsymbol{\hat{y}}_n\}_{n=1, \dots, N} \leftarrow ODESolve(\boldsymbol{f_\theta}, \boldsymbol{y}_0, (t_0, t_N))$;
    \State $P_{\text{trial}} \leftarrow$ Evaluate $\mathcal{L}_I(\boldsymbol{\theta}_{\text{trial}})$;
    
    \If{\emph{Preference point} strategy $== updatePrevious$}
	\State $\boldsymbol{\theta}_k, P_{k} \leftarrow updatePrevious(P_{k-1}, P_{\text{trial}}, \boldsymbol{\theta}_{k-1}, \boldsymbol{\theta}_{\text{trial}})$;
    \Else
	\State $\boldsymbol{\theta}_k, P_{\text{best}} \leftarrow updateBest(P_{\text{best}}, P_{\text{trial}}, \boldsymbol{\theta}_{\text{best}}, \boldsymbol{\theta}_{\text{trial}})$;
    \EndIf
\EndFor

\State $\boldsymbol{\theta} \leftarrow \boldsymbol{\theta}_k$;
\State \textbf{return} $\boldsymbol{\theta}$;

\end{algorithmic}
\end{algorithm}

\begin{algorithm}[]
	\caption{: \emph{Preference point} strategy \emph{updatePrevious()}.}
\label{alg:updPrev}
\begin{algorithmic}
\State \textbf{Input:} $P_{k-1}, P_{\text{trial}}, \boldsymbol{\theta}_{k-1}, \boldsymbol{\theta}_{\text{trial}}$;

\If{$P_{\text{trial}} > P_{k-1}$}
	\State \textbf{return} $\boldsymbol{\theta}_{k-1}, P_{k-1}$;
\Else
\State \textbf{return} $\boldsymbol{\theta}_{\text{trial}}, P_{\text{trial}}$;
\EndIf

\end{algorithmic}
\end{algorithm}

\begin{algorithm}[]
\caption{: \emph{Preference point} strategy \emph{updateBest()}.}
\label{alg:updBest}
\begin{algorithmic}

	\State \textbf{Input:} $P_{\text{best}}, P_{\text{trial}}, \boldsymbol{\theta}_{\text{best}}, \boldsymbol{\theta}_{\text{trial}}$;

\If{$P_{\text{trial}} > P_{\text{best}}$}
\State \textbf{return} $\boldsymbol{\theta}_{\text{best}}, P_{\text{best}}$;
\Else
\State \textbf{return} $\boldsymbol{\theta}_{\text{trial}}, P_{\text{trial}}$;
\EndIf

\end{algorithmic}
\end{algorithm}

\textbf{Remark:} The proposed two-stage method is not confined to Neural ODEs, rather it can be seamlessly incorporated into any NN architectures. This adaptability arises from the method's fundamental principles of transforming constrained optimization problems into two unconstrained sub-problems and solving them sequentially. By design, this framework is agnostic to the specific NN structure facilitating swift and straightforward implementation. The addition of a second training loop (admissibility stage) is a minimal adjustment that can be easily integrated. This makes the two-stage method an accessible and flexible tool for addressing constraints in various NN paradigms.

\section{Numerical Experiments}
\label{sec:results}

In this section, we present a comprehensive experimental analysis of the proposed two-stage training method for Neural ODEs. The experiments conducted in this paper are based on the datasets, World Population Growth (WPG) \cite{coelho_population_2023} and Chemical Reaction (CR) \cite{coelho_chemical_2023}, and three experiments (reconstruction, extrapolation and completion) as in \cite{coelho2023prior}. 

For the reconstruction experiment, the training and testing sets are the same. The goal of this experiment is to evaluate the performance of the models at fitting to the training data. By assessing the models' performance in reproducing the observed data, we can determine the effectiveness of the learned system dynamics. 

For the extrapolation experiment, the training and testing sets have the same number of data points but the testing time interval is larger. This experiment is designed to assess the model's ability to generalize and make accurate predictions outside the training time interval. By evaluating the model's performance in extrapolating the system dynamics, we can determine whether the model has successfully learned the underlying constraints and can make reliable predictions beyond the observed data, being essential for forecasting and decision-making in many domains. 

For the completion experiment, the training and testing time interval is the same but the testing set has more time-steps. The objective is to evaluate the model's performance in completing missing data points within the observed time interval, crucial when dealing with incomplete datasets where the missing data needs to be estimated. By providing the model with partial observations and evaluating its ability to fill in the gaps, we can assess the model's effectiveness in capturing the dynamics and completing the system's trajectory.

In this work there are some key differences in the analysis when compared to the one presented in \cite{coelho2023prior}. 
Instead of reporting average results for three runs as in \cite{coelho2023prior}, we provide results for four independent runs. We include a baseline comparison with \emph{vanilla} Neural ODE. This baseline is trained without the admissibility stage and \emph{preference point} strategy, which are the main components of the two-stage method. Thus, we can assess the impact and effectiveness of the proposed method in comparison to the conventional approach.
Additionally, we present test results for Neural ODE models trained using the two-stage method without the \emph{preference point} strategy. We note that, this analysis aims to emphasize the importance of ensuring that the optimization stage remains within the feasible region of the constrained problem. By omitting this strategy, we can evaluate its impact on model performance and constraints satisfaction.
Furthermore, we conduct extended experiments to evaluate the proposed method's effectiveness in scenarios with smaller or larger training datasets gaining insight into its generalization capabilities. Moreover, a convergence analysis is performed to assess the convergence behavior of the proposed two-stage training method and to understand the stability of the optimization process.
Since in the admissibility stage the stopping criterion is based on the feasibility tolerance, it is important to analyze how different feasibility tolerance values impact the overall solution.
We provide results for training Neural ODE models with the two-stage training method with four different tolerance values: 1E-2, 1E-4, 1E-6, 1E-8.   

All implementations were done in \emph{Pytorch} using the \emph{Torchdiffeq} Neural ODE library  and can be found at [LINK]\footnote{available after acceptance}. The experiments were conducted on a computer equipped with an AMD Ryzen 9 5900HS processor, 16GB of RAM, and an NVIDIA GeForce RTX 3060 graphics card.

\textbf{Remark:} Comparing a \emph{vanilla} Neural ODE, trained solely on the unconstrained problem \eqref{eq:unconstrained}, with a Neural ODE trained using the two-stage method is not a fair comparison. Thus, when comparing the performance and capabilities of these two approaches, it is crucial to recognize the fundamental difference in their objectives. In spite of this, conducting a comparative analysis can still be valuable in understanding the distinctive characteristics and potential advantages of the proposed two-stage method by allowing us to highlight the benefits of explicitly incorporating constraints and showcasing the improved accuracy and reliability in modeling constrained systems.

\subsection{Performance Analysis}

The WPG and CR datasets were modeled using two different approaches: the \emph{vanilla} Neural ODE and the Neural ODE with the proposed two-stage training method. These models were used to conduct seven experiments, as described in Section \ref{subsec:setup}.

For Neural ODE models trained with the two-stage method, three variations were employed: without the \emph{preference point} strategy (\emph{noStrategy}), and with the \emph{preference point} strategy namely \emph{updatePrevious} and \emph{updateBest}. The average MSE ($\text{MSE}_{\text{avg}}$) and average constraints violation ($\text{V}_{\text{avg}}$) were measured on the testing sets.

The performance results of these experiments with \emph{vanilla} Neural ODE and all variations of two-stage Neural ODE, namely \emph{noStrategy}, \emph{updatePrevious} and \emph{updateBest}, are presented in Appendix \ref{app:exp} Tables \ref{tab:performanceWPG} and \ref{tab:performanceCR}, which display, for each, the $\text{MSE}_{\text{avg}}$ and $\text{V}_{\text{avg}}$ along with their respective standard deviations (std).


\subsubsection{World Population Growth}

Based on the results presented in Appendix \ref{app:WPG} Table \ref{tab:performanceWPG}, we conducted a comprehensive analysis of the performance of the different modeling approaches and variations for the WPG dataset. Here, we summarize the main findings for each experiment:

\begin{itemize}
    \item \textbf{Experiment 1.0 - Reconstruction:} The \emph{vanilla} Neural ODE was outperformed by all models obtained with the two-stage method, particularly the \emph{updatePrevious} variation showed values of $\text{MSE}_{\text{avg}}$ one order of magnitude lower. In general, increasing the feasibility tolerance criteria for the admissibility stage did not show significant changes in model performance. With the exception of \emph{noStrategy} that achieved better performance when using higher tolerance values;

    \item \textbf{Experiment 2.0 - Extrapolation:} As expected, the performance of the models decreased when compared to the reconstruction experiment, as they were predicting for an unseen time-horizon with increased complexity. All models trained with the two-stage method variations demonstrate similar performance for this experiment achieving $\text{V}_{\text{avg}}$ values one order of magnitude lower than \emph{vanilla} Neural ODE. The \emph{vanilla} Neural ODE achieved a similar $\text{MSE}_{\text{avg}}$ value to the two-stage models. This indicates that the solutions are not feasible. The \emph{noStrategy} shows best performance with the two lowest tolerance values;
    
    \item \textbf{Experiment 2.1 - Extrapolation with sparser training set:} When compared to 2.0, the \emph{vanilla} Neural ODE models showed lower performance in the extrapolation experiment when trained with fewer points, as expected. However, the models trained with the proposed two-stage method did not show a decrease in performance and show equivalent performance to 2.0. This indicates that our method is superior when the available datasets are sparse, as it explicitly incorporates constraints that leverage information extracted from the data. In general, the \emph{updatePrevious} strategy presented the best performance. Changing the tolerance did not significantly affect the performance;
    
    \item \textbf{Experiment 2.2 - Extrapolation with abundant training set:} The performances were equivalent to those observed in 2.0, indicating that training with a larger dataset did not significantly improve the models' performance. The tolerance values chosen for the first stage did not affect the results significantly;
    
    \item \textbf{Experiment 3.0 - Completion:} The models trained with the two-stage method showed a one-order-of-magnitude increase in performance, both $\text{MSE}_{\text{avg}}$ and $\text{V}_{\text{avg}}$ values, compared to the \emph{vanilla} Neural ODE. Lowering the tolerance value at the admissibility stage resulted in higher performance models using the \emph{updatePrevious} strategy while higher tolerance values favored the \emph{noStrategy};
    
    \item \textbf{Experiment 3.1 - Completion with sparser training set:} The \emph{vanilla} Neural ODE was able to produce models with one order of magnitude lower $\text{V}_{\text{avg}}$ than in 3.0, showing that using a sparser dataset for training does not affect negatively the models. In general all models trained with the two-staged method show similar performance to 3.0 showing that the performance was not impacted, when doing completion, by reducing the training set;

    \item \textbf{Experiment 3.2 - Completion with abundant training set:} The models trained with the proposed two-stage method demonstrated better performance than the \emph{vanilla} Neural ODE models, with a one-order-of-magnitude difference in the $\text{MSE}_{\text{avg}}$. As observed in the extrapolation experiments, increasing the number of points used for training did not have a significant impact on the performance.
    
\end{itemize}

\subsubsection{Chemical Reaction}

Based on the results presented in in Appendix \ref{app:CR} Table \ref{tab:performanceCR}, we conducted a comprehensive analysis of the performance of the different modeling approaches and variations for the CR dataset. Here, we summarize the main findings for each experiment:

\begin{itemize}
    \item \textbf{Experiment 1.0 - Reconstruction:} The \emph{vanilla} Neural ODE models were outperformed by all the models obtained with the three two-stage method strategies. The most remarkable performance boost was observed in the case of the \emph{noStrategy} and \emph{updateBest} where a tolerance of 1E-8 led to a noteworthy reduction in both the $MSE_{avg}$ and $V_{avg}$  by approximately two orders of magnitude compared to the baseline \emph{vanilla} Neural ODE. Reducing tolerance to both 1E-6 and 1E-8 consistently resulted in improved outcomes for models trained with the \emph{noStrategy} and \emph{updateBest}. In contrast, the \emph{updatePrevious} method did not show a significant difference in performance when the tolerance was adjusted.;

    \item \textbf{Experiment 2.0 - Extrapolation:} he performance of the models in this experiment exhibited a decline in comparison to the reconstruction task. Notably, models trained using the proposed two-stage training method, with the 3 strategies, significantly outperformed the baseline \emph{vanilla} Neural ODE models. In fact, the \emph{vanilla} Neural ODE models are proved to be inadequate, displaying high values of $MSE_{avg}$ and $V_{avg}$. The most remarkable overall performance was achieved by \emph{noStrategy} with $tol=$1E-6;
    
    \item \textbf{Experiment 2.1 - Extrapolation with sparser training set:} The results in this experiment exhibit a consistent trend similar to that observed in Experiment 2.0, with the \emph{vanilla} Neural ODE  models delivering poor performance. Among the various strategies tested, \emph{ noStrategy} stands out, displaying the lowest values for both $MSE_{avg}$ and $V_{avg}$ when using a tolerance of 1E-4. In particular, a higher tolerance of 1E-2 leads to models with diminished performance;
    
    \item \textbf{Experiment 2.2 - Extrapolation with abundant training set:} The \emph{vanilla} Neural ODE demonstrated a notable enhancement in model performance by leveraging a larger dataset for training, resulting in a remarkable reduction of approximately two orders of magnitude in the values $MSE_{avg}$ and $V_{avg}$ compared to the results of experiments 2.0 and 2.1. This did not occur with the two-stage training method models showing their robustness to the availability of data.
    Among the proposed strategies, \emph{noStrategy} consistently delivered the most impressive performance in the various tolerance values. It is worth highlighting that, surprisingly, setting the tolerance at 1E-2 produced models with inferior performance compared to the \emph{vanilla} Neural ODE;
    
    \item \textbf{Experiment 3.0 - Completion:} The models trained with the proposed two-stage method offer similar or better performance than the models obtained with the \emph{vanilla} Neural ODE baseline being the lowest $MSE_{avg}$ and $V_{avg}$ obtained with \emph{noStrategy} and  $tol=$1E-8. Changing the tolerance threshold of the admissibility stage does not show to have a significant impact to the results;
    
    \item \textbf{Experiment 3.1 - Completion with sparser training set:} The models obtained through the proposed two-stage training method consistently outperform the \emph{vanilla} Neural ODE baseline by a significant margin, with improvements of at least one order of magnitude for both $MSE_{avg}$ and $V_{avg}$. As expected, the baseline performance was adversely affected due to the reduced amount of data, when compared with experiment 3.0. On the contrary, the models trained with the two-stage method show robust performance. The best overall performance was achieved with the \emph{updateBest} strategy. It is worth mentioning that a higher tolerance value of $tol=$1E-2 generally resulted in models that were less performant, while no significant differences were observed between the other values;

    \item \textbf{Experiment 3.2 - Completion with abundant training set:} Impressively, the performance of the \emph{vanilla} Neural ODE in this experiment closely mirrors that of Experiment 3.1, indicating that the increased volume of data does not yield performance improvements for the models. Once again, models trained using the proposed two-stage method generally exhibit better performance compared to the baseline. when it comes to adjusting the tolerance, the \emph{updatePrevious} strategy proves to be the most robust, with consistent performance even when changing the tolerance values. In contrast, the \emph{updateBest} strategy struggles when using a tolerance of 1E-2, performing similarly to the baseline in such cases.
    
\end{itemize}

\subsection{Experimental Convergence Analysis}

To analyze the convergence towards both feasible and optimal solutions during the training process, we plotted the optimization loss values corresponding to $\mathcal{L}_{II}$ and the total constraints violation during the optimization stage, corresponding to $\mathcal{L}_I$, for both the \emph{vanilla} Neural ODEs and the Neural ODEs trained with the two-stage method. The latter includes three different strategies and various feasibility tolerance values. The plots were generated with a frequency of 20 iterations. Note that all values were taken during the optimization stage and consequently the constraints violations curves represented describe the progress during the optimization stage (inclusive for \emph{vanilla} Neural ODEs).

\textbf{Remark:} Be aware that the scale of the plots may vary between different tolerance values, as we aimed to identify the most appropriate window for visualization. Note that since the \emph{vanilla} Neural ODE only undergoes the optimization stage, the plot curves for these models are identical across all tolerance values.

\subsubsection{World Population Growth}

Figures \ref{fig:normal}-\ref{fig:more} aggregate the plots for the WPG dataset for the various tolerance values for the models trained with: $200$ points in the time interval $(0,300)$ used in experiments 1.0, 2.0 and 3.0, Figure \ref{fig:normal}; $100$ points in the time interval $(0,300)$ used in experiments 2.1 and 3.1, Figure \ref{fig:fewer}; $300$ points in the time interval $(0,300)$ used in experiments 2.2 and 3.2, Figure \ref{fig:more}. The best runs at the extrapolation experiments 2.0, 2.1 and 2.2 for each tolerance were chosen to be plotted.

\begin{figure}[]
\begin{subfigure}[h]{0.5\textwidth}
        \centering
        \includegraphics[width=\textwidth]{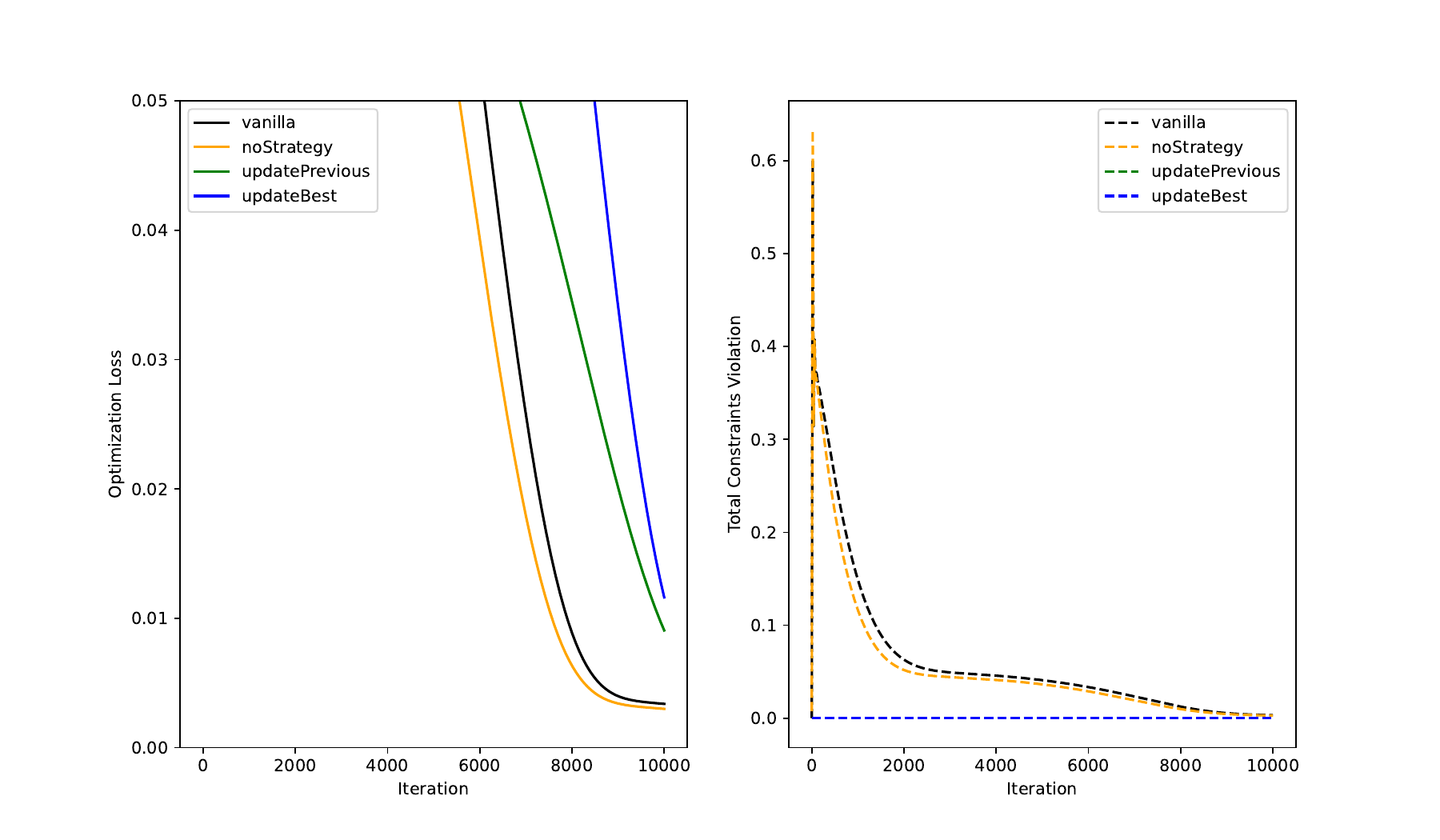} 
        \caption{1E-2}
    \end{subfigure}
    \hskip -2ex
    \begin{subfigure}[h]{0.5\textwidth}
        \centering
        \includegraphics[width=\textwidth]{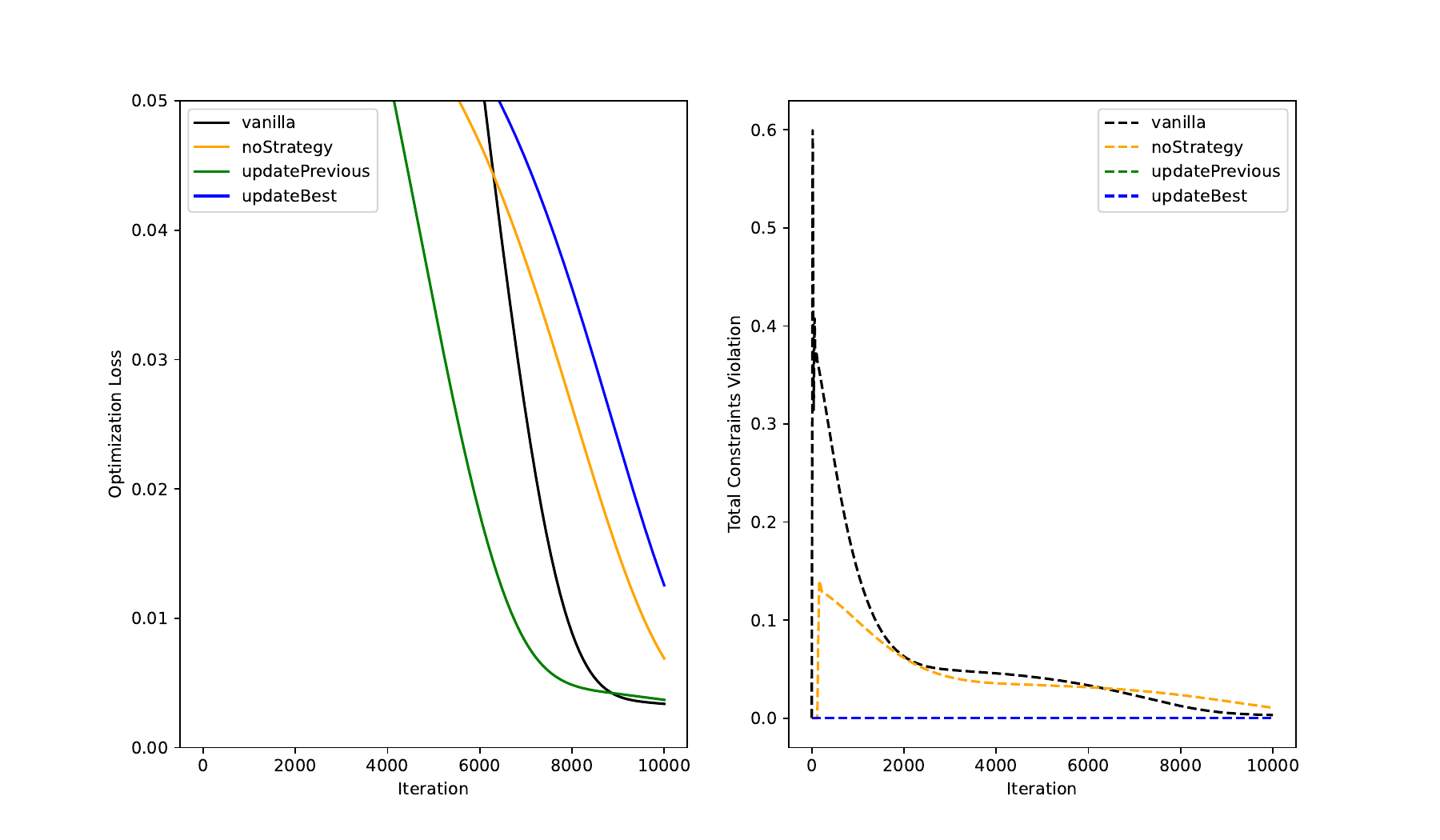} 
        \caption{1E-4}
    \end{subfigure}
    \hskip -2ex
    \begin{subfigure}[h]{0.5\textwidth}
        \centering
        \includegraphics[width=\textwidth]{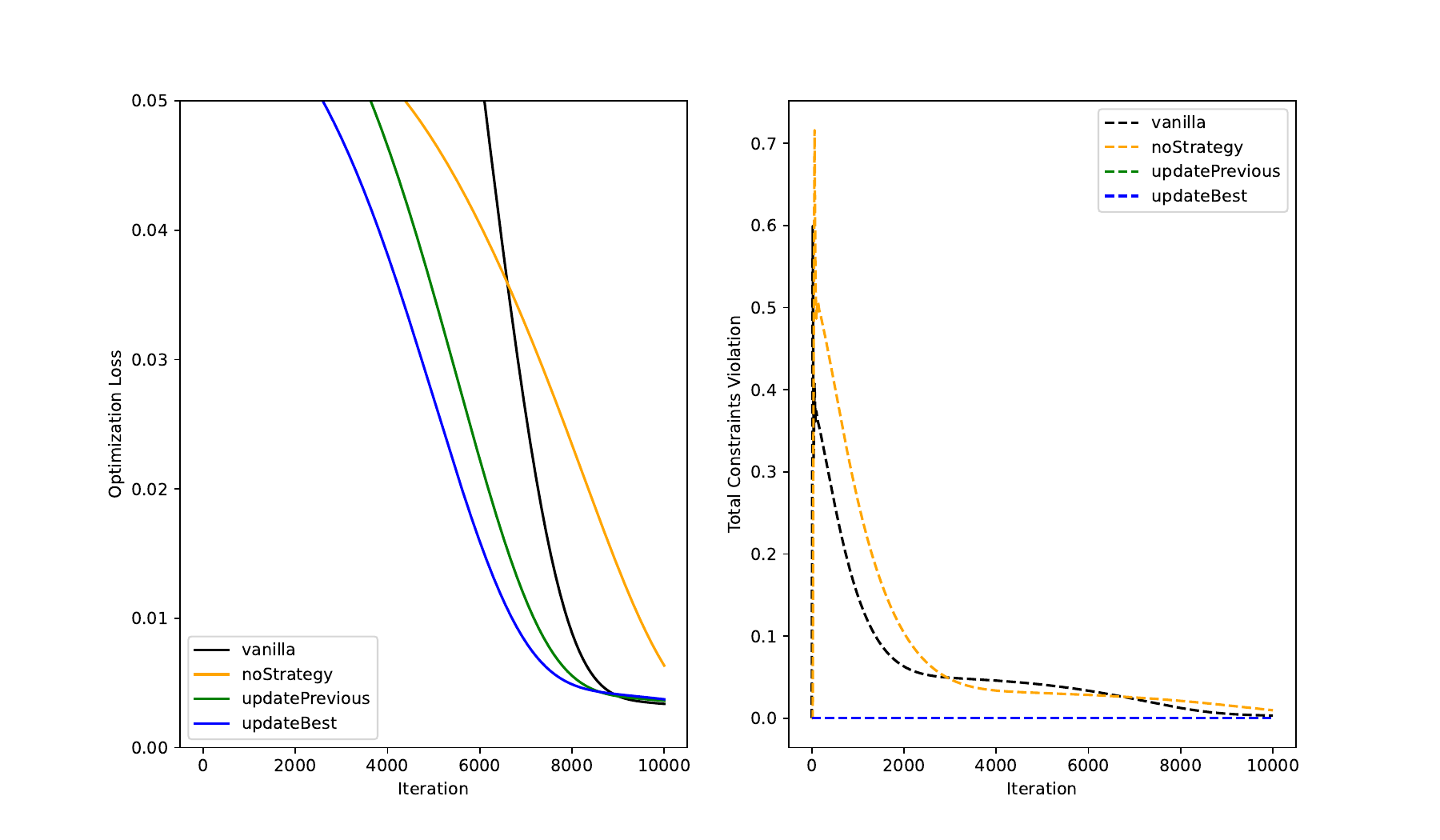} 
        \caption{1E-6}
    \end{subfigure}
    \hskip -2ex
    \begin{subfigure}[h]{0.5\textwidth}
        \centering
        \includegraphics[width=\textwidth]{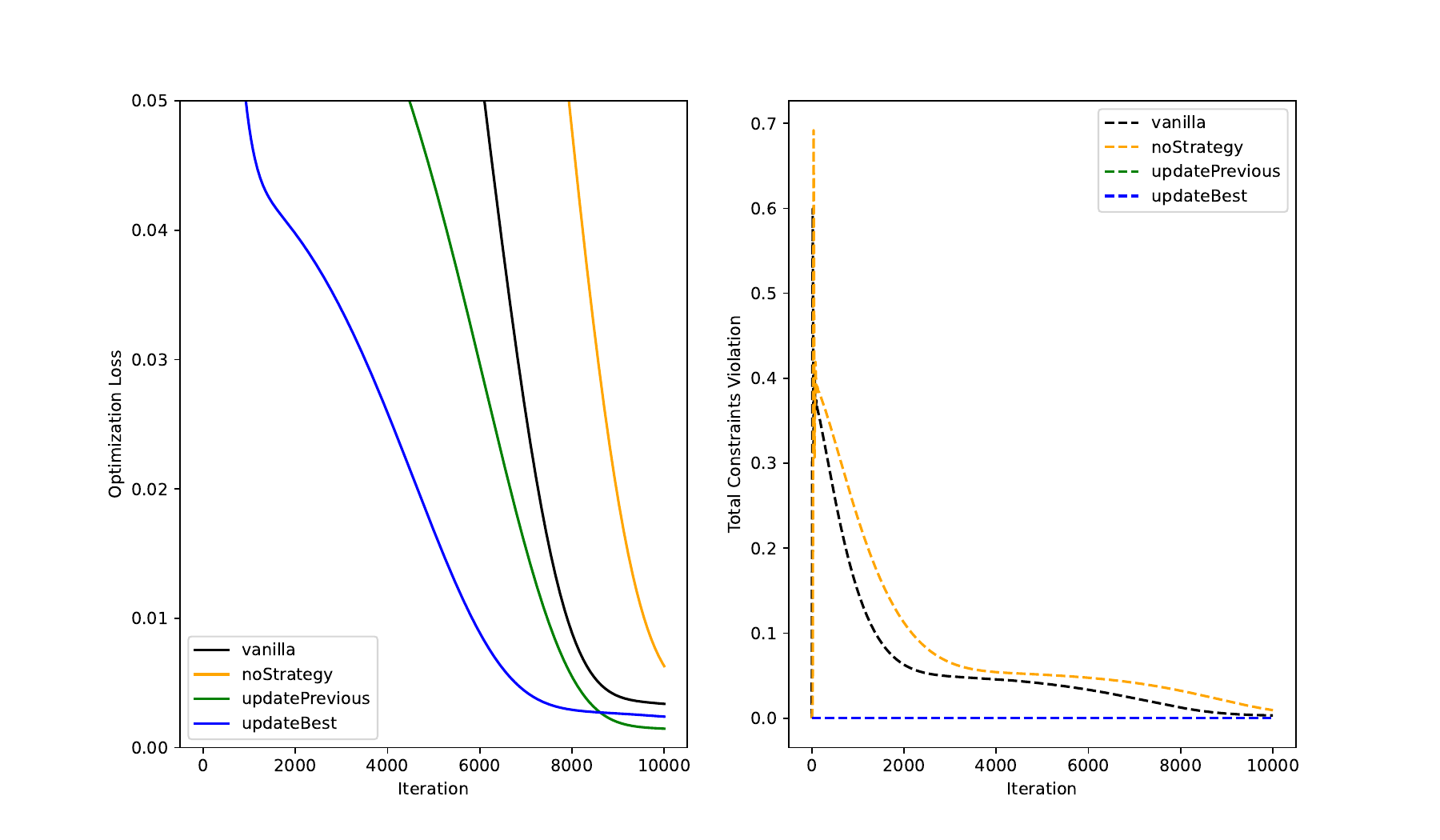}
        \caption{1E-8}
    \end{subfigure}
    \caption{Plots of loss (left) and constraints violation (right), during admissibility stage, for the various tolerance values during training of the models used in experiments 1.0, 2.0 and 3.0.}
\label{fig:normal}
\end{figure}

Upon examining Figure \ref{fig:normal}, it becomes evident that the loss in the admissibility stage ($\mathcal{L}_{I}$) decreases faster when using the two-stage method. Furthermore, this reduction leads to attaining lower values across all tolerance settings, except for the 1E-4 tolerance, where the \emph{vanilla} Neural ODE reaches the final iteration with the lowest value. 

As expected,  the devised \emph{preference point} strategies, designed to keep the solutions inside the feasible region during the second stage, have the desired effect. Both the \emph{updatePrevious} and \emph{updateBest} strategies ensure that that the total constraints violation during the optimization phase do not increase. Although, during the optimization stage it does not decrease either. Conversely, the \emph{noStrategy} and \emph{vanilla} Neural ODE approaches exhibit fluctuating and unstable values.

\begin{figure}[]
\begin{subfigure}[h]{0.5\textwidth}
        \centering
        \includegraphics[width=\textwidth]{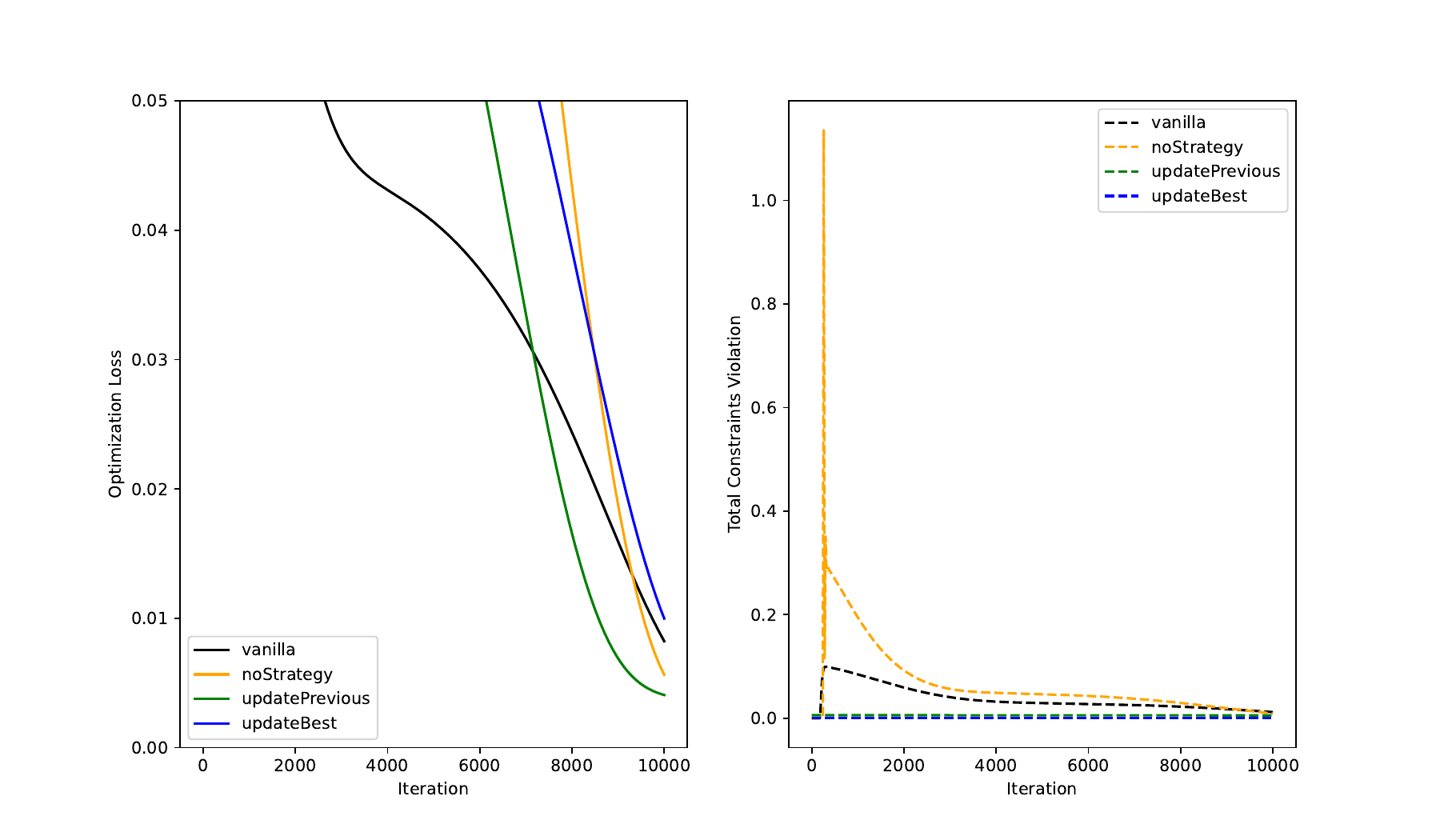} 
        \caption{1E-2}
    \end{subfigure}
    \hskip -2ex
    \begin{subfigure}[h]{0.5\textwidth}
        \centering
        \includegraphics[width=\textwidth]{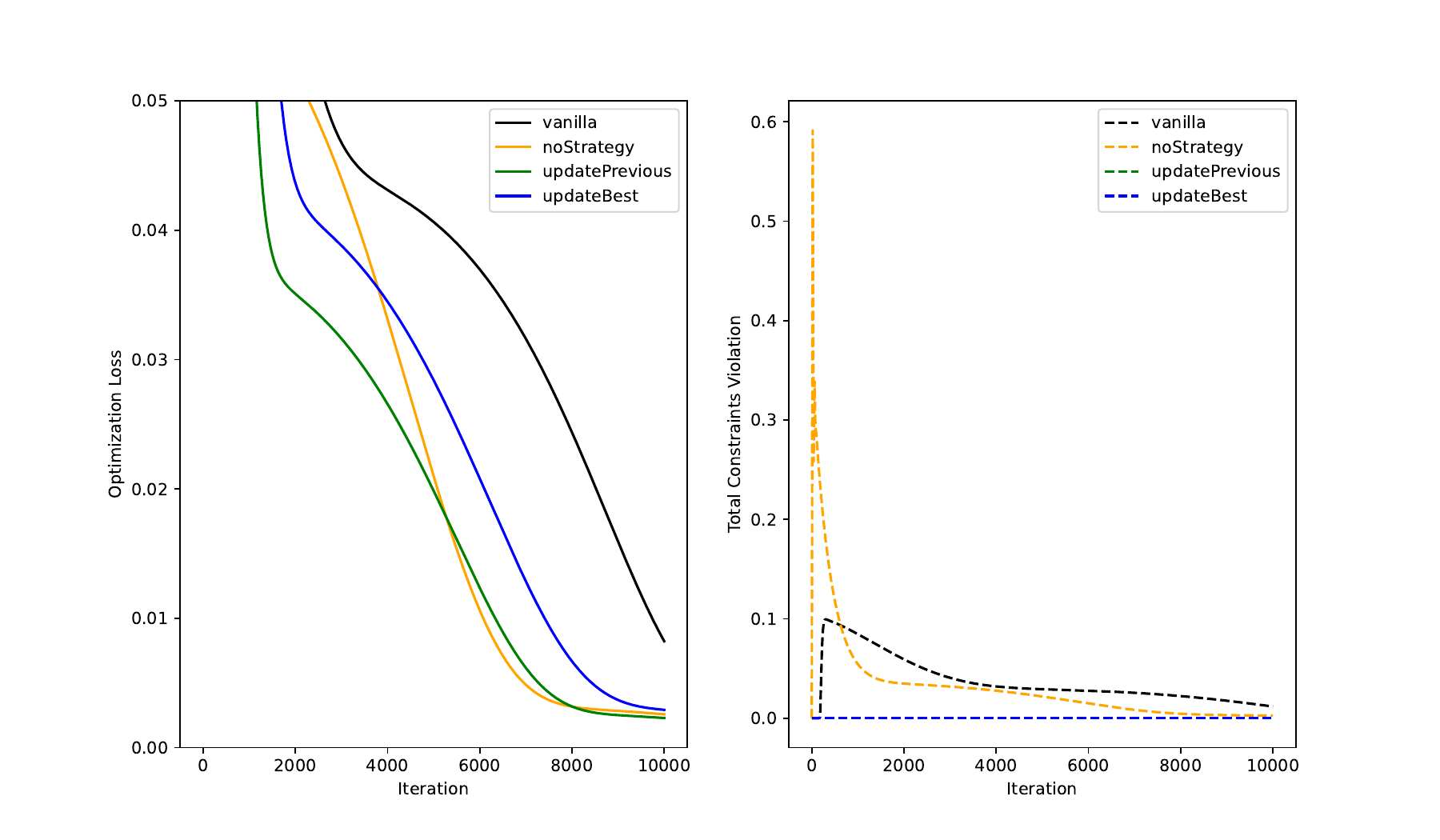} 
        \caption{1E-4}
    \end{subfigure}
    \hskip -2ex
    \begin{subfigure}[h]{0.5\textwidth}
        \centering
        \includegraphics[width=\textwidth]{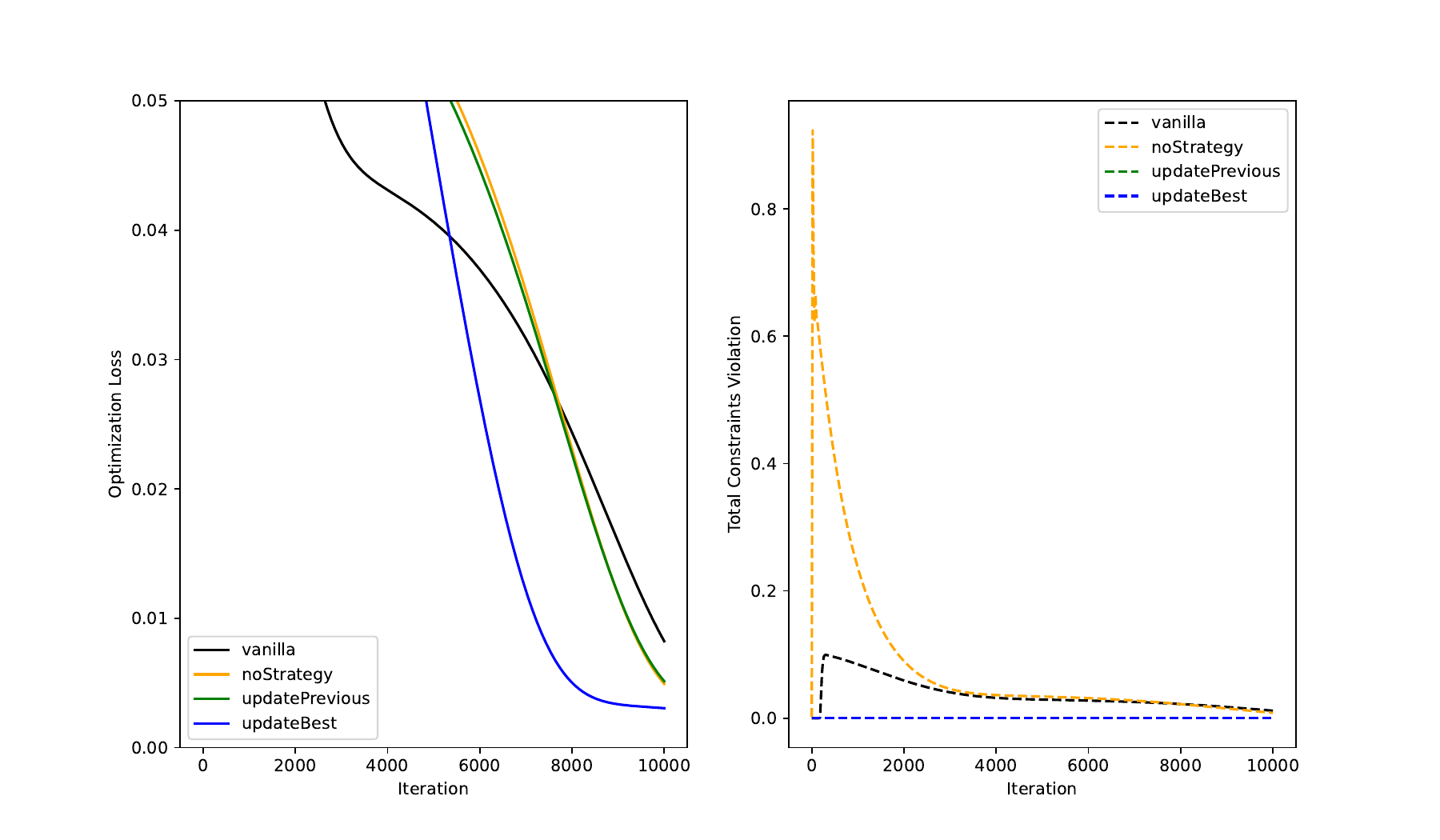} 
        \caption{1E-6}
    \end{subfigure}
    \hskip -2ex
    \begin{subfigure}[h]{0.5\textwidth}
        \centering
        \includegraphics[width=\textwidth]{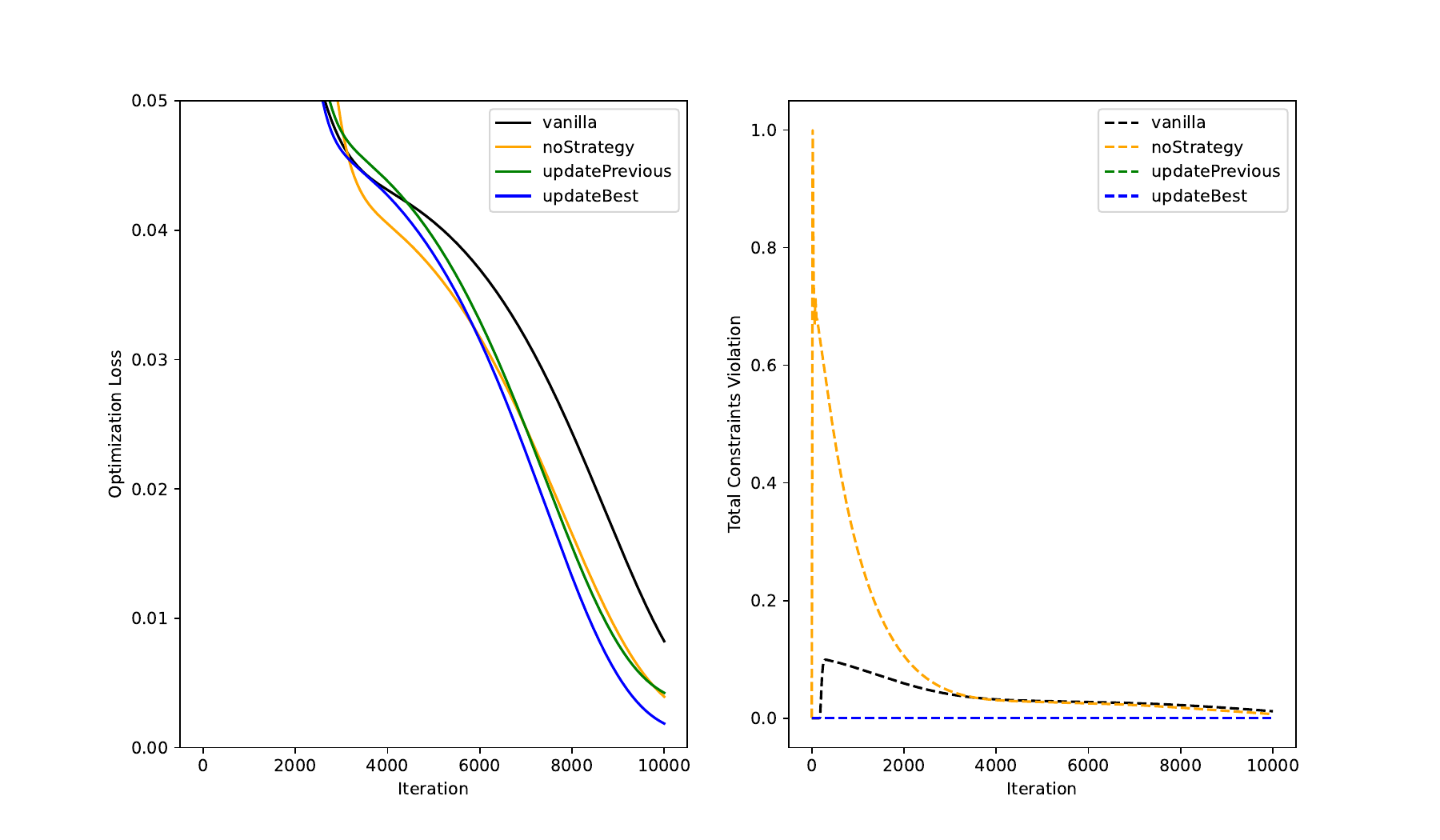}
        \caption{1E-8}
    \end{subfigure}
    \caption{Plots of loss (left) and constraints violation (right), during admissibility stage,  for the various tolerance values during training of the models used in experiments 2.1 and 3.1.}
\label{fig:fewer}
\end{figure}

Figure \ref{fig:fewer} provides a visual representation of the plots pertaining to models trained with sparser data. Once again, the two-stage training method exhibits quicker reductions in losses, except for the scenario involving a tolerance of 1E-2. Notably, unlike the observations in Figure \ref{fig:normal}, the \emph{vanilla} Neural ODE concludes its training with markedly higher loss values. This shows that training with fewer data significantly impacts the training process of these models.

Similarly to the findings in Figure \ref{fig:normal}, the instability of constraints violation during the optimization phase persists when employing the \emph{vanilla} Neural ODE approach and the two-stage method employing the \emph{noStrategy}.

\begin{figure}[]
\begin{subfigure}[h]{0.5\textwidth}
        \centering
        \includegraphics[width=\textwidth]{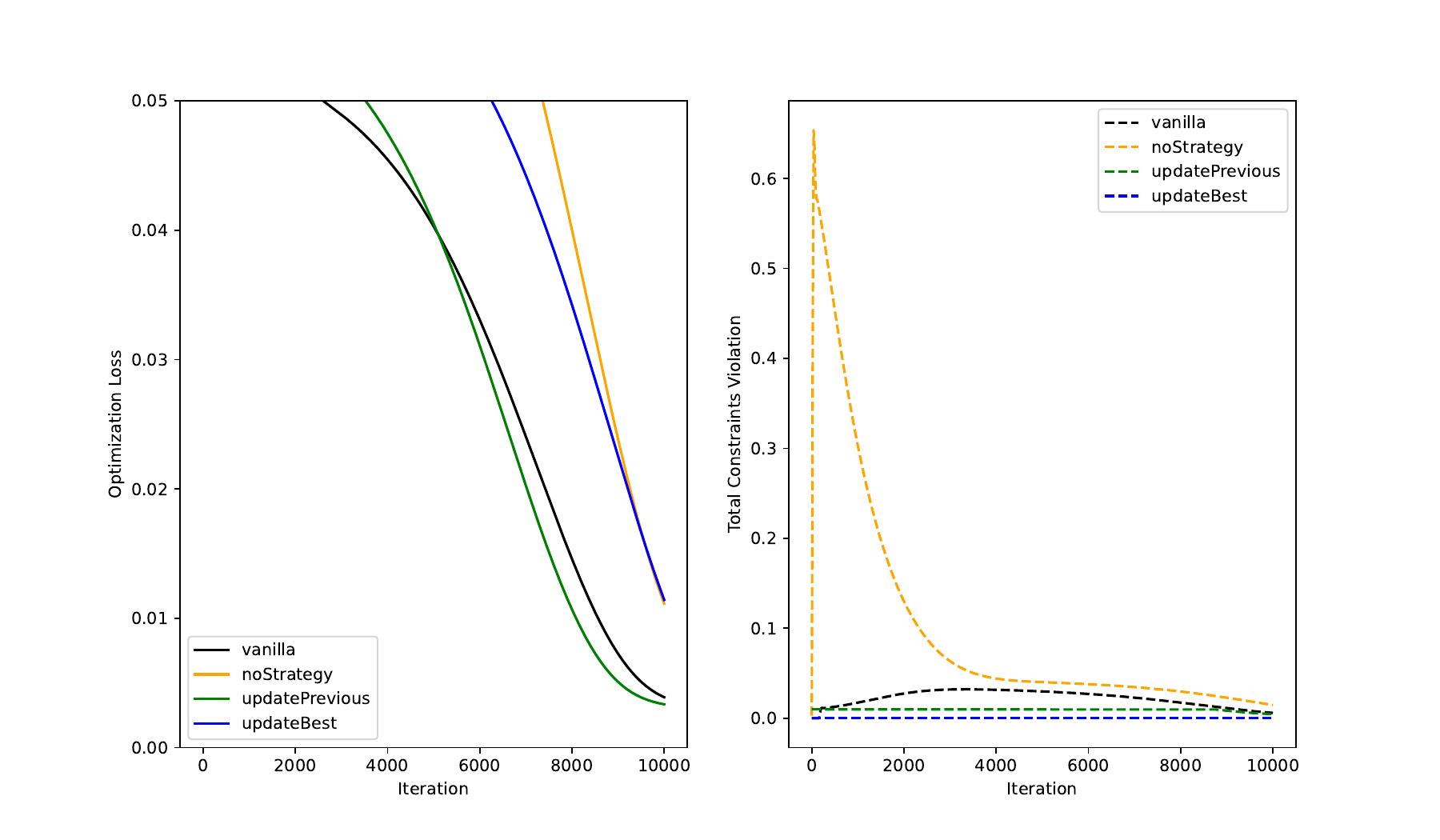} 
        \caption{1E-2}
    \end{subfigure}
    \hskip -2ex
    \begin{subfigure}[h]{0.5\textwidth}
        \centering
        \includegraphics[width=\textwidth]{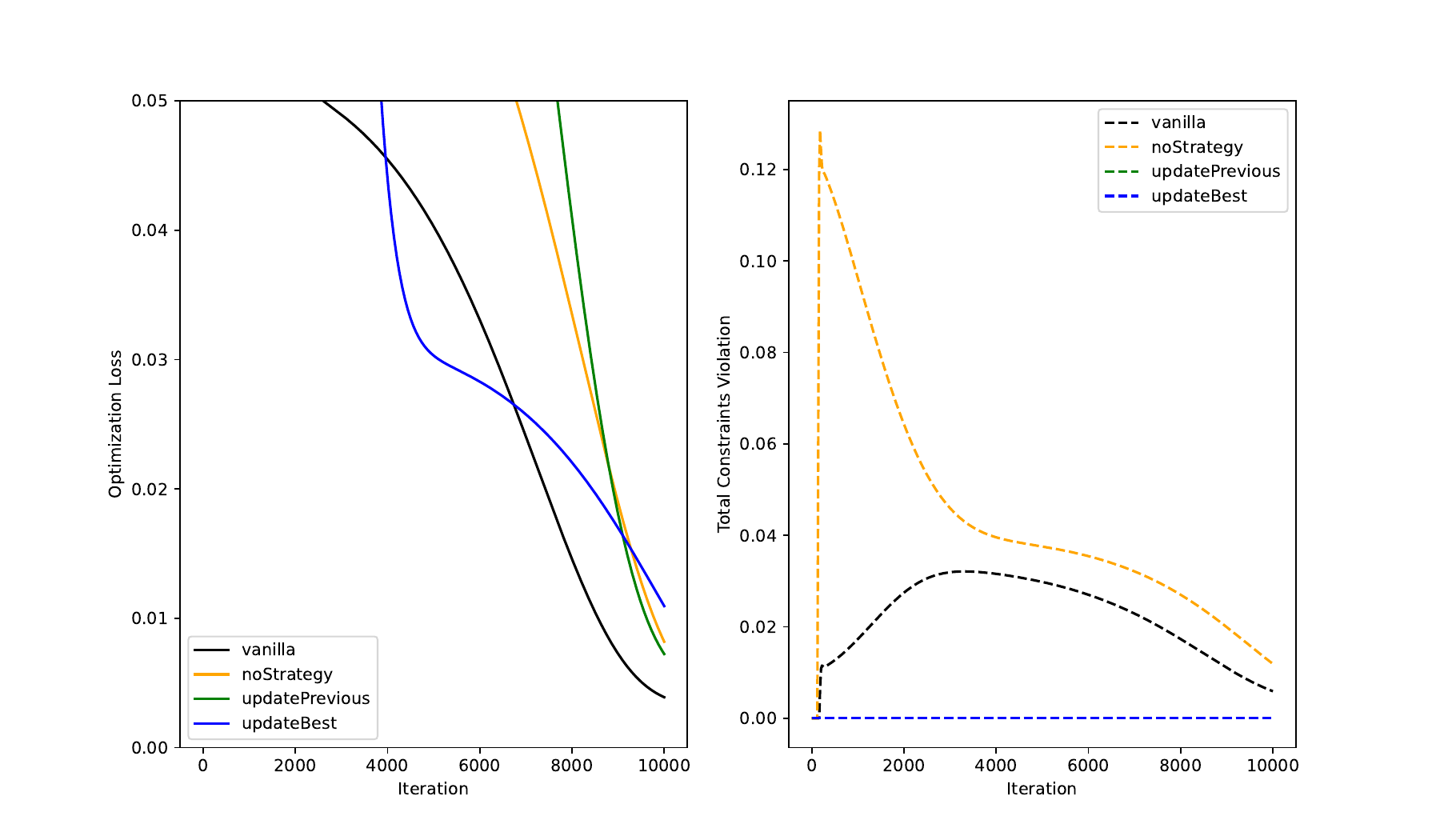} 
        \caption{1E-4}
    \end{subfigure}
    \hskip -2ex
    \begin{subfigure}[h]{0.5\textwidth}
        \centering
        \includegraphics[width=\textwidth]{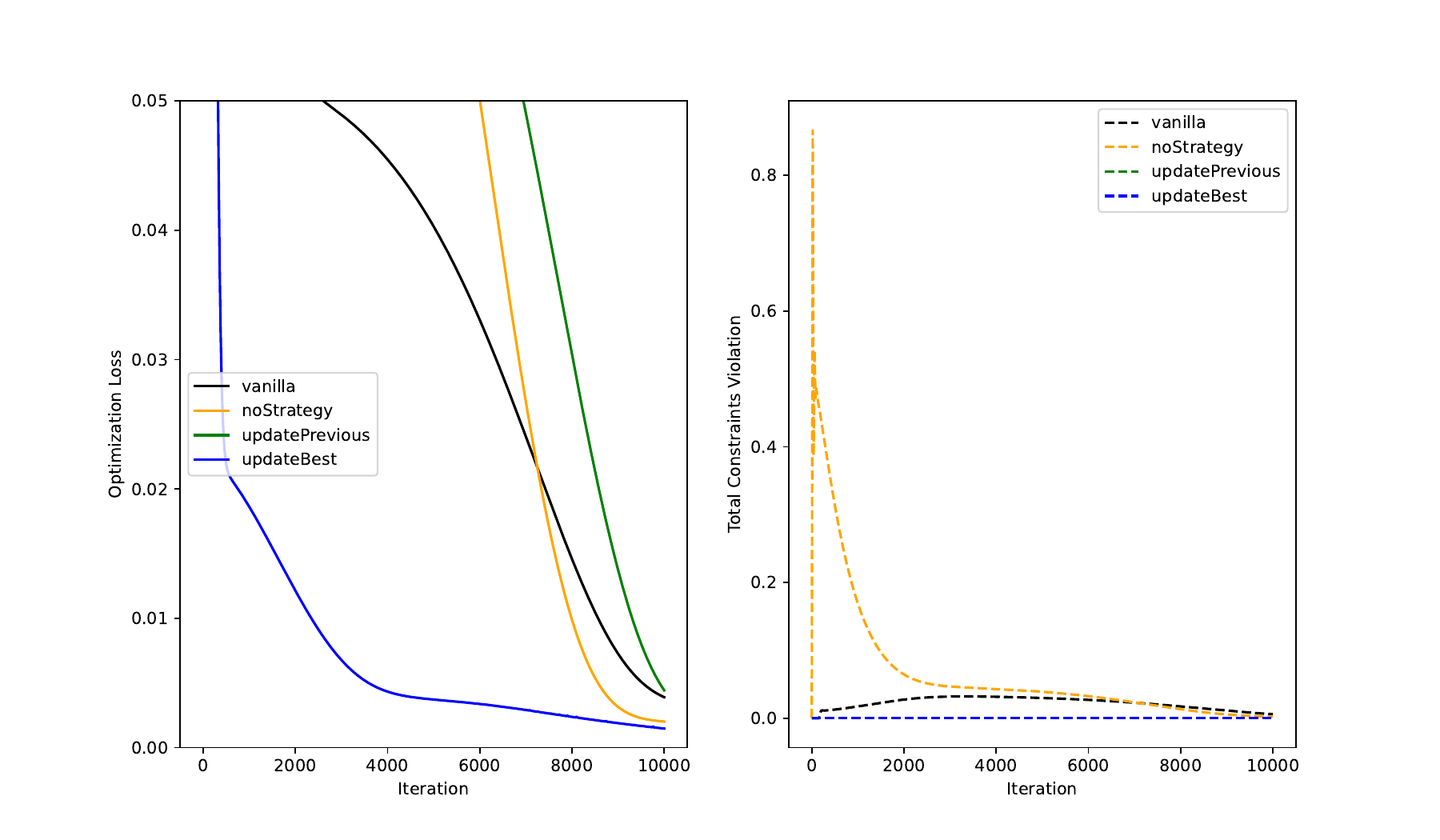} 
        \caption{1E-6}
    \end{subfigure}
    \hskip -2ex
    \begin{subfigure}[h]{0.5\textwidth}
        \centering
        \includegraphics[width=\textwidth]{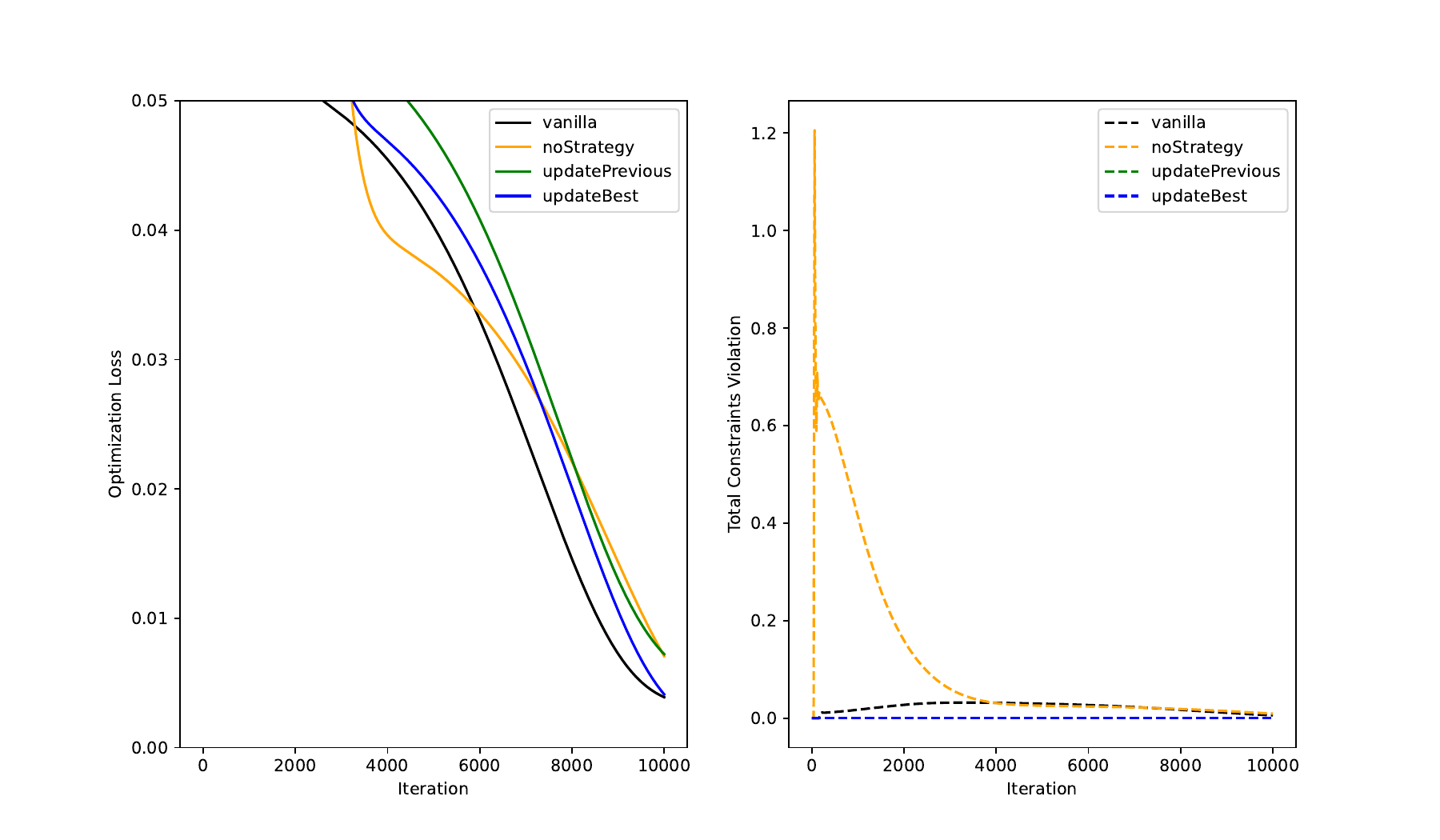}
        \caption{1E-8}
    \end{subfigure}
    \caption{Plots of loss (left) and constraints violation (right), during admissibility stage,  for the various tolerance values during training of the models used in experiments 2.2 and 3.2.}
\label{fig:more}
\end{figure}

Figure \ref{fig:more} provides a visual representation of the optimization stage of models trained with abundant data. In general, comparing to the models trained with ``normal'' and sparser data, Figure \ref{fig:normal} and Figure \ref{fig:fewer} respectively, the rate of decrease in loss is similar to between \emph{vanilla} Neural ODE and the two-stage training method models. It's noteworthy that the \emph{updateBest} strategy presents a faster decrease, outpacing all others, for a tolerance of 1E-6 and reaches the lowest value.

As for the constraints violation during the training process, the \emph{vanilla} Neural ODE initially experiences a significant increase, subsequently transitioning into a diminishing trend. A parallel behavior is observed in the case of the \emph{noStrategy}. Again, the \emph{preference point} strategy avoided an increase of the constraints violation during the optimization stage.

Gathering the conclusions taken from the results in Table \ref{tab:performanceWPG} and the findings from Figures \ref{fig:normal}-\ref{fig:more} we can conclude that the proposed two-stage training method yields models with  higher predictive performance and provides faster convergence to an optimal solution. The efficacy of our method is particularly evidenced in scenarios involving sparser data. It demonstrates robust performance and maintains the quality of the resulting models even when trained with limited data points.
The introduction of the \emph{preference point} strategy shows to be effective in keeping the solutions inside the feasible region during the optimization stage and provide lower $\text{V}_\text{avg}$ values during training and testing.
In general, selecting a tolerance value of 1E-4 shows to be a favorable choice.

\subsubsection{Chemical Reaction}

Figures \ref{fig:normalCR} - \ref{fig:moreCR} aggregate the plots for the CR dataset for the various tolerance values for the models trained with: $100$ points in the time interval $(0,100)$ used in experiments 1.0, 2.0 and 3.0, Figure \ref{fig:normalCR}; $50$ points in the time interval $(0,100)$ used in experiments 2.1 and 3.1, Figure \ref{fig:fewerCR}; $150$ points in the time interval $(0,100)$ used in experiments 2.2 and 3.2, Figure \ref{fig:moreCR}. The best runs at the extrapolation experiments (2.0, 2.1 and 2.2) for each tolerance were chosen to be plotted.

From the plots in Figure \ref{fig:normalCR} we can see that, unlike the WPG dataset, the loss $\mathcal{L_{II}}$ does not decrease faster with the two-stage method.
This observation corroborates with the numerical results in Table \ref{tab:performanceCR} showing that the \emph{vanilla} Neural ODE models have worse generalization capabilities when contrasted with models resulting from the two-stage method.

Upon scrutiny of the evolution of the total constraints violation during the admissibility stage, it is evident that the two-stage method achieves smaller values faster. By employing the \emph{preference point} strategy, it becomes possible to avert deterioration, while its absence, \emph{noStrategy}, results in oscillations. It is worth noting that for $tol=$1E-4, both the \emph{updatePrevious} and \emph{updateBest} strategies not only prevent the deterioration of satisfiability but also further reduce its magnitude.

\begin{figure}[]
\begin{subfigure}[h]{0.5\textwidth}
        \centering
        \includegraphics[width=\textwidth]{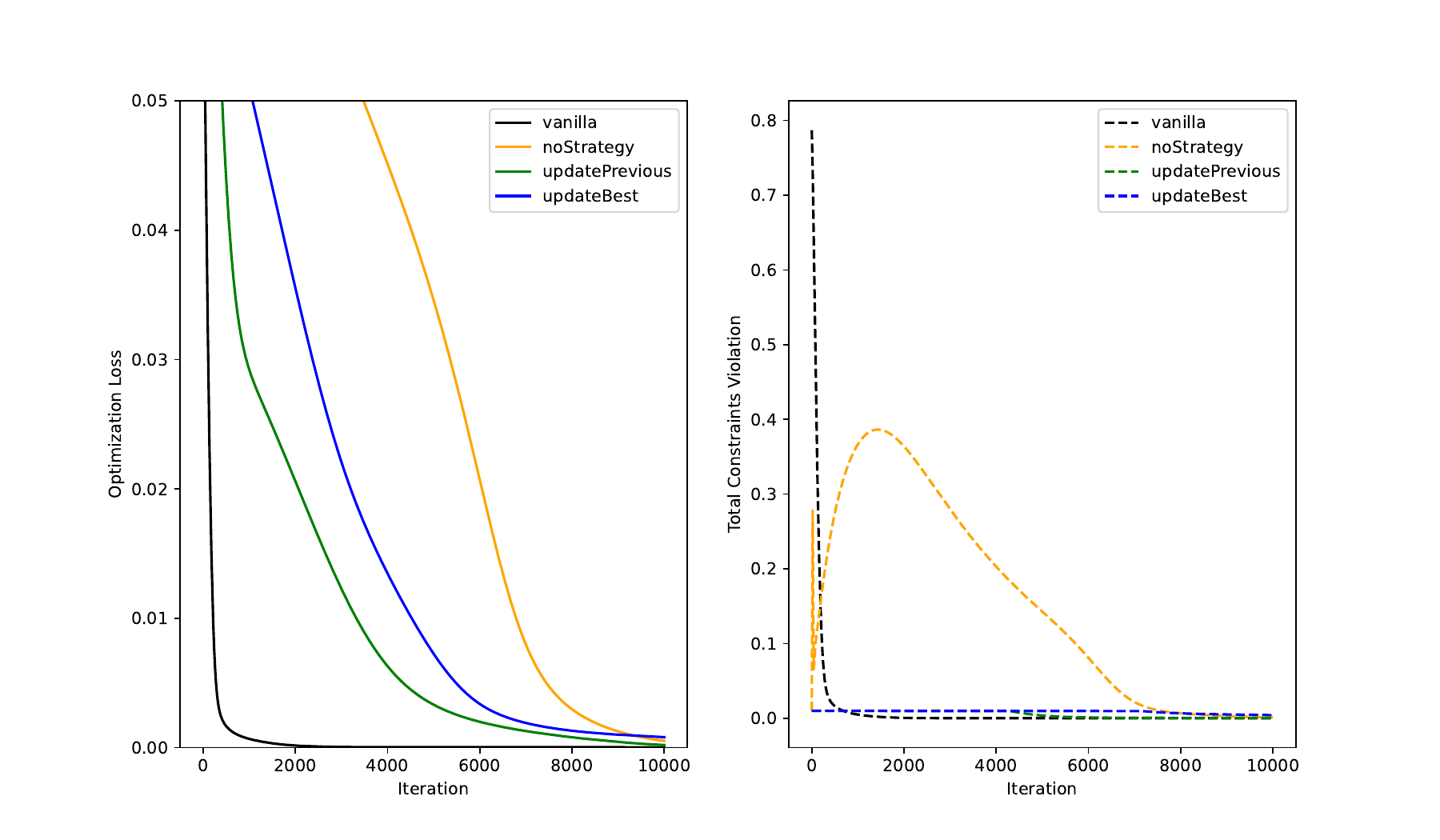} 
        \caption{1E-2}
    \end{subfigure}
    \hskip -2ex
    \begin{subfigure}[h]{0.5\textwidth}
        \centering
        \includegraphics[width=\textwidth]{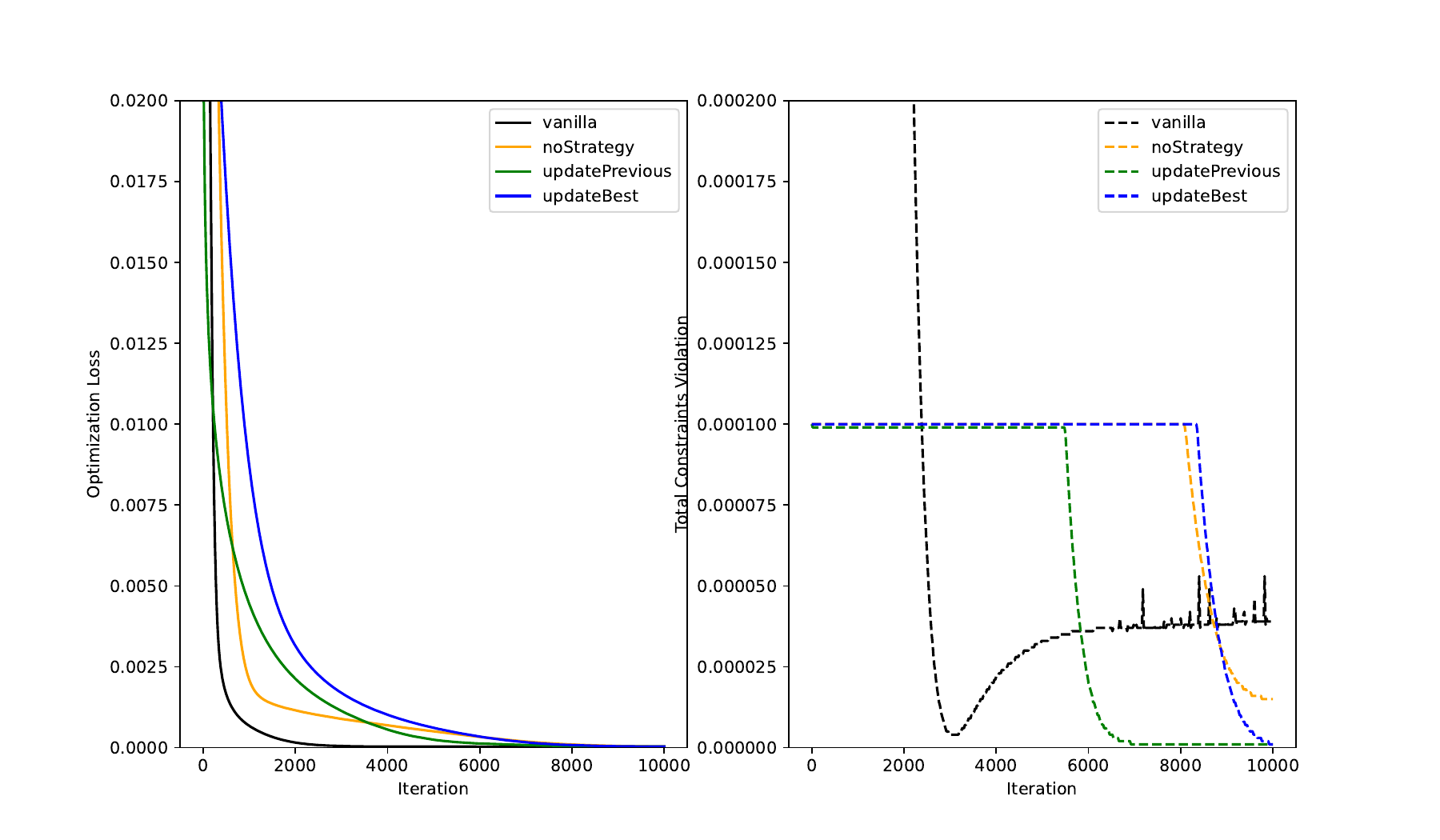}
        \caption{1E-4}
    \end{subfigure}
    \hskip -2ex
        \begin{subfigure}[h]{0.5\textwidth}
        \centering
        \includegraphics[width=\textwidth]{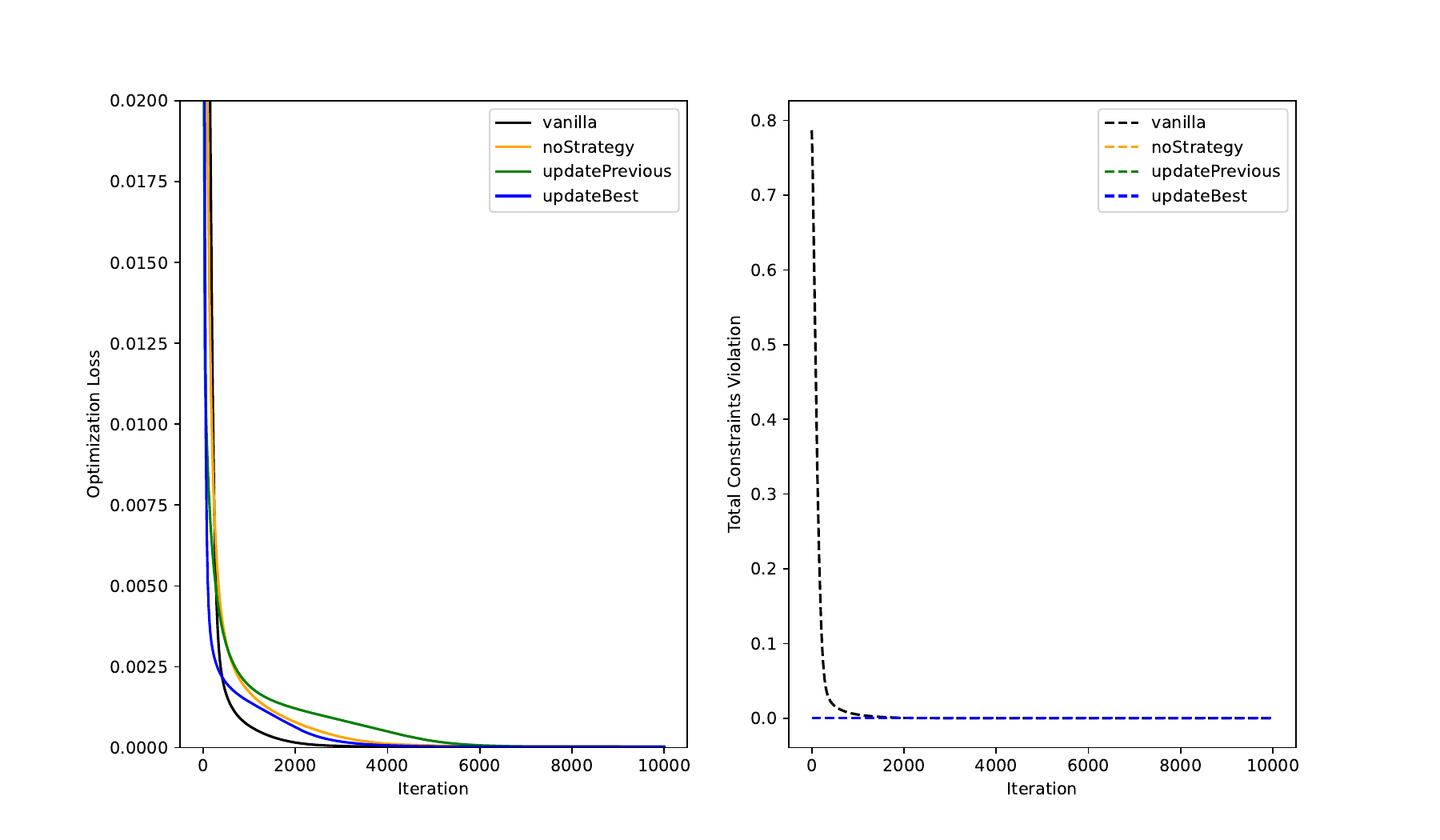}
        \caption{1E-6}
    \end{subfigure}
    \hskip -2ex
    \begin{subfigure}[h]{0.5\textwidth}
        \centering
        \includegraphics[width=\textwidth]{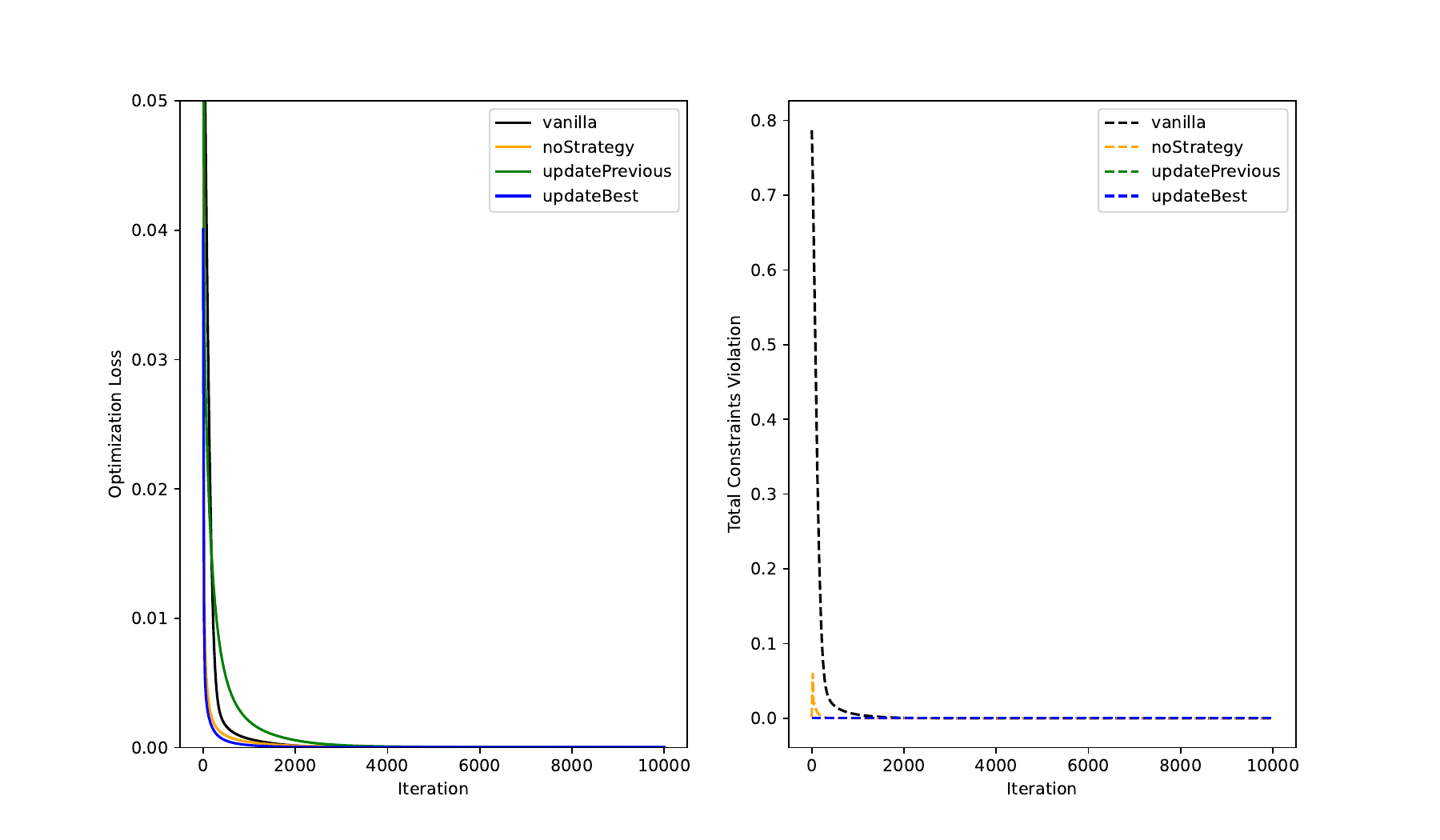}
        \caption{1E-8}
    \end{subfigure}
    \caption{Plots of loss (left) and constraints violation (right), during admissibility stage, for the various tolerance values during training of the models used in experiments 1.0, 2.0 and 3.0.}
    \label{fig:normalCR}
\end{figure}

When training with a sparser dataset (experiments 2.1 and 3.1), Figure \ref{fig:fewerCR} shows that the two-stage method exhibits faster convergence and reaches lower values than \emph{vanilla} Neural ODE. The inclusion of the \emph{preference point} strategies in the two-stage method brings the best training performance.
The dynamics of total constraint violation evolution is similar to Figure \ref{fig:normalCR}. Here, the \emph{updatePrevious} and \emph{updateBest} strategies stand out by delivering optimal performance, preventing deterioration of satisfiability and contributing to a further reduction in feasibility.

\begin{figure}[]
\begin{subfigure}[h]{0.5\textwidth}
        \centering
        \includegraphics[width=\textwidth]{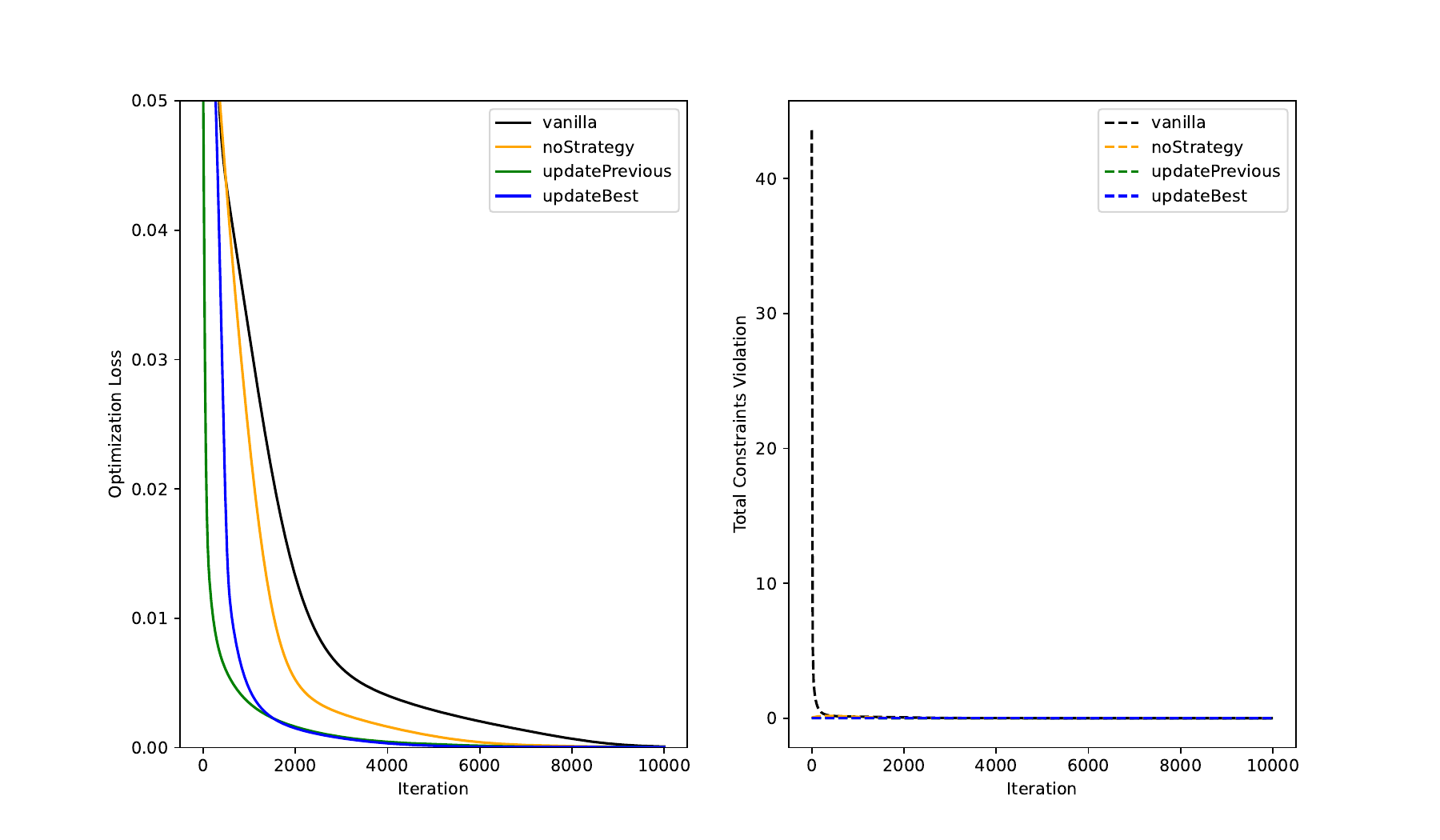} 
        \caption{1E-2}
    \end{subfigure}
    \hskip -2ex
    \begin{subfigure}[h]{0.5\textwidth}
        \centering
        \includegraphics[width=\textwidth]{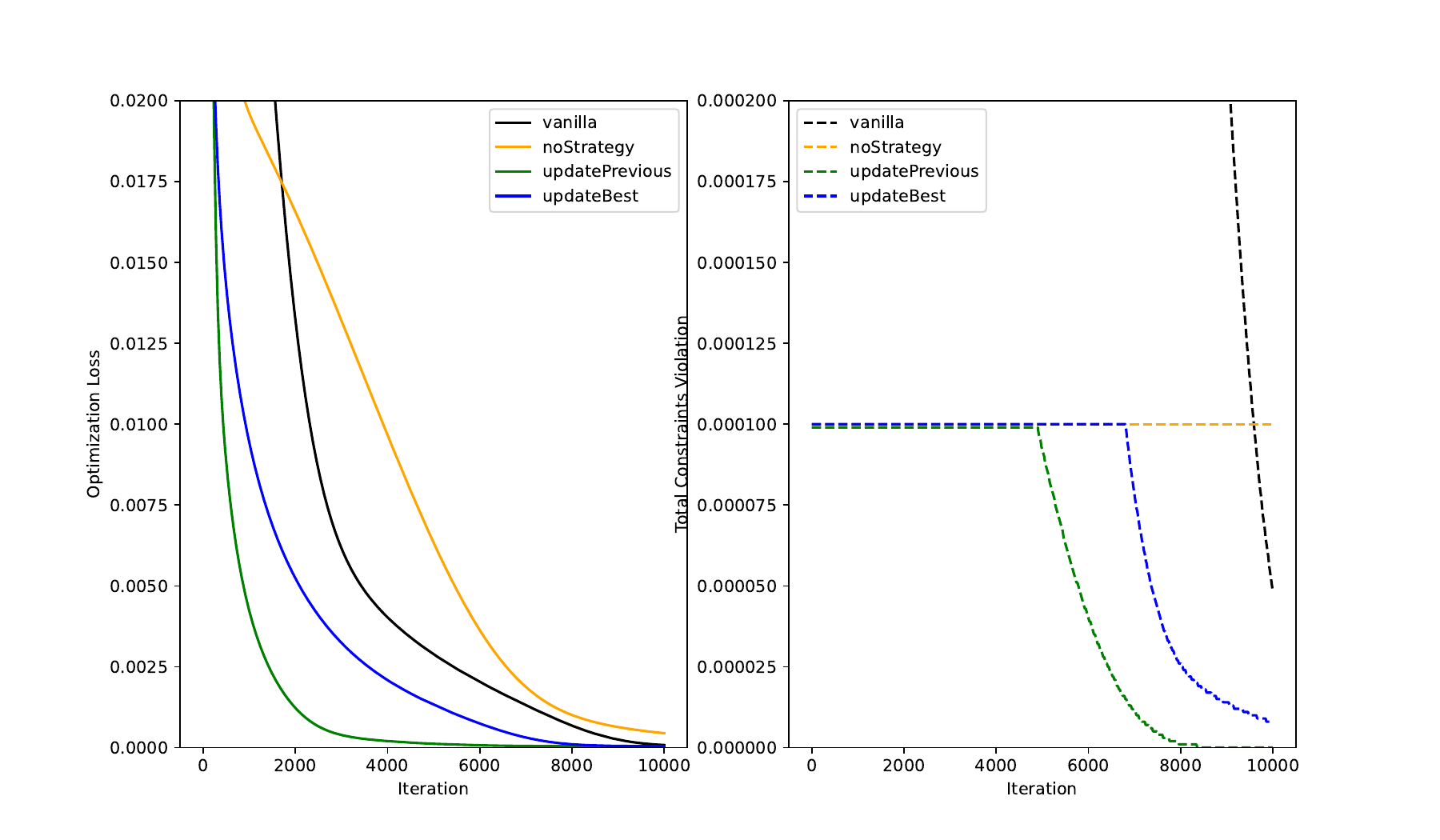} 
        \caption{1E-4}
    \end{subfigure}
    \hskip -2ex
        \begin{subfigure}[h]{0.5\textwidth}
        \centering
        \includegraphics[width=\textwidth]{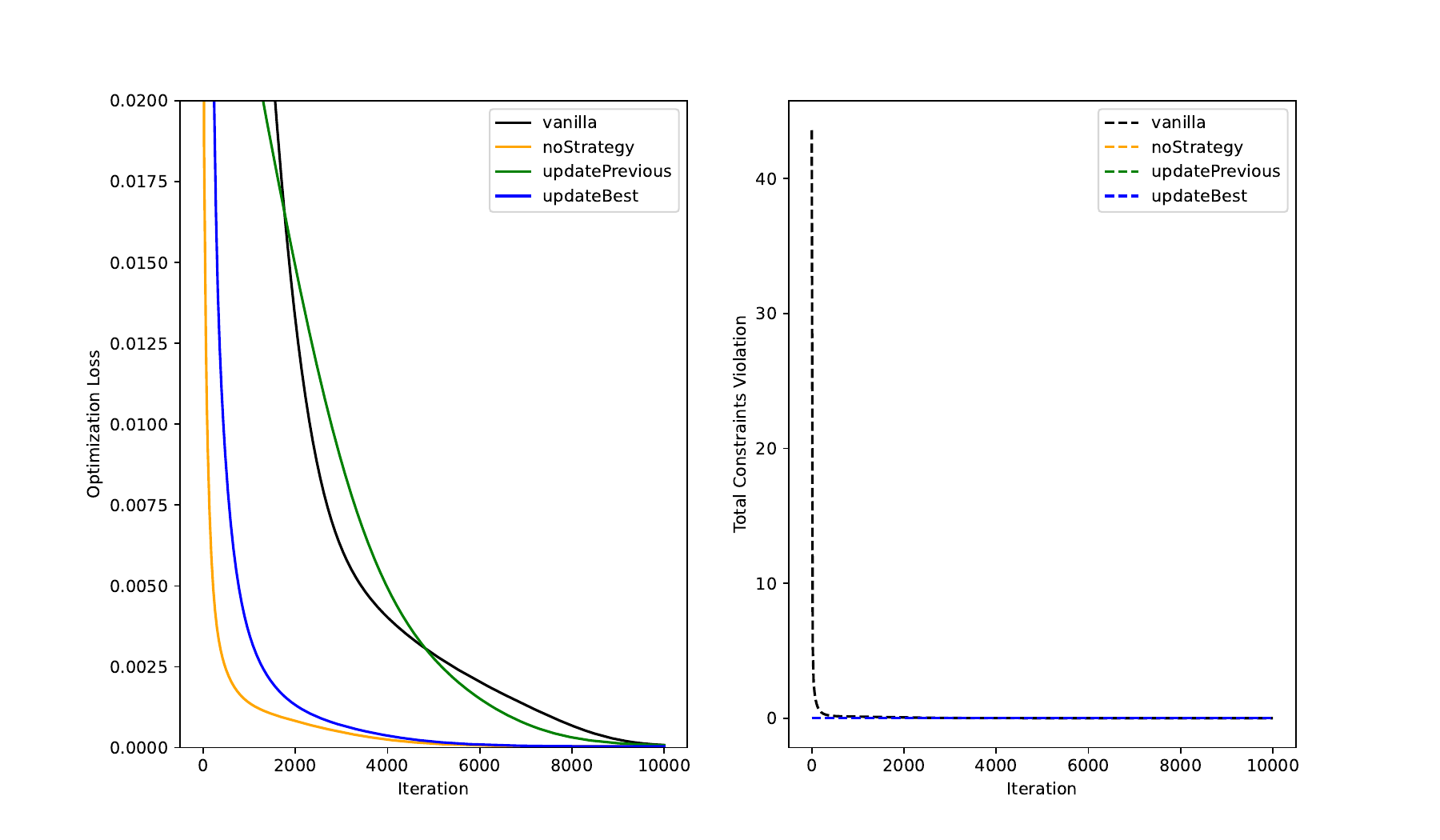}
        \caption{1E-6}
    \end{subfigure}
    \hskip -2ex
    \begin{subfigure}[h]{0.5\textwidth}
        \centering
        \includegraphics[width=\textwidth]{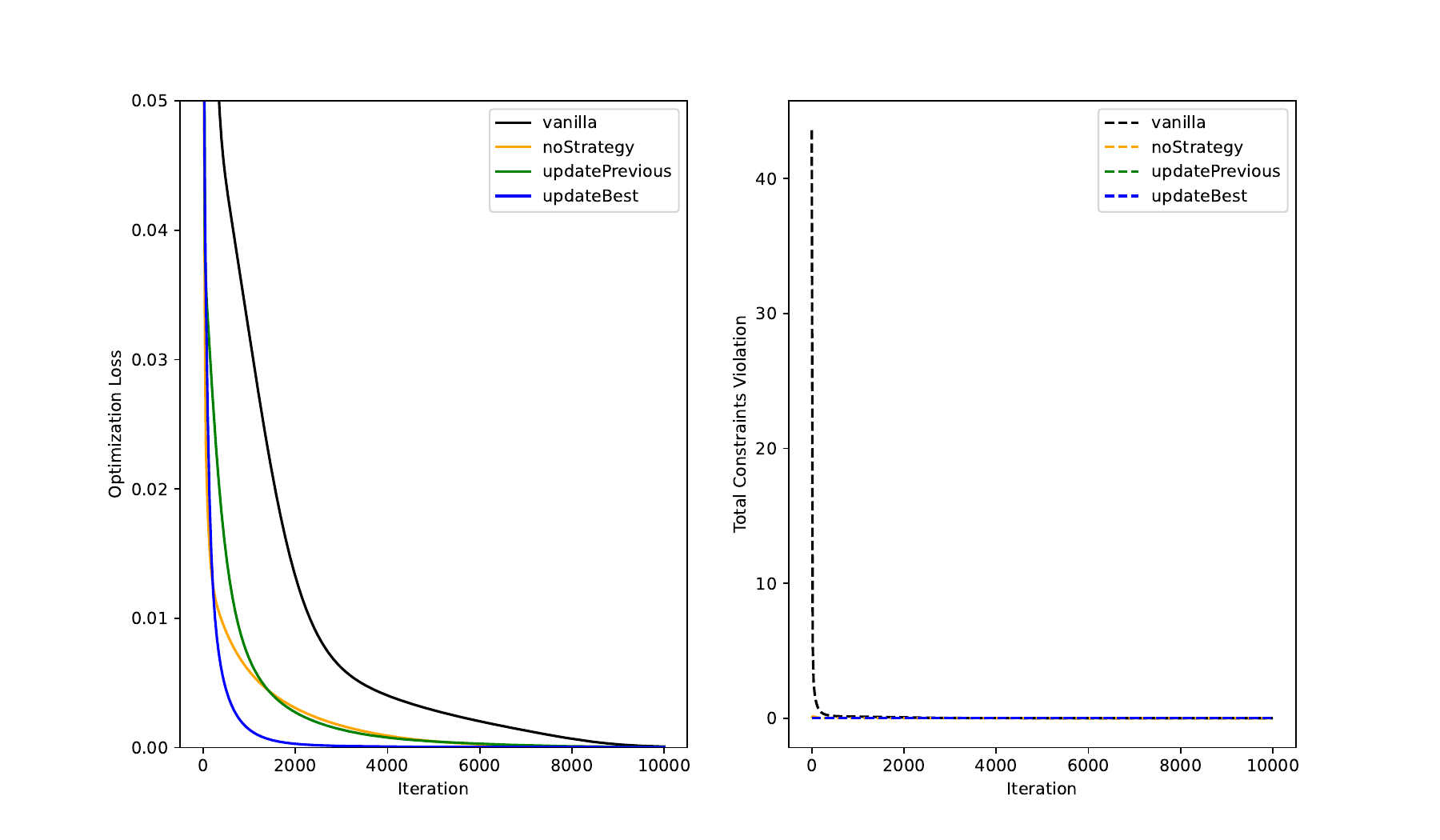}
        \caption{1E-8}
    \end{subfigure}
    \caption{Plots of loss (left) and constraints violation (right), during admissibility stage, for the various tolerance values during training of the models used in experiments 2.1 and 3.1.}
    \label{fig:fewerCR}
\end{figure}

From Figure \ref{fig:moreCR}, training with an abundant training set does not promote significant changes to the loss landscapes with the two-stage methods again offering best performance in terms of both training loss and total constraints violation. Notice however that for lower tolerance values $tol=$1E-2 and $tol=$1E-4, training with an abundant dataset seemed to cause harm to the performance of the \emph{updatePrevious} strategy resulting in slower convergence and even not being able to improve the feasibility, which was possible with the sparser (Figure \ref{fig:fewerCR}) and ``normal" (Figure \ref{fig:normalCR}) datasets.

\begin{figure}[]
\begin{subfigure}[h]{0.5\textwidth}
        \centering
        \includegraphics[width=\textwidth]{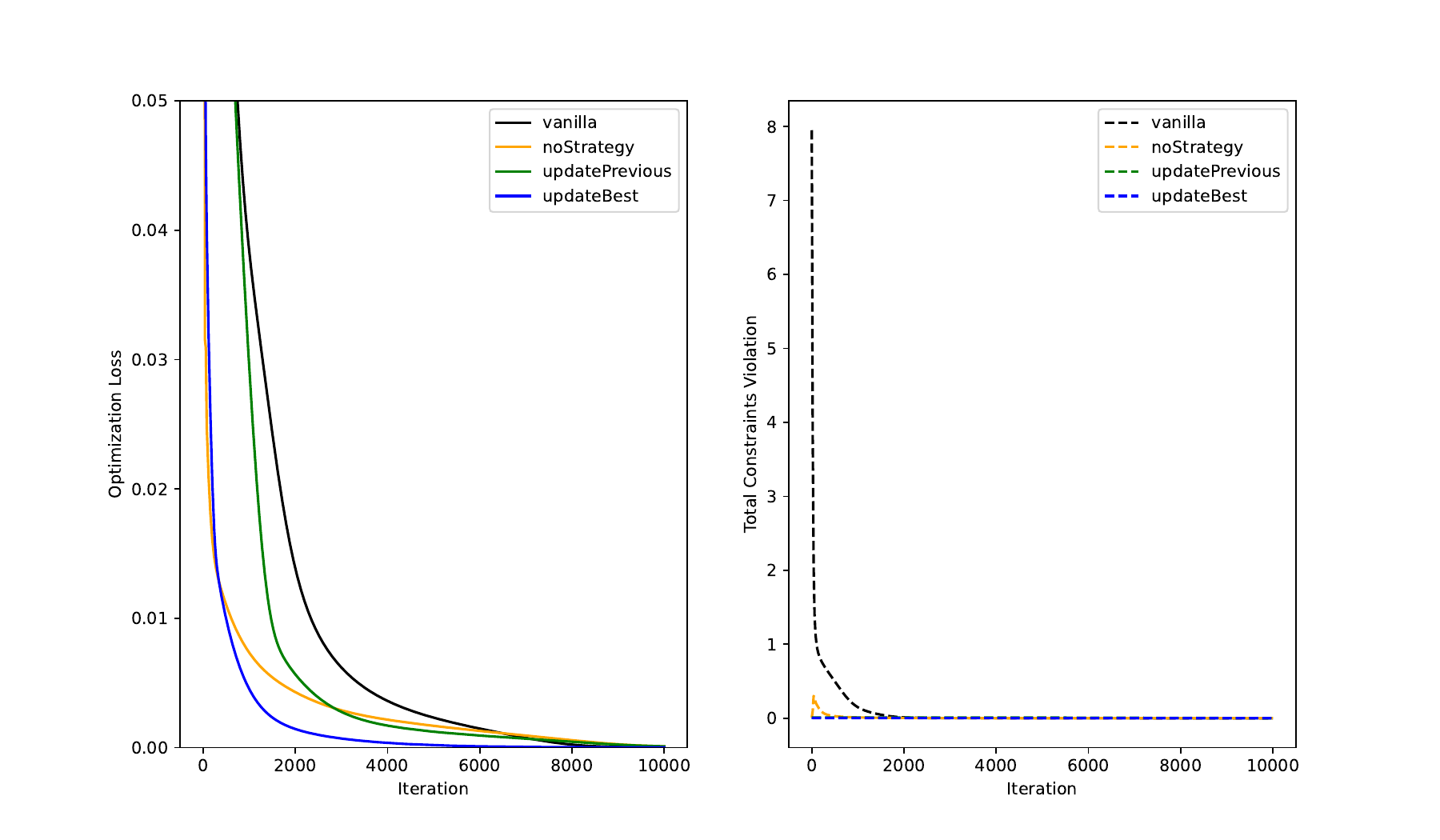} 
        \caption{1E-2}
    \end{subfigure}
    \hskip -2ex
    \begin{subfigure}[h]{0.5\textwidth}
        \centering
        \includegraphics[width=\textwidth]{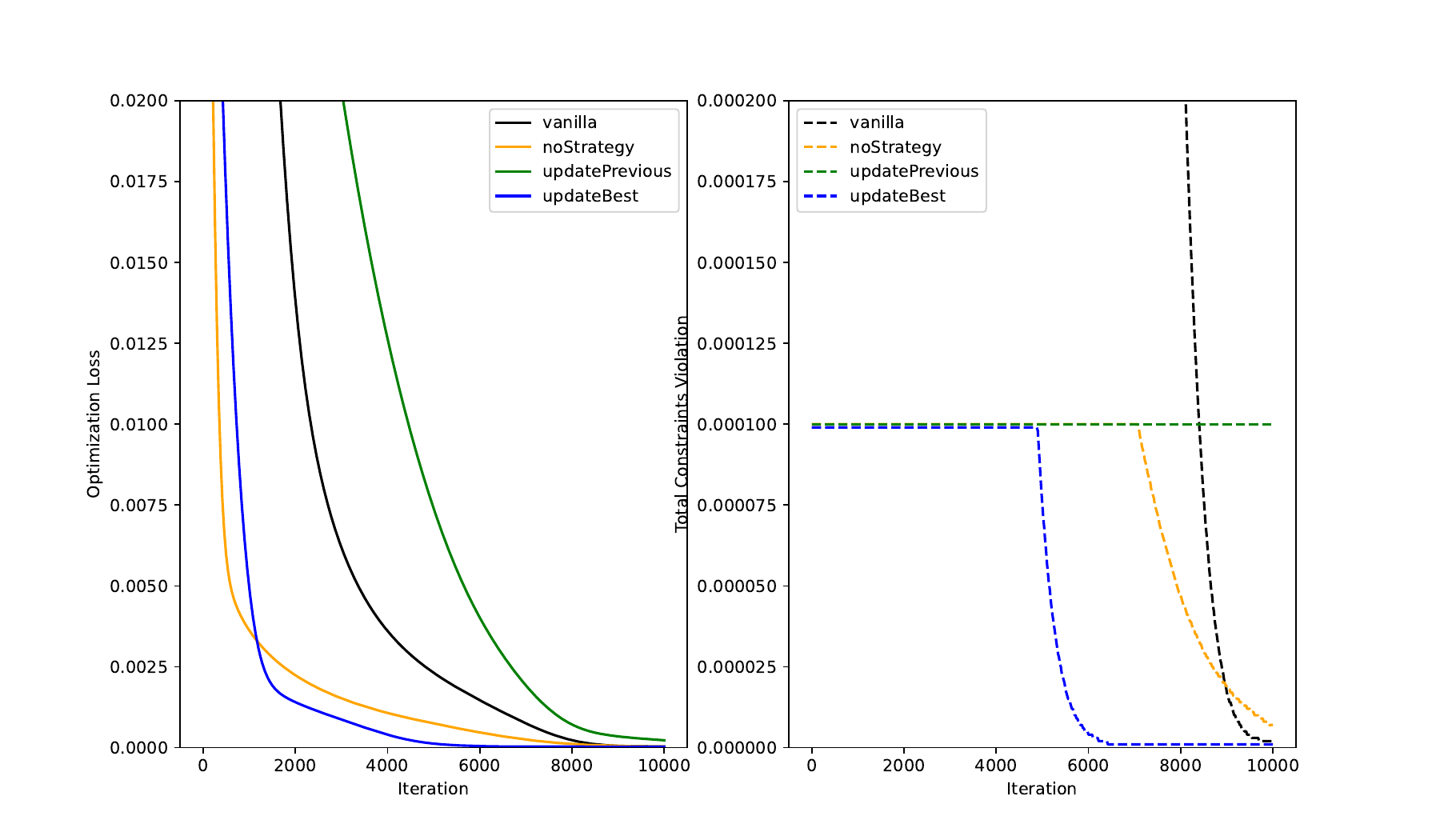} 
        \caption{1E-4}
    \end{subfigure}
    \hskip -2ex
            \begin{subfigure}[h]{0.5\textwidth}
        \centering
        \includegraphics[width=\textwidth]{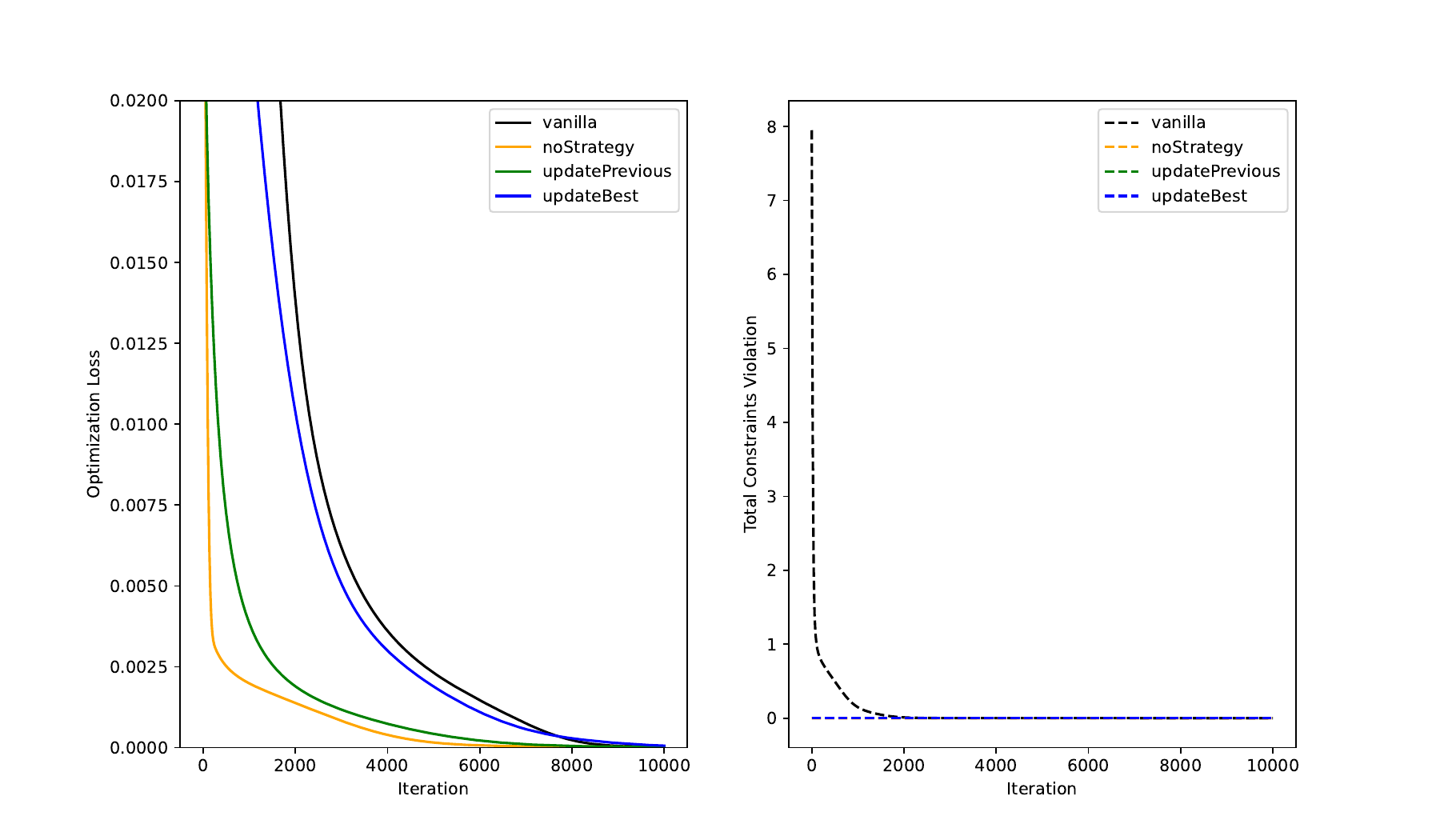}
        \caption{1E-6}
    \end{subfigure}
    \hskip -2ex
    \begin{subfigure}[h]{0.5\textwidth}
        \centering
        \includegraphics[width=\textwidth]{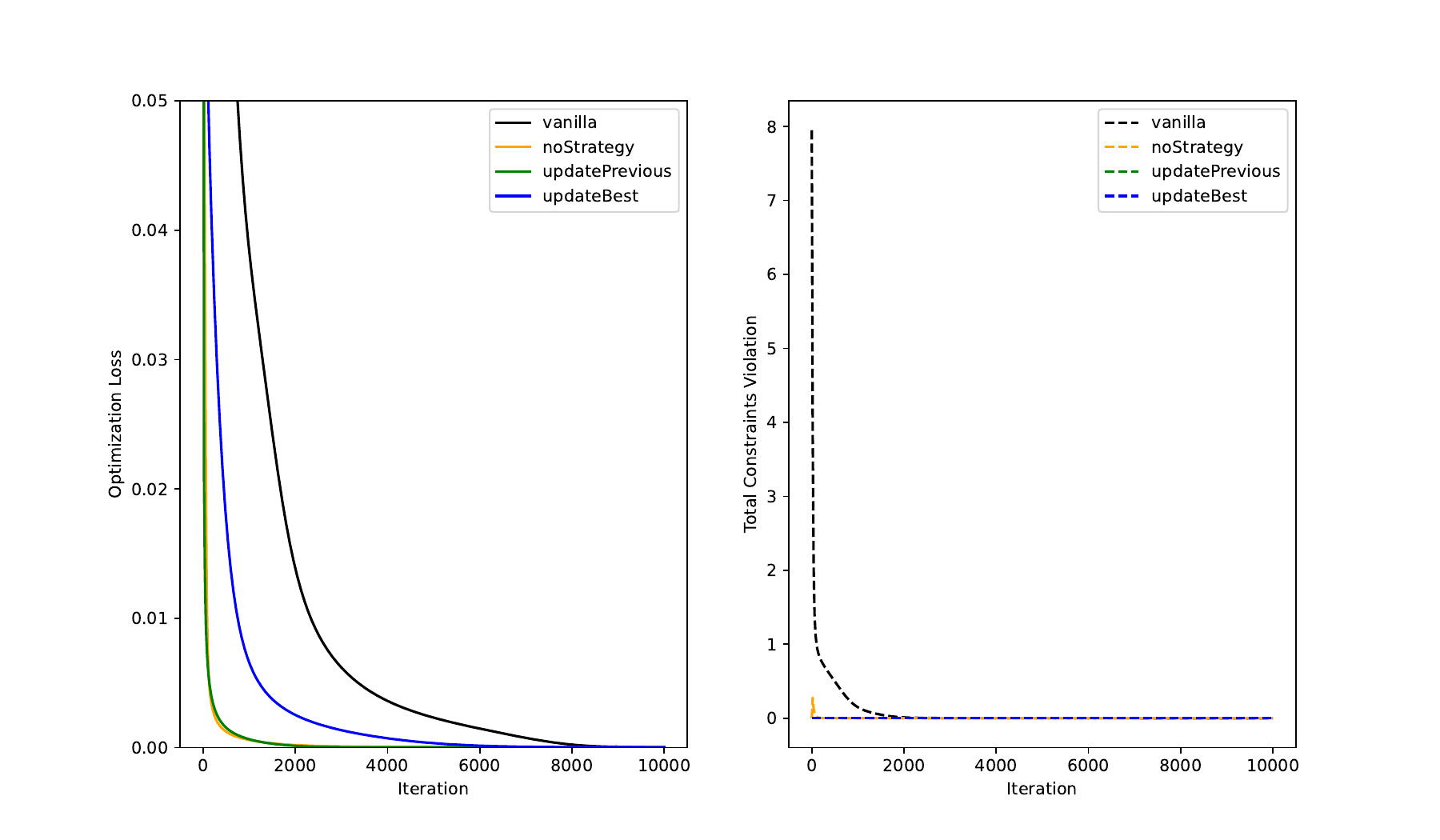}
        \caption{1E-8}
    \end{subfigure}
    \caption{Plots of loss (left) and constraints violation (right), during admissibility stage, for the various tolerance values during training of the models used in experiments 2.2 and 3.2.}
    \label{fig:moreCR}
\end{figure}

Combining the insights from Figures \ref{fig:normalCR}-\ref{fig:moreCR} with the numerical results in Table \ref{tab:performanceCR}, it is evident that, similar to the observations with the WPG dataset, the proposed two-stage method consistently outperforms the \emph{vanilla} Neural ODE, with lower training loss and total constraints violation values. It is noteworthy that, when training with the ``normal" training set, the training loss dynamics of the \emph{vanilla} Neural ODE may suggest improved performance. However, this is contradicted by the testing results in Table \ref{tab:performanceCR}, revealing a lack of robustness and generalization.

The size of the training dataset does not impact the performance of models resulting from the proposed two-stage method. Conversely, the \emph{vanilla} Neural ODE experiences deterioration when trained with a sparser dataset, as it can be seen in Figure \ref{fig:fewerCR}.

In general, the two-stage method with the \emph{preference point} strategy yields optimal performance in generating models. 
Note that, in certain experiments, the \emph{noStrategy} occasionally outperforms other strategies. 
Nevertheless, the reliability of constraint satisfaction is not guaranteed, unlike the more secure outcomes produced by the \emph{updatePrevious} and \emph{updateBest} strategies.

Furthermore, there is no clear relationship between tolerance values for the admissibility stage stopping criteria and testing outcomes. However, setting $tol=$1E-4 appears to be the best all rounder.

\section{Conclusion}
\label{sec:conclusion}

This paper is a follow-up of our preliminary work \cite{coelho2023prior}, where we introduced a two-stage training method for Neural ODEs aimed at explicitly incorporating prior knowledge constraints into the model.

The proposed two-stage method rewrites constrained optimization problems as two unconstrained sub-problems, solving them sequentially during the Neural ODE optimization process. In the first stage, the loss function is defined by the total constraints violations, to find a feasible solution of the original constrained problem. Subsequently, the second stage starts with the solution from the first stage and optimizes a loss function given by the original loss function. To keep the optimization process inside the feasible region during the second stage, a \emph{preference point} strategy, featuring two variations, is proposed. This strategy rejects any point that is infeasible or does not improve admissibility, proceeding with either the point with the best admissibility value or the previous iteration point.

The proposed two-stage training method offers several benefits in modeling constrained systems with Neural ODEs. By avoiding the use of penalty parameters and incorporating prior knowledge through separate training stages, the method ensures constraint satisfaction while minimizing the error between predictions and ground-truth. This approach enhances the interpretability of Neural ODE models and yields robust results, as demonstrated by various numerical experiments.

The decoupling of constraints from the architecture renders the proposed method a flexible framework for modeling constrained systems with Neural ODEs. This flexibility allows for the integration of various types of constraints, allowing for a wide range of applications. 

Furthermore, the incorporation of constraints into the modeling process offers an additional advantage of reducing the required amount of training data. By explicitly integrating known constraints into the model, it becomes feasible to leverage existing knowledge to guide the learning process. This constraint-guided approach allows for more efficient and effective training, as the model is steered towards solutions that satisfy the known constraints. As a result, the need for an extensive dataset to capture all aspects of the system's behavior is mitigated, leading to improved model performance and reduced data acquisition efforts. Thus, incorporation of constraints not only ensures compliance with critical system properties but also contributes to optimizing the training process by reducing data amount requirements.

While our work introduces and emphasizes the two-stage method in the context of Neural ODEs, it is important to underscore the flexibility of this approach, extending its applicability to any NN architecture.

\bmhead{Acknowledgments}

The authors acknowledge the funding by Fundação para a Ciência e Tecnologia (Portuguese Foundation for Science and Technology) through CMAT projects UIDB/00013/2020 and UIDP/00013/2020 and the funding by FCT and Google Cloud partnership through projects CPCA-IAC/AV/589164/2023 and CPCA-IAC/AF/589140/2023.

\noindent C. Coelho would like to thank FCT the funding through the scholarship with reference 2021.05201.BD.

This work is also financially supported by national funds through the FCT/MCTES (PIDDAC), under the project 2022.06672.PTDC - iMAD - Improving the Modelling of Anomalous Diffusion and Viscoelasticity: solutions to industrial problems.

\bibliography{library}


\begin{thebibliography}{13}
\ifx \bisbn   \undefined \def \bisbn  #1{ISBN #1}\fi
\ifx \binits  \undefined \def \binits#1{#1}\fi
\ifx \bauthor  \undefined \def \bauthor#1{#1}\fi
\ifx \batitle  \undefined \def \batitle#1{#1}\fi
\ifx \bjtitle  \undefined \def \bjtitle#1{#1}\fi
\ifx \bvolume  \undefined \def \bvolume#1{\textbf{#1}}\fi
\ifx \byear  \undefined \def \byear#1{#1}\fi
\ifx \bissue  \undefined \def \bissue#1{#1}\fi
\ifx \bfpage  \undefined \def \bfpage#1{#1}\fi
\ifx \blpage  \undefined \def \blpage #1{#1}\fi
\ifx \burl  \undefined \def \burl#1{\textsf{#1}}\fi
\ifx \doiurl  \undefined \def \doiurl#1{\url{https://doi.org/#1}}\fi
\ifx \betal  \undefined \def \betal{\textit{et al.}}\fi
\ifx \binstitute  \undefined \def \binstitute#1{#1}\fi
\ifx \binstitutionaled  \undefined \def \binstitutionaled#1{#1}\fi
\ifx \bctitle  \undefined \def \bctitle#1{#1}\fi
\ifx \beditor  \undefined \def \beditor#1{#1}\fi
\ifx \bpublisher  \undefined \def \bpublisher#1{#1}\fi
\ifx \bbtitle  \undefined \def \bbtitle#1{#1}\fi
\ifx \bedition  \undefined \def \bedition#1{#1}\fi
\ifx \bseriesno  \undefined \def \bseriesno#1{#1}\fi
\ifx \blocation  \undefined \def \blocation#1{#1}\fi
\ifx \bsertitle  \undefined \def \bsertitle#1{#1}\fi
\ifx \bsnm \undefined \def \bsnm#1{#1}\fi
\ifx \bsuffix \undefined \def \bsuffix#1{#1}\fi
\ifx \bparticle \undefined \def \bparticle#1{#1}\fi
\ifx \barticle \undefined \def \barticle#1{#1}\fi
\bibcommenthead
\ifx \bconfdate \undefined \def \bconfdate #1{#1}\fi
\ifx \botherref \undefined \def \botherref #1{#1}\fi
\ifx \url \undefined \def \url#1{\textsf{#1}}\fi
\ifx \bchapter \undefined \def \bchapter#1{#1}\fi
\ifx \bbook \undefined \def \bbook#1{#1}\fi
\ifx \bcomment \undefined \def \bcomment#1{#1}\fi
\ifx \oauthor \undefined \def \oauthor#1{#1}\fi
\ifx \citeauthoryear \undefined \def \citeauthoryear#1{#1}\fi
\ifx \endbibitem  \undefined \def \endbibitem {}\fi
\ifx \bconflocation  \undefined \def \bconflocation#1{#1}\fi
\ifx \arxivurl  \undefined \def \arxivurl#1{\textsf{#1}}\fi
\csname PreBibitemsHook\endcsname

\bibitem[\protect\citeauthoryear{Chen
  et~al.}{2018}]{chenNeuralOrdinaryDifferential2019a}
\begin{botherref}
\oauthor{\bsnm{Chen}, \binits{R.T.}},
\oauthor{\bsnm{Rubanova}, \binits{Y.}},
\oauthor{\bsnm{Bettencourt}, \binits{J.}},
\oauthor{\bsnm{Duvenaud}, \binits{D.K.}}:
Neural ordinary differential equations.
Advances in neural information processing systems
\textbf{31}
(2018)
\end{botherref}
\endbibitem

\bibitem[\protect\citeauthoryear{Xing
  et~al.}{2022}]{xingContinuousGlucoseMonitoring2022}
\begin{bchapter}
\bauthor{\bsnm{Xing}, \binits{Y.}},
\bauthor{\bsnm{Ye}, \binits{H.}},
\bauthor{\bsnm{Zhang}, \binits{X.}},
\bauthor{\bsnm{Cao}, \binits{W.}},
\bauthor{\bsnm{Zheng}, \binits{S.}},
\bauthor{\bsnm{Bian}, \binits{J.}},
\bauthor{\bsnm{Guo}, \binits{Y.}}:
\bctitle{A continuous glucose monitoring measurements forecasting approach via
  sporadic blood glucose monitoring}.
In: \bbtitle{2022 {{IEEE International Conference}} on {{Bioinformatics}} and
  {{Biomedicine}} ({{BIBM}})},
pp. \bfpage{860}--\blpage{863}
(\byear{2022}).
\doiurl{10.1109/BIBM55620.2022.9995522}
\end{bchapter}
\endbibitem

\bibitem[\protect\citeauthoryear{Su
  et~al.}{2022}]{suKineticsParameterOptimization2022}
\begin{botherref}
\oauthor{\bsnm{Su}, \binits{X.}},
\oauthor{\bsnm{Ji}, \binits{W.}},
\oauthor{\bsnm{An}, \binits{J.}},
\oauthor{\bsnm{Ren}, \binits{Z.}},
\oauthor{\bsnm{Deng}, \binits{S.}},
\oauthor{\bsnm{Law}, \binits{C.K.}}:
Kinetics {{Parameter Optimization}} via {{Neural Ordinary Differential
  Equations}}
(arXiv:2209.01862)
(2022)
{\href{https://arxiv.org/abs/2209.01862}{{arXiv:2209.01862}}}
{[physics]}
\end{botherref}
\endbibitem

\bibitem[\protect\citeauthoryear{Nocedal and
  Wright}{2006}]{nocedalNumericalOptimization2006}
\begin{bbook}
\bauthor{\bsnm{Nocedal}, \binits{J.}},
\bauthor{\bsnm{Wright}, \binits{S.J.}}:
\bbtitle{Numerical Optimization},
\bedition{2}nd edn.
\bsertitle{Springer Series in Operation Research and Financial Engineering}.
\bpublisher{{Springer}},
\blocation{{New York, NY}}
(\byear{2006})
\end{bbook}
\endbibitem

\bibitem[\protect\citeauthoryear{Coelho et~al.}{2023}]{coelho2023prior}
\begin{bchapter}
\bauthor{\bsnm{Coelho}, \binits{C.}},
\bauthor{\bsnm{Costa}, \binits{M.F.P.}},
\bauthor{\bsnm{Ferr{\'{a}}s}, \binits{L.L.}}:
\bctitle{Prior knowledge meets neural odes: a two-stage training method for
  improved explainability}.
In: \beditor{\bsnm{Maughan}, \binits{K.}},
\beditor{\bsnm{Liu}, \binits{R.}},
\beditor{\bsnm{Burns}, \binits{T.F.}} (eds.)
\bbtitle{The First Tiny Papers Track at {ICLR} 2023, Tiny Papers @ {ICLR} 2023,
  Kigali, Rwanda, May 5, 2023}
(\byear{2023}).
\burl{https://openreview.net/pdf?id=p7sHcNt\_tqo}
\end{bchapter}
\endbibitem

\bibitem[\protect\citeauthoryear{Raissi
  et~al.}{2017}]{raissiPhysicsInformedDeep2017a}
\begin{botherref}
\oauthor{\bsnm{Raissi}, \binits{M.}},
\oauthor{\bsnm{Perdikaris}, \binits{P.}},
\oauthor{\bsnm{Karniadakis}, \binits{G.E.}}:
Physics {{Informed Deep Learning}} ({{Part I}}): {{Data-driven Solutions}} of
  {{Nonlinear Partial Differential Equations}}
(arXiv:1711.10561)
(2017)
\doiurl{10.48550/arXiv.1711.10561}
{\href{https://arxiv.org/abs/1711.10561}{{arXiv:1711.10561}}}
{[cs, math, stat]}
\end{botherref}
\endbibitem

\bibitem[\protect\citeauthoryear{Roehrl
  et~al.}{2020}]{roehrlModelingSystemDynamics2020}
\begin{barticle}
\bauthor{\bsnm{Roehrl}, \binits{M.A.}},
\bauthor{\bsnm{Runkler}, \binits{T.A.}},
\bauthor{\bsnm{Brandtstetter}, \binits{V.}},
\bauthor{\bsnm{Tokic}, \binits{M.}},
\bauthor{\bsnm{Obermayer}, \binits{S.}}:
\batitle{Modeling {{System Dynamics}} with {{Physics-Informed Neural Networks
  Based}} on {{Lagrangian Mechanics}}}.
\bjtitle{IFAC-PapersOnLine}
\bvolume{53}(\bissue{2}),
\bfpage{9195}--\blpage{9200}
(\byear{2020})
\doiurl{10.1016/j.ifacol.2020.12.2182}
{\href{https://arxiv.org/abs/2005.14617}{{arxiv:2005.14617}}}
{[cs, stat]}.
\bcomment{Comment: Accepted for publication at the 21st IFAC World Congress
  2020}
\end{barticle}
\endbibitem

\bibitem[\protect\citeauthoryear{Zhong
  et~al.}{2020}]{zhongDissipativeSymODENEncoding2020}
\begin{botherref}
\oauthor{\bsnm{Zhong}, \binits{Y.D.}},
\oauthor{\bsnm{Dey}, \binits{B.}},
\oauthor{\bsnm{Chakraborty}, \binits{A.}}:
Dissipative {{SymODEN}}: {{Encoding Hamiltonian Dynamics}} with {{Dissipation}}
  and {{Control}} into {{Deep Learning}}.
{arXiv}.
Comment: Published at ICLR 2020 Workshop on Integration of Deep Neural Models
  and Differential Equations (DeepDiffEq)
(2020)
\end{botherref}
\endbibitem

\bibitem[\protect\citeauthoryear{Tuor
  et~al.}{2020}]{tuorConstrainedNeuralOrdinary2020}
\begin{botherref}
\oauthor{\bsnm{Tuor}, \binits{A.}},
\oauthor{\bsnm{Drgona}, \binits{J.}},
\oauthor{\bsnm{Vrabie}, \binits{D.}}:
Constrained neural ordinary differential equations with stability guarantees.
arXiv preprint arXiv:2004.10883
(2020)
\end{botherref}
\endbibitem

\bibitem[\protect\citeauthoryear{Lim and
  Kasim}{2022}]{limUnifyingPhysicalSystems2022}
\begin{barticle}
\bauthor{\bsnm{Lim}, \binits{Y.H.}},
\bauthor{\bsnm{Kasim}, \binits{M.F.}}:
\batitle{Unifying physical systems' inductive biases in neural {{ODE}} using
  dynamics constraints}.
\bjtitle{Trans. Mach. Learn. Res.}
(\byear{2022})
\doiurl{10.48550/arXiv.2208.02632}
\end{barticle}
\endbibitem

\bibitem[\protect\citeauthoryear{Coelho et~al.}{2023a}]{coelho_population_2023}
\begin{botherref}
\oauthor{\bsnm{Coelho}, \binits{C.}},
\oauthor{\bsnm{Costa}, \binits{M.F.P.}},
\oauthor{\bsnm{Ferrás}, \binits{L.L.}}:
World Population Growth.
Kaggle
(2023).
\doiurl{10.34740/KAGGLE/DS/3010437}
\end{botherref}
\endbibitem

\bibitem[\protect\citeauthoryear{Coelho et~al.}{2023b}]{coelho_chemical_2023}
\begin{botherref}
\oauthor{\bsnm{Coelho}, \binits{C.}},
\oauthor{\bsnm{Costa}, \binits{M.F.P.}},
\oauthor{\bsnm{Ferrás}, \binits{L.L.}}:
Synthetic Chemical Reaction.
Kaggle
(2023).
\doiurl{10.34740/KAGGLE/DS/3010478}
\end{botherref}
\endbibitem

\bibitem[\protect\citeauthoryear{Rodrigues and
  Hauser}{2014}]{rodriguesModeloLogisticoVerhulst2014}
\begin{bchapter}
\bauthor{\bsnm{Rodrigues}, \binits{D.D.S.}},
\bauthor{\bsnm{Hauser}, \binits{E.B.}}:
\bctitle{{Modelo Log\'istico de Verhulst e M\'etodos Num\'ericos na An\'alise
  do Censo Populacional Mundial}}.
In: \bbtitle{{CMAC Sul \textendash{} Congresso de Matem\'atica Aplicada e
  Computacional}}
(\byear{2014}).
\doiurl{10.5540/03.2014.002.01.0068}
\end{bchapter}
\endbibitem

\end{thebibliography}

\newpage    

\appendix

\section{Developed Datasets} \label{app:datasets}

To evaluate the effectiveness of the proposed two-stage method, we specifically created two datasets of constrained systems to be used in \cite{coelho2023prior}: World Population Growth (WPG) \cite{coelho_population_2023} and Chemical Reaction (CR) \cite{coelho_chemical_2023}. These datasets were synthetically generated to represent real-world scenarios and challenges with known constraints.

In this paper, we provide a comprehensive description and detailed information regarding the creation of these datasets. The datasets were generated by implementations in \emph{Python} using \emph{Pytorch} and \emph{Torchdiffeq}.

\paragraph{World Population Growth} 

The WPG dataset simulates the growth of the world population over time. The dynamics of this system follow an exponential growth pattern until the population reaches its carrying capacity. The carrying capacity represents the maximum number of people that the system can sustain, considering limited resources.

The exponential growth in the WPG dataset reflects the natural tendency of populations to multiply rapidly when resources are abundant. However, as the population approaches the carrying capacity, resource limitations and competition start to constrain the growth rate. This results in a slowdown of population growth until it stabilizes around the carrying capacity. We considered the Verhulst logistic model as an IVP as described in \cite{rodriguesModeloLogisticoVerhulst2014},

\begin{equation}
	\begin{cases}
		\dfrac{dP(t)}{dt} = rP(t) \left( 1 - \dfrac{P(t)}{K} \right) \\
		P(0) = P_0 = 2.518629,
	\end{cases}
\end{equation}

\noindent where $P(t)$ is the world population, measured in billions of people ($10^9$ individuals), $r=0.026$ is the rate of growth and $K<12$ is the carrying capacity.

In the WPG dataset, two features are included: the time-step $t_n$ and the population $P$. To generate the dataset, we solved the IVP using the Runge-Kutta method of order 5 of Dormand-Prince-Shampine with varying time intervals and sampling frequencies to accommodate the three different experiments introduced in \cite{coelho2023prior}, experiments 1.0, 2.0 and 3.0. In addition, for this paper, we introduced four new training and testing sets: experiments 2.1, 2.2, 3.1, and 3.2. These new experiments aim to evaluate the influence of the training dataset size on the performance of the two-stage training method. The details of the available data for each experiment are as follows:

\begin{itemize}
	\item \textbf{Experiment 1.0 - Reconstruction:}  This experiment involves training and testing sets with 200 points equally spaced in the time interval $(0,300)$;
	\item \textbf{Experiment 2.0 - Extrapolation:} The training set consists of 200 points equally spaced in the time interval $(0,300)$, while the testing set includes 200 equally spaced points in the extended time interval $(0,400)$; 
	\item \textbf{Experiment 2.1 - Extrapolation with sparser training set:} This experiment is a variation of Experiment 2.0, where the training set includes only 100 equally spaced points in the time interval $(0,300)$;
	\item \textbf{Experiment 2.2 - Extrapolation with abundant training set:} This experiment is a variation of Experiment 2.0, where the training set includes 300 equally spaced points in the time interval $(0,300)$;
	\item \textbf{Experiment 3.0 - Completion:} The training set consists of 200 points equally spaced in the time interval $(0,300)$, while the testing set includes 300 equally spaced points in the same time interval $(0,300)$; 
	\item \textbf{Experiment 3.1 - Completion with sparser training set:} This experiment is a variation of Experiment 3.0, where the training set includes only 100 equally spaced points in the time interval $(0,300)$;
	\item \textbf{Experiment 3.2 - Completion with abundant training set:} This experiment is a variation of Experiment 3.0, where the training set includes 300 equally spaced points in the time interval $(0,300)$;
	
\end{itemize}

By capturing the dynamics of population growth and the concept of carrying capacity, the WPG dataset provides a realistic representation of population dynamics and the constraints imposed by limited resources. 

The WPG dataset with the training and testing sets to conduct every experiment is publicly available on \emph{Kaggle} \cite{coelho_population_2023}.

\paragraph{Chemical Reaction}

The CR dataset simulates a chemical reaction involving four fictitious chemical components: A, B, C, and D. At the initial time, there is a presence of $1g$ of species A and B. As time progresses, these two species undergo a chemical reaction, resulting in the formation of species C and D:

$$A(t) + B(t) \rightarrow C(t) + D(t).$$

The dynamics of this system are governed by a set of reaction equations that describe the rates of change of each species as a function of time. These equations capture the transformation of species A and B into species C and D over time. During the reaction, the masses of species A and B gradually decrease, while the masses of species C and D increase accordingly. This reflects the consumption of species A and B and the production of species C and D as the reaction progresses. As the reaction proceeds, the system reaches an equilibrium state where the masses of all species remain constant. The reaction equations arbitrarily chosen by us are as follows:

\noindent\begin{minipage}{.5\linewidth}
\begin{equation*}
	A(t) + B(t) \rightarrow C(t)
	\begin{cases}
		\frac{dm_A(t)}{dt} = -k_1  m_A(t)  m_B(t) \\
		\frac{dm_B(t)}{dt} = -k_1  m_A(t)  m_B(t) \\
		\frac{dm_C(t)}{dt} = k_1  m_A(t)  m_B(t)
		
	\end{cases}
\end{equation*}
\end{minipage}%
\begin{minipage}{.5\linewidth}
\begin{equation*}
	C(t) \rightarrow B(t) + D(t) 
	\begin{cases}
		\frac{dm_C(t)}{dt} = -k_2  m_C(t) \\
		\frac{dm_B(t)}{dt} = k_2  m_C(t) \\
		\frac{dm_D(t)}{dt} = k_2  m_C(t)
	\end{cases}
\end{equation*}
\end{minipage}

\vspace{0.2cm}
\noindent where $m_A(t), m_B(t), m_C(t) \text{ and } m_D(t)$ are the masses of species A, B, C and D, respectively, at time $t$ and $k_1=0.1$ and $k_2=0.05$ determine the rates of the respective reactions.

It is important to note that in our simulation, we assumed ideal conditions where no mass is lost. This means that at each time step, the conservation of mass must be satisfied. According to this principle, the total mass of the system $m_{total}$, which is the sum of the masses of all chemical species, remains constant throughout the simulation:

$$m_A(t) + m_B(t) + m_C(t) + m_D(t) = m_{total}.$$

In the CR dataset, five features are included: the time $t$ and the masses of species A, $m_A$, B, $m_B$, C, $m_C$ and D, $m_D$. To generate the dataset, the reaction equations were used to simulate the changes in the masses of each species over time by using the Euler's method to solve an IVP in different time intervals and sampling frequencies, to accommodate the three experiments introduced in \cite{coelho2023prior}, experiments 1.0, 2.0 and 3.0. In addition, for this paper, we introduced four new training and testing sets: experiments 2.1, 2.2, 3.1, and 3.2. These new experiments aim to evaluate the influence of the training dataset size on the performance of the two-stage training method. The details of the available data for each experiment are as follows:

\begin{itemize}
	\item \textbf{Experiment 1.0 - Reconstruction:}  This experiment involves training and testing sets with 100 points equally spaced in the time interval $(0,100)$;
	\item \textbf{Experiment 2.0 - Extrapolation:} The training set consists of 100 points equally spaced in the time interval $(0,100)$, while the testing set includes 100 equally spaced points in the extended time interval $(0,200)$; 
	\item \textbf{Experiment 2.1 - Extrapolation with sparser training set:} This experiment is a variation of Experiment 2.0, where the training set includes only 50 equally spaced points in the time interval $(0,100)$;
	\item \textbf{Experiment 2.2 - Extrapolation with abundant training set:} This experiment is a variation of Experiment 2.0, where the training set includes 150 equally spaced points in the time interval $(0,100)$;
	\item \textbf{Experiment 3.0 - Completion:} The training set consists of 100 points equally spaced in the time interval $(0,100)$, while the testing set includes 200 equally spaced points in the same time interval $(0,100)$; 
	\item \textbf{Experiment 3.1 - Completion with sparser training set:} This experiment is a variation of Experiment 3.0, where the training set includes only 50 equally spaced points in the time interval $(0,100)$;
	\item \textbf{Experiment 3.2 - Completion with abundant training set:} This experiment is a variation of Experiment 3.0, where the training set includes 150 equally spaced points in the time interval $(0,100)$;
	
\end{itemize}

By considering these reaction equations and the conservation of mass, we can accurately model the chemical reaction and analyze the dynamics of the system.

The CR dataset with the training and testing sets to conduct every experiment described is publicly available on \emph{Kaggle} \cite{coelho_chemical_2023}.

\section{Experimental Setup} \label{subsec:setup}

\paragraph{World Population Growth}

The modeling of the system given by the WPG dataset by a NN can be formulated as a constrained optimization problem as follows:

\begin{mini}|l|[0]
	{\boldsymbol{\theta} \in \mathbb{R}^{n_\theta}}{\dfrac{1}{N} \sum_{n=1}^N (\boldsymbol{\hat{y}}_{n}(\boldsymbol{\theta}) - \boldsymbol{y}_{n})^2}
    {\label{eq:constrainedWPG}}
    {}
    \addConstraint{\boldsymbol{\hat{y}}_n(\boldsymbol{\theta)}}{ \le 12, \,\,\, n \in 1,\dots,N}.
\end{mini}

\noindent where $\boldsymbol{\hat{y}}_n(\boldsymbol{\theta})$ are the predictions, made by the model, of the population $P$ at each time-step $t_n$ and $\boldsymbol{y}_n$ are the corresponding expected/real values. The loss function is the fit to the training data and the system is characterized by having one inequality constraint per time-step given by the carrying capacity. In our experiments, the Mean Squared Error (MSE) was used as the loss function.

This constrained problem can be rewritten as a sequence of problems as to be used the two-stage method as follows:

\begin{mini}|l|[0]
	{\boldsymbol{\theta} \in \mathbb{R}^{n_\theta}}{\mathcal{L}_{I}(\boldsymbol{\theta}) = \dfrac{1}{N} \sum_{n=1}^N \left(\max(\boldsymbol{\hat{y}}_n(\boldsymbol{\theta})-12, 0)\right)^2}
    {\label{eq:admissibilityWPG}}
    {}
\end{mini}

\begin{mini}|l|[0]
	{\boldsymbol{\theta} \in \mathbb{R}^{n_\theta}}{\mathcal{L}_{II}(\boldsymbol{\theta}) = \dfrac{1}{N} \sum_{n=1}^N (\boldsymbol{\hat{y}}_{n}(\boldsymbol{\theta}) - \boldsymbol{y}_{n})^2}
    {\label{eq:optimizationWPG}}
    {}
\end{mini}

In the admissibility stage, \eqref{eq:admissibilityWPG} is solved while in the optimization stage \eqref{eq:optimizationWPG} is solved. 

\textbf{Remark:} The \emph{vanilla} Neural ODE is trained following the traditional NN optimization process by minimizing the loss function given by \eqref{eq:optimizationWPG}.

The WPG dataset was modeled by \emph{vanilla} Neural ODE and Neural ODE trained with the proposed two-stage method using the same NN architecture and training conditions. The NN has $4$ hidden layers: linear with $50$ neurons; hyperbolic tangent (tanh); linear with $50$ neurons; Exponential Linear Unit (ELU). The input and output layers have $1$ neuron. The Adam optimizer was used with a learning rate of 1E-5, a feasibility tolerance of 1E-4 and a minimum of $20$ iterations for the admissibility stage, and $10000$ iterations for the optimization stage.

To effectively model the WPG dataset and ensure meaningful predictions, it is crucial to respect the carrying capacity constraint imposed by the dataset. Failure to do so may lead to inaccurate and unreliable predictions, undermining the trustworthiness of the model.

\paragraph{Chemical Reaction}

The modeling of the system given by the CR dataset by a NN can be formulated as a constrained optimization problem:

\begin{mini}|l|[0]
	{\boldsymbol{\theta} \in \mathbb{R}^{n_\theta}}{\dfrac{1}{N} \sum_{n=1}^N (\boldsymbol{\hat{y}}_{n}(\boldsymbol{\theta}) - \boldsymbol{y}_{n})^2}
    {\label{eq:constrainedWPG}}
    {}
    \addConstraint{\boldsymbol{1}^\top \boldsymbol{\hat{y}}_n(\boldsymbol{\theta)}}{=m_{total} , \,\,\, n \in 1,\dots,N}.
\end{mini}

\noindent where $\boldsymbol{\hat{y}}_n(\boldsymbol{\theta})$ are the predictions, made by the model, of the masses of species A, B, C and D at each time-step $t_n$ and $\boldsymbol{y}_n$ are the corresponding expected/real values and $constant$ represents the total mass of the system. The loss function is the fit to the training data and the system is characterized by having one equality constraint per time-step given by the conservation of mass law. In our experiments, the MSE was used as the loss function.

This constrained problem can be rewritten as a sequence of problems as to be used the two-stage method as follows:

\begin{mini}|l|[0]
{\boldsymbol{\theta} \in \mathbb{R}^{n_\theta}}{\mathcal{L}_{I}(\boldsymbol{\theta}) = \dfrac{1}{N} \sum_{n=1}^N \left( \boldsymbol{1}^\top \boldsymbol{\hat{y}}_n - m_{total} \right)^2}
    {\label{eq:admissibilityCR}}
    {}
\end{mini}

\begin{mini}|l|[0]
	{\boldsymbol{\theta} \in \mathbb{R}^{n_\theta}}{\mathcal{L}_{II}(\boldsymbol{\theta}) = \dfrac{1}{N} \sum_{n=1}^N (\boldsymbol{\hat{y}}_{n}(\boldsymbol{\theta}) - \boldsymbol{y}_{n})^2}
    {\label{eq:optimizationCR}}
    {}
\end{mini}

In the admissibility stage, \eqref{eq:admissibilityCR} is solved, while in the optimization stage \eqref{eq:optimizationCR} is solved. 

\textbf{Remark:} The \emph{vanilla} Neural ODE is trained following the traditional NN optimization process by minimizing the loss function given by \eqref{eq:optimizationCR}.

The CR dataset was modeled by \emph{vanilla} Neural ODE and Neural ODE trained with the proposed two-stage method using the same NN architecture and training conditions. The NN has $6$ hidden layers: linear with $50$ neurons; tanh; linear with $64$ neurons; ELU; linear with $50$ neurons; tanh. The input and output layers have $4$ neurons. The Adam optimizer was used with a learning rate of 1E-5, a feasibility tolerance of 1E-4 and a minimum of $20$ iterations for the admissibility stage, and $10000$ iterations for the optimization stage. In this experiment, the NN architecture used is more complex than the architecture used to model the WPG dataset given that this problem is more difficult due to its higher dimensionality.

\section{Experimental Results} \label{app:exp}

\subsection{World Population Growth} \label{app:WPG}

The results obtained for the WPG dataset in all experiments for all tolerance values are organized in Table \ref{tab:performanceWPG}. The best performance for each experiment, MSE and tolerance, is highlighted in bold.

\begin{landscape}
\begin{table}[]
\tiny
\addtolength{\tabcolsep}{-5pt}
\centering
\caption{Performance at the seven experiments on the WPG dataset.}\label{tab:performanceWPG}
\begin{tabular}{@{}cccccccccc@{}}
\toprule
\multirow{2}{*}{}    &      & \multicolumn{2}{c}{\multirow{2}{*}{vanilla Neural ODE}}                         & \multicolumn{6}{c}{two-stage Neural ODE}                                                                                                                                                                                                                  \\ \cmidrule(l){5-10} 
                     &      & \multicolumn{2}{c}{}                   & \multicolumn{2}{c}{noStrategy}                                         & \multicolumn{2}{c}{updatePrevious}                                     & \multicolumn{2}{c}{updateBest}                                                                          \\ \midrule
Experiment           & Tol  & $\text{MSE}_{\text{avg}}$ $\pm$ std    & $\text{V}_{\text{avg}}$$\pm$ std       & $\text{MSE}_{\text{avg}}$ $\pm$ std & $\text{V}_{\text{avg}}$$\pm$ std & $\text{MSE}_{\text{avg}}$ $\pm$ std & $\text{V}_{\text{avg}}$$\pm$ std & $\text{MSE}_{\text{avg}}$ $\pm$ std                & $\text{V}_{\text{avg}}$$\pm$ std                   \\ \midrule
\multirow{4}{*}{1.0} & 1E-2 & \multirow{4}{*}{1.59E-2 $\pm$ 1.35E-2} & \multirow{4}{*}{3.02E-2 $\pm$ 1.70E-2} & 8.37E-3 $\pm$ 7.46E-3               & 1.91E-2 $\pm$ 1.27E-2            & \textbf{9.48E-3 $\pm$ 6.12E-3}      & \textbf{2.02E-2 $\pm$ 1.39E-2}   & 1.24E-2 $\pm$ 4.74E-3                              & 2.83E-2 $\pm$ 5.95E-3                              \\
                     & 1E-4 &                                        &                                        & 1.78E-2 $\pm$ 2.22E-2               & 2.61E-2 $\pm$ 2.49E-2            & \textbf{6.37E-3 $\pm$ 4.17E-3}      & \textbf{1.54E-2 $\pm$ 8.96E-3}   & 1.05E-2 $\pm$ 5.50E-3                              & 2.59E-2 $\pm$ 7.86E-3                              \\
                     & 1E-6 &                                        &                                        & 2.92E-2 $\pm$ 2.82E-2               & 3.96E-2 $\pm$ 2.46E-2            & 1.47E-2 $\pm$ 1.05E-2               & 2.78E-2 $\pm$ 1.63E-2            & \multicolumn{1}{l}{\textbf{1.07E-2 $\pm$ 4.12E-3}} & \multicolumn{1}{l}{\textbf{2.36E-2 $\pm$ 9.41E-3}} \\
                     & 1E-8 &                                        &                                        & 1.74E-2 $\pm$ 1.67E-2               & 2.94E-2 $\pm$ 1.80E-2            & \textbf{6.66E-3 $\pm$ 3.43E-3}      & \textbf{1.72E-2 $\pm$ 8.06E-3}   & \multicolumn{1}{l}{3.24E-2 $\pm$ 2.42E-2}          & \multicolumn{1}{l}{4.52E-2 $\pm$ 1.91E-2}          \\ \midrule
2.0                  & 1E-2 & \multirow{4}{*}{6.22E-2 $\pm$ 5.73E-2} & \multirow{4}{*}{1.87E-1 $\pm$ 1.08E-1} & 7.58E-3 $\pm$ 2.99E-3               & 6.84E-2 $\pm$ 1.90E-2            & 3.13E-2 $\pm$ 2.12E-2               & 1.41E-1 $\pm$ 4.04E-2            & \textbf{1.14E-2 $\pm$ 5.99E-3}                     & \textbf{8.42E-2 $\pm$ 2.67E-2}                     \\
                     & 1E-4 &                                        &                                        & 6.54E-2 $\pm$ 8.71E-2               & 1.74E-1 $\pm$ 1.24E-1            & \textbf{1.54E-2 $\pm$ 1.47E-2}      & \textbf{9.07E-2 $\pm$ 4.96E-2}   & 3.13E-2 $\pm$ 1.74E-2                              & 1.45E-1 $\pm$ 3.45E-2                              \\
                     & 1E-6 &                                        &                                        & \textbf{1.30E-2 $\pm$ 3.06E-3}      & \textbf{9.36E-2 $\pm$ 1.06E-2}   & 2.99E-2 $\pm$ 2.63E-2               & 1.28E-1 $\pm$ 6.55E-2            & \multicolumn{1}{l}{2.37E-2 $\pm$ 1.67E-2}          & \multicolumn{1}{l}{1.20E-1 $\pm$ 5.04E-2}          \\
                     & 1E-8 &                                        &                                        & \textbf{1.87E-2 $\pm$ 8.53E-3}      & \textbf{1.13E-1 $\pm$ 2.78E-2}   & 2.51E-2 $\pm$ 1.80E-2               & 1.17E-1 $\pm$ 6.11E-2            & \multicolumn{1}{l}{5.91E-2 $\pm$ 4.20E-2}          & \multicolumn{1}{l}{1.74E-1 $\pm$ 9.22E-2}          \\ \midrule
2.1                  & 1E-2 & \multirow{4}{*}{2.05E-1 $\pm$ 2.11E-1} & \multirow{4}{*}{3.70E-1 $\pm$ 2.55E-1} & 3.10E-2 $\pm$ 2.63E-2               & 1.36E-1 $\pm$ 5.48E-2            & \textbf{1.29E-2 $\pm$ 7.38E-3}      & \textbf{8.81E-2 $\pm$ 2.70E-2}   & 3.29E-2 $\pm$ 2.80E-2                              & 1.38E-1 $\pm$ 6.30E-2                              \\
                     & 1E-4 &                                        &                                        & 1.58E-2 $\pm$ 1.40E-2               & 9.30E-2 $\pm$ 4.63E-2            & \textbf{1.32E-2 $\pm$ 1.62E-2}      & \textbf{7.20E-2 $\pm$ 5.40E-2}   & 1.43E-2 $\pm$ 6.80E-3                              & 9.45E-2 $\pm$ 3.11E-2                              \\
                     & 1E-6 &                                        &                                        & 1.99E-2 $\pm$ 1.82E-2               & 1.07E-1 $\pm$ 4.55E-2            & 2.93E-2 $\pm$ 3.52E-2               & 1.22E-1 $\pm$ 7.40E-2            & \multicolumn{1}{l}{\textbf{8.96E-3 $\pm$ 3.20E-3}} & \multicolumn{1}{l}{\textbf{7.41E-2 $\pm$ 1.99E-2}} \\
                     & 1E-8 &                                        &                                        & 2.37E-2 $\pm$ 1.46E-2               & 1.20E-1 $\pm$ 4.11E-2            & \textbf{1.37E-2 $\pm$ 8.05E-3}      & \textbf{9.41E-2 $\pm$ 3.07E-2}   & \multicolumn{1}{l}{1.49E-2 $\pm$ 7.25E-3}          & \multicolumn{1}{l}{9.59E-2 $\pm$ 3.24E-2}          \\ \midrule
2.2                  & 1E-2 & \multirow{4}{*}{8.12E-2 $\pm$ 6.72E-2} & \multirow{4}{*}{2.24E-1 $\pm$ 1.22E-1} & 2.93E-2 $\pm$ 8.53E-3               & 1.46E-1 $\pm$ 2.03E-2            & \textbf{1.28E-2 $\pm$ 6.49E-3}      & \textbf{9.00E-2 $\pm$ 2.67E-2}   & 2.41E-2 $\pm$ 6.94E-3                              & 1.28E-1 $\pm$ 1.74E-2                              \\
                     & 1E-4 &                                        &                                        & 3.87E-2 $\pm$ 2.01E-2               & 1.61E-1 $\pm$ 4.47E-2            & \textbf{2.21E-2 $\pm$ 1.06E-2}      & \textbf{1.21E-1 $\pm$ 2.72E-2}   & 3.17E-2 $\pm$ 1.26E-2                              & 1.48E-1 $\pm$ 3.24E-2                              \\
                     & 1E-6 &                                        &                                        & 1.75E-2 $\pm$ 1.91E-2               & 9.44E-2 $\pm$ 6.45E-2            & 1.72E-2 $\pm$ 9.15E-3               & 1.07E-1 $\pm$ 2.99E-2            & \multicolumn{1}{l}{\textbf{7.99E-3 $\pm$ 8.46E-3}} & \multicolumn{1}{l}{\textbf{5.98E-2 $\pm$ 4.05E-2}} \\
                     & 1E-8 &                                        &                                        & \textbf{1.38E-2 $\pm$ 3.34E-3}      & \textbf{9.66E-2 $\pm$ 1.32E-2}   & 2.03E-2 $\pm$ 5.66E-3               & 1.20E-1 $\pm$ 1.90E-2            & \multicolumn{1}{l}{1.47E-2 $\pm$ 9.32E-3}          & \multicolumn{1}{l}{9.17E-2 $\pm$ 3.05E-2}          \\ \midrule
3.0                  & 1E-2 & \multirow{4}{*}{6.67E-2 $\pm$ 4.94E-2} & \multirow{4}{*}{1.04E-1 $\pm$ 4.10E-2} & \textbf{1.55E-2 $\pm$ 1.42E-2}      & \textbf{4.15E-2 $\pm$ 2.62E-2}   & 1.95E-2 $\pm$ 1.62E-2               & 4.84E-2 $\pm$ 3.01E-2            & 2.39E-2 $\pm$ 1.86E-2                              & 5.55E-2 $\pm$ 2.27E-2                              \\
                     & 1E-4 &                                        &                                        & \textbf{8.82E-3 $\pm$ 6.07E-3}      & \textbf{2.95E-2 $\pm$ 1.84E-2}   & 2.86E-2 $\pm$ 3.09E-2               & 5.45E-2 $\pm$ 3.58E-2            & 2.01E-2 $\pm$ 1.93E-2                              & 4.46E-2 $\pm$ 3.62E-2                              \\
                     & 1E-6 &                                        &                                        & 1.57E-2 $\pm$ 8.90E-3               & 4.61E-2 $\pm$ 1.85E-2            & \textbf{1.07E-2 $\pm$ 7.78E-3}      & \textbf{3.28E-2 $\pm$ 2.29E-2}   & \multicolumn{1}{l}{2.01E-2 $\pm$ 1.71E-2}          & \multicolumn{1}{l}{5.08E-2 $\pm$ 2.80E-2}          \\
                     & 1E-8 &                                        &                                        & 3.89E-2 $\pm$ 5.04E-2               & 6.91E-2 $\pm$ 5.90E-2            & \textbf{4.24E-3 $\pm$ 2.80E-3}      & \textbf{1.52E-2 $\pm$ 1.18E-2}   & \multicolumn{1}{l}{1.03E-2 $\pm$ 7.24E-3}          & \multicolumn{1}{l}{3.52E-2 $\pm$ 1.85E-2}          \\ \midrule
3.1                  & 1E-2 & \multirow{4}{*}{5.94E-2 $\pm$ 4.86E-2} & \multirow{4}{*}{9.76E-2 $\pm$ 5.75E-2} & \textbf{9.18E-3 $\pm$ 1.08E-2}      & \textbf{2.65E-2 $\pm$ 2.49E-2}   & 2.54E-2 $\pm$ 2.06E-2               & 5.67E-2 $\pm$ 2.79E-2            & 1.56E-2 $\pm$ 7.72E-3                              & 4.78E-2 $\pm$ 1.66E-2                              \\
                     & 1E-4 &                                        &                                        & 1.93E-2 $\pm$ 1.54E-2               & 4.86E-2 $\pm$ 3.00E-2            & 2.31E-2 $\pm$ 9.87E-3               & 5.95E-2 $\pm$ 1.52E-2            & \textbf{1.72E-2 $\pm$ 8.19E-3}                     & \textbf{5.00E-2 $\pm$ 1.61E-2}                     \\
                     & 1E-6 &                                        &                                        & 1.38E-2 $\pm$ 1.22E-2               & 3.60E-2 $\pm$ 2.56E-2            & 1.55E-2 $\pm$ 1.40E-2               & 4.40E-2 $\pm$ 2.29E-2            & \multicolumn{1}{l}{\textbf{6.27E-3 $\pm$ 4.15E-3}} & \multicolumn{1}{l}{\textbf{2.08E-2 $\pm$ 1.31E-2}} \\
                     & 1E-8 &                                        &                                        & 1.60E-2 $\pm$ 4.33E-3               & 4.79E-2 $\pm$ 7.18E-3            & \textbf{1.56E-2 $\pm$ 4.69E-3}      & \textbf{4.80E-2 $\pm$ 1.13E-2}   & \multicolumn{1}{l}{2.22E-2 $\pm$ 2.20E-2}          & \multicolumn{1}{l}{5.13E-2 $\pm$ 2.64E-2}          \\ \midrule
3.2                  & 1E-2 & \multirow{4}{*}{4.84E-2 $\pm$ 3.41E-2} & \multirow{4}{*}{9.87E-2 $\pm$ 5.45E-2} & 5.48E-3 $\pm$ 1.67E-3               & 2.25E-2 $\pm$ 7.84E-3            & \textbf{1.32E-2 $\pm$ 9.11E-3}      & \textbf{4.27E-2 $\pm$ 2.16E-2}   & 4.62E-2 $\pm$ 2.55E-2                              & 8.40E-2 $\pm$ 2.20E-2                              \\
                     & 1E-4 &                                        &                                        & 4.25E-2 $\pm$ 3.16E-2               & 7.90E-2 $\pm$ 3.53E-2            & 1.29E-2 $\pm$ 3.94E-3               & 4.40E-2 $\pm$ 7.02E-3            & \textbf{8.38E-3 $\pm$ 5.57E-3}                     & \textbf{2.96E-2 $\pm$ 1.56E-2}                     \\
                     & 1E-6 &                                        &                                        & \textbf{1.07E-2 $\pm$ 4.53E-3}      & \textbf{3.71E-2 $\pm$ 1.57E-2}   & 1.18E-2 $\pm$ 9.98E-3               & 3.35E-2 $\pm$ 2.69E-2            & \multicolumn{1}{l}{1.49E-2 $\pm$ 1.25E-2}          & \multicolumn{1}{l}{4.28E-2 $\pm$ 2.53E-2}          \\
                     & 1E-8 &                                        &                                        & 1.71E-2 $\pm$ 2.31E-2               & 3.73E-2 $\pm$ 4.07E-2            & 1.37E-2 $\pm$ 1.87E-3               & 4.89E-2 $\pm$ 2.80E-3            & \multicolumn{1}{l}{\textbf{6.85E-3 $\pm$ 4.68E-3}} & \multicolumn{1}{l}{\textbf{2.38E-2 $\pm$ 1.61E-2}} \\ \bottomrule
\end{tabular}%
\end{table}
\end{landscape}

\subsection{Chemical Reaction} \label{app:CR}

The results for the CR dataset in all experiments for all tolerance values are organized in Table \ref{tab:performanceCR}. The best performance for each experiment, MSE and tolerance, is highlighted in bold.

\begin{landscape}
\begin{table}[]
\tiny
\addtolength{\tabcolsep}{-5pt}
\centering
\caption{Performance at the seven experiments on the CR dataset.}\label{tab:performanceCR}
\begin{tabular}{@{}cccccccccc@{}}
\toprule
\multirow{2}{*}{}    &      & \multicolumn{2}{c}{\multirow{2}{*}{vanilla Neural ODE}}                           & \multicolumn{6}{c}{two-stage Neural ODE}                                                                                                                                                                                 \\ \cmidrule(l){5-10} 
                     &      & \multicolumn{2}{c}{}                     & \multicolumn{2}{c}{noStrategy}                                         & \multicolumn{2}{c}{updatePrevious}                                     & \multicolumn{2}{c}{updateBest}                                         \\ \midrule
Experiment           & Tol  & $\text{MSE}_{\text{avg}}$ $\pm$ std      & $\text{V}_{\text{avg}}$$\pm$ std       & $\text{MSE}_{\text{avg}}$ $\pm$ std & $\text{V}_{\text{avg}}$$\pm$ std & $\text{MSE}_{\text{avg}}$ $\pm$ std & $\text{V}_{\text{avg}}$$\pm$ std & $\text{MSE}_{\text{avg}}$ $\pm$ std & $\text{V}_{\text{avg}}$$\pm$ std \\ \midrule
\multirow{4}{*}{1.0} & 1E-2 & \multirow{4}{*}{7.23E-3 $\pm$ 1.56E-2}   & \multirow{4}{*}{7.89E-2 $\pm$ 1.56E-1} & 9.00E-04 $\pm$ 5.36E-04             & 2.21E-2 $\pm$ 1.58E-2            & 2.70E-3 $\pm$ 2.68E-3               & 4.80E-2 $\pm$ 2.36E-2            & \textbf{6.94E-4 $\pm$ 9.72E-4}      & \textbf{2.71E-2 $\pm$ 3.02E-2}   \\
                     & 1E-4 &                                          &                                        & 1.45E-03 $\pm$ 1.13E-03             & 2.60E-2 $\pm$ 1.25E-2            & \textbf{7.95E-4 $\pm$ 8.09E-4}      & \textbf{1.83E-2 $\pm$ 1.22E-2}   & 1.32E-3 $\pm$ 1.24E-3               & 3.33E-2 $\pm$ 2.20E-2            \\
                     & 1E-6 &                                          &                                        & \textbf{8.60E-05 $\pm$ 1.01E-04}    & \textbf{7.22E-3 $\pm$ 8.88E-3}   & 1.28E-3 $\pm$ 2.05E-3               & 1.22E-2 $\pm$ 1.36E-2            & 2.60E-3 $\pm$ 4.43E-3               & 1.75E-2 $\pm$ 2.52E-2            \\
                     & 1E-8 &                                          &                                        & 6.60E-05 $\pm$ 6.58E-05             & 4.12E-3 $\pm$ 3.77E-3            & 7.89E-04 $\pm$ 1.31E-03             & 1.49E-2 $\pm$ 2.29E-2            & \textbf{2.88E-5 $\pm$ 1.48E-6}      & \textbf{2.31E-3 $\pm$ 1.48E-3}   \\ \midrule
2.0                  & 1E-2 & \multirow{4}{*}{1.68E+01 $\pm$ 2.87E+01} & \multirow{4}{*}{1.28 $\pm$ 2.02}       & \textbf{2.33E-01 $\pm$ 3.01E-01}    & 2.68E-1 $\pm$ 2.99E-1            & 3.61 $\pm$ 6.23                     & 1.58 $\pm$ 2.49                  & 7.01E-1 $\pm$ 7.18E-1               & 4.97E-1 $\pm$ 4.02E-1            \\
                     & 1E-4 &                                          &                                        & 5.27E-01 $\pm$ 9.13E-01             & 2.77E-1 $\pm$ 4.69E-1            & 6.56E-1 $\pm$ 1.14                  & 6.35E-1 $\pm$ 1.09               & \textbf{5.75E-1 $\pm$ 6.41E-1}      & \textbf{1.95E-4 $\pm$ 3.29E-4}   \\
                     & 1E-6 &                                          &                                        & \textbf{1.56E-04 $\pm$ 1.02E-04}    & \textbf{8.16E-3 $\pm$ 8.07E-3}   & 7.38E-2 $\pm$ 9.35E-2               & 1.78E-1 $\pm$ 1.91E-1            & 4.76E-4 $\pm$ 7.16E-4               & 1.26E-2 $\pm$ 1.47E-2            \\
                     & 1E-8 &                                          &                                        & 9.52E-01 $\pm$ 1.65                 & 1.22E-1 $\pm$ 2.05E-1            & \textbf{7.75E-05 $\pm$ 2.99E-05}    & \textbf{1.02E-2 $\pm$ 5.83E-3}   & 6.80E-3 $\pm$ 1.18E-3               & 2.65E-2 $\pm$ 2.73E-2            \\ \midrule
2.1                  & 1E-2 & \multirow{4}{*}{1.14E+01 $\pm$ 1.73E+01} & \multirow{4}{*}{1.74 $\pm$ 2.38}       & 1.99 $\pm$ 2.48                     & 7.03E-1 $\pm$ 8.43E-1            & 4.87E-2 $\pm$ 5.75E-2               & 2.62E-1 $\pm$ 2.80E-1            & \textbf{3.03E-2 $\pm$ 5.19E-2}      & \textbf{1.53E-1 $\pm$ 1.75E-1}   \\
                     & 1E-4 &                                          &                                        & \textbf{5.73E-05 $\pm$ 1.20E-05}    & \textbf{6.97E-3 $\pm$ 2.73E-3}   & 3.11E-1 $\pm$ 3.37E-1               & 4.02E-1 $\pm$ 4.58E-1            & 1.95E-2 $\pm$ 3.29E-2               & 5.51E-2 $\pm$ 6.36E-2            \\
                     & 1E-6 &                                          &                                        & \textbf{5.26E-02 $\pm$ 9.10E-02}    & \textbf{2.57E-2 $\pm$ 1.90E-2}   & 2.70E-2 $\pm$ 4.65E-2               & 1.90E-1 $\pm$ 2.95E-1            & 9.21E-2 $\pm$ 1.57E-1               & 3.63E-2 $\pm$ 3.35E-2            \\
                     & 1E-8 &                                          &                                        & 1.84E-03 $\pm$ 3.10E-03             & 2.18E-2 $\pm$ 2.29E-2            & \textbf{4.05E-04 $\pm$ 6.04E-04}    & \textbf{6.41E-3 $\pm$ 4.09E-3}   & 1.12E-3 $\pm$ 1.74E-3               & 1.95E-2 $\pm$ 2.00E-2            \\ \midrule
2.2                  & 1E-2 & \multirow{4}{*}{9.49E-1 $\pm$ 1.67E+0}   & \multirow{4}{*}{2.57E-1 $\pm$ 3.06E-1} & 2.75E+00 $\pm$ 4.68E+00             & 7.99E-1 $\pm$ 1.11               & 2.31E+0 $\pm$ 3.96E+0               & 4.30E-1 $\pm$ 4.46E-1            & 4.44 $\pm$ 7.68                     & 1.07 $\pm$ 1.84                  \\
                     & 1E-4 &                                          &                                        & \textbf{5.42E-02 $\pm$ 7.26E-02}    & \textbf{1.75E-1 $\pm$ 2.30E-1}   & 3.88E-1 $\pm$ 6.66E-1               & 1.19E-1 $\pm$ 1.31E-1            & 1.07E1 $\pm$ 1.85E1                 & 5.90E-1 $\pm$ 9.94E-1            \\
                     & 1E-6 &                                          &                                        & \textbf{3.08E-02 $\pm$ 5.31E-02}    & \textbf{1.62E-1 $\pm$ 2.70E-1}   & 4.54E+0 $\pm$ 7.86E+0               & 6.72E-1 $\pm$ 1.14               & 1.41E1 $\pm$ 2.36E1                 & 1.30 $\pm$ 1.88                  \\
                     & 1E-8 &                                          &                                        & \textbf{2.38E-04 $\pm$ 1.75E-04}    & \textbf{1.15E-2 $\pm$ 5.38E-1}   & 4.79E-02 $\pm$ 8.29E-02             & 1.46E-2 $\pm$ 1.35E-2            & 1.72E $\pm$ 2.99                    & 1.15 $\pm$ 1.98                  \\ \midrule
3.0                  & 1E-2 & \multirow{4}{*}{2.15E-3 $\pm$ 1.64E-3}   & \multirow{4}{*}{6.00E-2 $\pm$ 5.69E-2} & 2.10E-03 $\pm$ 2.10E-03             & 8.43E-2 $\pm$ 7.92E-2            & 2.30E-3 $\pm$ 2.77E-3               & 6.47E-2 $\pm$ 6.17E-2            & \textbf{5.75E-4 $\pm$ 4.95E-4}      & \textbf{5.36E-2 $\pm$ 5.33E-2}   \\
                     & 1E-4 &                                          &                                        & 3.36E-04 $\pm$ 3.88E-04             & 2.82E-2 $\pm$ 2.15E-2            & 3.99E-4 $\pm$ 3.73E-4               & 2.46E-2 $\pm$ 3.17E-2            & \textbf{3.53E-4 $\pm$ 5.19E-4}      & \textbf{2.34E-2 $\pm$ 1.95E-2}   \\
                     & 1E-6 &                                          &                                        & \textbf{1.86E-04 $\pm$ 1.76E-04}    & \textbf{2.06E-2 $\pm$ 1.82E-2}   & 2.62E-4 $\pm$ 4.07E-4               & 2.07E-2 $\pm$ 2.97E-2            & 5.59E-4 $\pm$ 8.93E-4               & 3.96E-2 $\pm$ 6.19E-2            \\
                     & 1E-8 &                                          &                                        & \textbf{2.98E-05 $\pm$ 6.26E-06}    & 5.05E-3 $\pm$ 1.95E-3            & 7.97E-05 $\pm$ 8.91E-05             & 6.03E-3 $\pm$ 6.50E-3            & 1.50E-4 $\pm$ 1.70E-4               & 8.59E-1 $\pm$ 8.90E-1            \\ \midrule
3.1                  & 1E-2 & \multirow{4}{*}{5.70E-2 $\pm$ 1.24E-1}   & \multirow{4}{*}{3.49E-1 $\pm$ 6.77E-1} & 9.02E-03 $\pm$ 1.41E-02             & 8.92E-2 $\pm$ 1.10E-1            & 1.20E-3 $\pm$ 1.07E-3               & 8.79E-2 $\pm$ 7.41E-2            & \textbf{1.81E-3 $\pm$ 7.02E-4}      & \textbf{7.34E-2 $\pm$ 2.18E-2}   \\
                     & 1E-4 &                                          &                                        & 3.40E-04 $\pm$ 2.85E-04             & 1.92E-2 $\pm$ 1.21E-2            & 3.10E-4 $\pm$ 4.13E-4               & 1.41E-2 $\pm$ 1.48E-2            & \textbf{2.77E-4 $\pm$ 3.75E-4}      & \textbf{2.63E-2 $\pm$ 3.67E-2}   \\
                     & 1E-6 &                                          &                                        & 2.52E-04 $\pm$ 1.53E-04             & 1.89E-2 $\pm$ 1.07E-2            & 5.89E-4 $\pm$ 8.60E-4               & 3.16E-2 $\pm$ 3.48E-2            & \textbf{6.30E-5 $\pm$ 7.07E-7}      & 4.84E-3 $\pm$ 1.33E-3            \\
                     & 1E-8 &                                          &                                        & 6.20E-05 $\pm$ 2.00E-06             & 1.83E-3 $\pm$ 5.53E-4            & 6.35E-05 $\pm$ 1.12E-06             & 4.44E-3 $\pm$ 3.46E-3            & \textbf{6.07E-5 $\pm$ 4.49E-6}      & \textbf{3.75E-3 $\pm$ 2.62E-3}   \\ \midrule
3.2                  & 1E-2 & \multirow{4}{*}{1.31E-2 $\pm$ 1.71E-2}   & \multirow{4}{*}{2.86E-1 $\pm$ 3.78E-1} & 6.28E-04 $\pm$ 6.24E-04             & 4.73E-2 $\pm$ 3.01E-2            & \textbf{5.84E-4 $\pm$ 8.71E-4}      & \textbf{2.77E-2 $\pm$ 2.62E-2}   & 1.27E-2 $\pm$ 2.02E-2               & 1.53E-1 $\pm$ 1.90E-1            \\
                     & 1E-4 &                                          &                                        & 9.95E-04 $\pm$ 1.29E-03             & 3.66E-2 $\pm$ 2.22E-2            & \textbf{3.30E-5 $\pm$ 1.15E-5}      & \textbf{9.06E-3 $\pm$ 9.40E-3}   & 8.07E-5 $\pm$ 3.33E-5               & 1.14E-2 $\pm$ 4.53E-3            \\
                     & 1E-6 &                                          &                                        & 1.59E-02 $\pm$ 2.74E-02             & 1.74E-1 $\pm$ 2.90E-1            & 7.30E-4 $\pm$ 7.58E-4               & 3.73E-2 $\pm$ 3.62E-2            & \textbf{1.05E-4 $\pm$ 1.33E-4}      & \textbf{6.42E-3 $\pm$ 1.84E-3}   \\
                     & 1E-8 &                                          &                                        & 2.45E-05 $\pm$ 8.66E-07             & 5.74E-1 $\pm$ 3.10E-1            & 3.85E-04 $\pm$ 6.11E-04             & 1.71E-2 $\pm$ 1.91E-2            & \textbf{v2.42E-5 $\pm$ 1.48E-6}     & \textbf{4.34E-3 $\pm$ 7.74E-4}   \\ \bottomrule
\end{tabular}%
\end{table}
\end{landscape}

\end{document}